\documentclass[11pt]{article}

\usepackage{fullpage}
\usepackage{graphicx} 

\usepackage{times,url,bm}
\usepackage[normalem]{ulem}
\usepackage{amsthm,amsmath,amssymb,mathtools,epsfig,color,float,graphicx,verbatim}
\usepackage{bbm}
\usepackage[square,numbers]{natbib}
\usepackage{epstopdf}
\usepackage{tikz}
\usetikzlibrary{shapes.geometric}
\usetikzlibrary{positioning}
\usepackage{authblk}
\usepackage{natbib}
\usepackage[dvipsnames]{xcolor}
\usepackage{comment}
\usepackage{booktabs}
\usepackage{caption}
\usepackage{subcaption}
\usepackage[ruled,vlined]{algorithm2e}
\usepackage{etoc}

\usepackage{hyperref}

\hypersetup{
	colorlinks   = true, 
	urlcolor     = blue, 
	linkcolor    = blue, 
	citecolor   = OliveGreen 
}

\newtheorem{theorem}{Theorem}[section]
\newtheorem{proposition}[theorem]{Proposition}
\newtheorem{lemma}[theorem]{Lemma}
\newtheorem{corollary}[theorem]{Corollary}
\newtheorem{definition}[theorem]{Definition}

\newtheorem{remark}[theorem]{Remark}
\newtheorem{assumption}[theorem]{Assumption}
\newtheorem{fact}[theorem]{Fact}

\newcommand{\reals}{\mathbb{R}}
\newcommand{\E}{\mathop{\mathbb{E}}}

\newcommand{\relu}[1]{\left[ #1 \right]_+}
\newcommand{\base}{\mathrm{base}}

\newcommand{\be}{\mathbf{e}}

\newcommand{\bx}{\mathbf{x}}
\newcommand{\bX}{\mathbf{X}}

\newcommand{\bb}{\mathbf{b}}
\newcommand{\br}{\mathbf{r}}

\newcommand{\bv}{\mathbf{v}}
\newcommand{\bz}{\mathbf{z}}

\newcommand{\bR}{\mathbf{R}}
\newcommand{\bL}{\mathbf{L}}

\newcommand{\by}{\mathbf{y}}

\newcommand{\bp}{\mathbf{p}}

\newcommand{\Ocal}{\mathcal{O}}

\newcommand{\Dcal}{\mathcal{D}}

\newcommand{\Hcal}{\mathcal{H}}
\newcommand{\Rcal}{\mathcal{R}}
\newcommand{\Ncal}{\mathcal{N}}

\newcommand{\Pcal}{\mathcal{P}}

\newcommand{\Ucal}{\mathcal{U}}

\newcommand{\norm}[1]{\left\|#1\right\|}

\newcommand{\p}[1]{\left(#1\right)}
\newcommand{\pcc}[1]{\left[#1\right]}
\newcommand{\abs}[1]{\left|#1\right|}
\newcommand{\ceil}[1]{\left\lceil#1\right\rceil}
\newcommand{\floor}[1]{\left\lfloor#1\right\rfloor}

\newcommand{\one}[1]{\mathbbm{1}\left\{#1\right\}}

\newcommand{\Exp}{\operatorname*{\mathbb{E}}}

\newcommand{\pr}{\mathop{\mathbb{P}}}
\DeclareMathOperator{\med}{med}
\DeclareMathAlphabet{\mathcalligra}{T1}{calligra}{m}{n}

\title{The Median is Easier than it Looks:\\Approximation with a Constant-Depth, Linear-Width ReLU Network}

\author[1]{Abhigyan Dutta}
\affil[1]{Purdue University}
\author[2]{Itay Safran}
\affil[2]{Ben-Gurion University of the Negev}
\author[1]{Paul Valiant}

\etocsettagdepth{main}{none}          
\etocsettagdepth{appendix}{subsection} 

\DontPrintSemicolon
\date{}

\begin{document}

\etocdepthtag.toc{main}

\maketitle

\begin{abstract}%
    We study the approximation of the median of $d$ inputs using ReLU neural networks. We present depth-width tradeoffs under several settings, culminating in a constant-depth, linear-width construction that achieves exponentially small approximation error with respect to the uniform distribution over the unit hypercube. By further establishing a general reduction from the maximum to the median, our results break a barrier suggested by prior work on the maximum function, which indicated that linear width should require depth growing at least as $\log\log d$ to achieve comparable accuracy. Our construction relies on a multi-stage procedure that iteratively eliminates non-central elements while preserving a candidate set around the median. We overcome obstacles that do not arise for the maximum to yield approximation results that are strictly stronger than those previously known for the maximum itself.
\end{abstract}

\section{Introduction}  \label{sec:introduction}

Neural networks have become one of the most prominent tools in machine learning in recent years. Their success is often attributed, among other factors, to their expressive power \citep{cybenko1989approximation,Hornik1989,leshno1993multilayer}. A popular line of theoretical work, aimed at understanding the role of depth, studies depth-width tradeoffs in approximation capabilities \citep{telgarsky2016benefits,eldan2016power,arora2016understanding,yarotsky2017error,liang2017deep,safran2017depth,safran2019depth,Venturi2021DepthSeparation,safran2024max,Safran2024DepthSeparations}, showing that increased depth can lead to significantly smaller overall network size. Despite the variety of settings in which depth-width separation results have been established, the assumptions underlying these results and the target functions used are often tailored to technically facilitate the analysis, as handling more natural settings is typically much harder. Motivated by this, recent work has shifted focus to more naturally occurring target functions. Among these, the maximum function \citep{mukherjee2017lower,hertrich2021towards,matoba2022theoretical,haase2023lower,safran2024max,bakaev2025depth,averkov2025expressiveness,grillo2025depth,safran2026depth} has received significant attention in recent years, as it plays a pivotal role in many areas of machine learning. While it has long been known that the maximum of $d$ inputs can be computed exactly by a neural network of depth logarithmic in $d$ \citep{arora2016understanding,bakaev2025betterneuralnetworkexpressivity}, it remains an open question whether this function can be realized by a shallower ReLU network without imposing any restrictions on the approximating architecture \citep{hertrich2021towards,haase2023lower,bakaev2025depth,averkov2025expressiveness,grillo2025depth}.

Despite the recent focus on the precise computation of the maximum function, it is arguably more interesting from a practical perspective to study whether a given target function can be approximated (in an $L_2$ sense, with respect to some underlying input distribution) rather than computed exactly. This approach is better aligned with machine learning applications, since achieving good generalization typically only requires a small approximation error. Moreover, practical learning algorithms, such as gradient descent, only find approximate local minima of the objective rather than converge to its exact value. The recent work of \citet{safran2024max} studies the depth-width tradeoffs in \emph{approximating} the maximum function with respect to continuous distributions. One of their contributions is a construction that requires only linear width and depth $\Ocal(\log\log d)$ to approximate the maximum, but necessitates super-linear width for shallower architectures, suggesting a potential barrier to achieving a similar approximation with linear width at smaller depth. This non-constant depth requirement is further strengthened by the recent result in \citet{safran2026depth}, where it is shown that \emph{exact} computation of the maximum indeed requires super-linear width for constant depth.

Apart from the maximum function, there is a vast literature studying more general continuous piecewise linear functions (CPWL) \citep{arora2016understanding,he2020relu,hertrich2021towards,kuanlin2022}. One such prototypical example is the \emph{median} function, whose computation is a more general and challenging problem than the maximum, since, unlike the maximum, the median of medians of a partition of the input is not necessarily the median of the original input---a property commonly leveraged to construct efficient approximations of the maximum. Moreover, as pointed out earlier, computing the maximum using a neural network with size linear in $d$ is relatively straightforward, but there are currently no known constructions that achieve the same size for the median function. Furthermore, for sufficiently large $d$, there is no known construction of depth $c\log_2(d)$, even with $c$ in the thousands, that computes the median with polynomial size. While it is a classic result that algorithms exist for computing the median in linear time \citep{Blum1973TimeBounds}, it is not clear how such an algorithm could be implemented using a small-size neural network. Moreover, existing Boolean circuit lower bounds for the majority function imply\footnote{This can be shown by a simple reduction: when all inputs are binary, the median of the input coincides with the majority function.} super-polynomial Boolean circuit size lower bounds for approximating the median with constant depth \citep{o2007approximation}. 
This naturally raises the following question: can a linear-sized neural network approximate the median function?

In this paper, we study how well ReLU networks can approximate various CPWL functions with respect to the uniform distribution over the unit hypercube. Specifically, we focus on approximating the rank-$k$ element of a $d$-dimensional input, which includes both the maximum and the median as special cases. We prove various depth-width trade-offs for approximating the rank-$k$ function with respect to the uniform distribution on the unit hypercube, showing that increasing depth reduces the width required and culminating in a perhaps surprising linear-width, \emph{constant}-depth construction. Our proof is technically involved and requires several intermediate steps, some of which may be of independent interest. In particular, we show that it is possible to repeatedly estimate a window around the rank-$k$ element while zeroing out all elements outside this window. After four such iterations, the set of non-zero elements, which contains the true rank-$k$ element with overwhelming probability, becomes small enough to allow a novel implementation of a hashing trick to extract the rank-$k$ element. As a special case with $k=1$, this breaks the barrier suggested by prior work \citep{safran2024max,safran2026depth}. Additionally, we establish the first linear-sized approximation for the median. This result demonstrates that relaxing the precision requirement to a practical error threshold enables significantly more compact representations.

The rest of this paper is structured as follows: After presenting our main contributions in further detail below, we turn to review additional related work in the literature. In Section~\ref{sec:notation}, we introduce our notation and framework. In Section~\ref{sec:upper_bound_constructions} we present our upper bound constructions, and Section~\ref{sec:lower_bounds} specifies complementary lower bounds. Finally, in Section~\ref{sec:summary} we summarize this paper and detail potential future work directions.

\subsection*{Our contributions}

\begin{itemize}
    \item
    We provide a ReLU network construction with depth 46 and width linear in $d$ that approximates the median function with respect to the uniform distribution on the unit hypercube (Theorem~\ref{thm:linear_width_med_computation_distribution_zeroone}). Specifically, this construction achieves an accuracy that is exponentially small in $d$ for sufficiently large weights.
    \item
    We prove that for any $k$, there exists a depth-3 ReLU network with width $\Ocal(d^2)$, that approximates the rank-$k$ element to arbitrary accuracy with respect to the uniform distribution over the unit hypercube, provided the network weights are sufficiently large (Theorem~\ref{theorem:depththreeconstruction}).
    \item
    We show that a mild increase in depth can result in a noticeable improvement in the width requirement for computing any rank. For any $k$, there exists a ReLU network with depth-$5$ and width roughly $\Ocal(d^{5/3})$ that approximates the rank-$k$ element of the input to exponentially small target accuracy, provided that the weights are sufficiently large (Theorem~\ref{theorem:depth_5_median_approx}).
    \item
    We provide a reduction from the exact computation of any rank-$k$ function to the exact computation of the maximum. By reducing to the exact lower bound in \citet{safran2026depth}, we derive an exact computation lower bound for computing the median (Theorem~\ref{theorem:exact_median}). By combining this with our approximation result in Theorem~\ref{thm:linear_width_med_computation_distribution_zeroone}, we demonstrate a separation between the exact and approximate computation regimes. This formalizes the gap between these two settings, showing that obtaining an approximation is a significantly less stringent requirement than computing the function exactly.
    \item 
    Finally, we introduce a general reduction scheme from the median function to the maximum, enabling approximation lower bounds for the maximum to directly imply corresponding bounds for the median (Theorem~\ref{theorem:depthany_lower_bounds}). Specifically, by applying this reduction to existing results in \citet{safran2024max}, we derive lower bounds for approximating the median at depths 2 and 3 (Corollaries~\ref{cor:depth2_lower_bounds} and \ref{cor:depth3_lower_bounds}, respectively). Furthermore, we obtain a general lower bound for arbitrary depth, establishing that the required width must be at least linear in $d$ (Theorem~\ref{theorem:all_depth_lower_width_bounds}).

\end{itemize}

\subsection*{Additional related work} \label{subsec:related_work}

\paragraph{$L_2$ approximation of the maximum.}

The work most closely related to ours is \citet{safran2024max}, which establishes lower and upper bounds for approximating the maximum function with respect to the uniform distribution on the unit hypercube. Their constructions rely on the idempotent property of the maximum where the maximum of maxima of subsets equals the global maximum. Since this property fails to hold for the median, we introduce a more intricate multi-stage construction that yields an efficient approximation for the median and consequently for the maximum. Notably, our approach improves upon the upper bounds in \citet{safran2024max} within the regimes we consider. While our results are restricted to an accuracy threshold that is exponentially small in the input dimension, we emphasize that this setting encompasses the error regimes typically encountered in practical applications and captures the most relevant approximation scenarios.

\paragraph{Exact computation of CPWL functions.}

The exact computation of CPWL functions has a rich literature. Based on the fact that any CPWL function can be expressed as a nested sum of minima and maxima of affine transformations \citep{ovchinnikov2000max,wang2005generalization}, \citet{arora2016understanding} show that depth $\lceil \log d \rceil+1$ suffices to compute any CPWL function on $\mathbb{R}^d$. A subsequent line of work has improved the required network complexity, either by reducing the depth \citep{bakaev2025betterneuralnetworkexpressivity} or by decreasing the width based on the geometric properties of the target function \citep{he2020relu,hertrich2021towards,kuanlin2022}. Notably, these upper bounds depend on the number of convex regions in the function's polyhedral decomposition. While this quantity is exactly $d$ for the maximum, for the median it is a combinatorial quantity that grows exponentially with $d$, rendering such architectures prohibitively large.

A related line of work on sorting networks \citep{knuth1998art} provides constructions that imply the exact computation of the median with ReLU networks. Using depth that scales as $\Ocal(\log d)$ with a large constant $c$ or as $\Ocal(\log^2 d)$, there exist networks with linear width that compute the median \citep{batcher1968sorting,ajtai1983sorting}. However, because these constructions require either significantly more depth than the maximum or rely on prohibitively large constants, they remain far less efficient than current architectures for the maximum.

In light of the above, to achieve efficient median computation with constant depth and linear width, it is natural to relax the approximation requirement. We show that even demanding exponentially small accuracy results in a dramatic improvement over the exact computation bounds discussed above, bypassing the logarithmic depth barriers inherent to the exact regime.

\section{Preliminaries and notation}
\label{sec:notation}

\paragraph{Notations.} We use bold-faced letters to denote vectors: $\bx = \left(x_1,x_2, \ldots, x_d \right) \in \mathbb{R}^d$. We use the shorthand $[n] \coloneqq \{1,2,\ldots,n\}$. We use $\max(\bx)$ and $\med(\bx)$ to denote the \textit{maximum} and \textit{median} of the entries in the vector $\bx$, correspondingly. Given a set $S$, we denote by $\mathcal{U}(S)$ the uniform distribution over $S$. Given a vector $\bx\in\reals^d$, we denote its rank-$k$ element using $\Rcal_k(\bx)$---namely, the $k^{\textrm{th}}$ entry when $\bx$ is sorted in ascending order. All logarithms have base $e$ unless otherwise stated.

\paragraph{Neural networks.} Throughout, we use the notation $\relu{z}=\max\{0,z\}$ for the ReLU activation function. In our work we consider fully connected, feed-forward neural networks; i.e., each neuron in the network computes some non-linear activation function $\sigma$. When $\sigma$ is restricted to a certain class of activation functions or specifically a ReLU, we explicitly mention this in the text. A depth-$h$ Neural network consists of $h$ hidden layers of neurons, followed by the output neuron which computes an affine transformation of its inputs. Each hidden layer computes an affine transformation of its inputs and then applies its activation function separately on each coordinate before propagating the output forward to the next hidden layer. The \emph{depth} of a neural network is the number of hidden layers in it plus one. The \emph{width} of a network is defined as the number of neurons in the largest hidden layer. Finally, the \emph{size} of a neural network is defined as the overall number of neurons across all layers.

\paragraph{Approximation error.} We focus solely on a regression setting, where a neural network $\Ncal:\reals^d\to\reals$ computes a real-valued function of its input. We use the square loss throughout, and the approximation error is measured with respect to an underlying distribution $\Dcal$, supported in $\mathbb{R}^d$.  While our results are presented mainly for $\Dcal=\Ucal\p{[0,1]^d}$, we define our approximation scheme in a more general manner, as our results can be extended to hold for a broader family of continuous distributions. Formally, given a neural network $\Ncal$, a target function $f:\reals^d\to\reals$ to be approximated, and a data distribution $\Dcal$, our approximation error is the expected square loss given by
    \[
        \Exp_{\bx\sim\Dcal}\pcc{\p{\Ncal(\bx) - f(\bx)}^2}.
    \]

\section{Neural network approximation for the median function} \label{sec:upper_bound_constructions}

In this section, we present our upper bounds for the approximation of the median function. As previously discussed, we establish tradeoffs between width and depth, showing that increasing depth allows a reduction in the required width.

\subsection{Depth $3$ and width $\mathcal O(d^2)$ median computation} \label{section:depth3}

We start with the simplest positive approximation result, which allows the approximation of the median using depth $3$ and width $\Ocal(d^2)$ when the data distribution is uniform over the unit hypercube.\footnote{We emphasize that our constructions are effective for a significantly broader family of continuous distributions. Given that our analysis primarily requires input permutation invariance and sufficiently distinct values, extending these upper bounds to hold for any i.i.d.\ and bounded continuous distribution is not too difficult: the neural networks do not need to know the input distribution, only bounds on its range and density. However, to maintain clarity within an already technically demanding analysis, we focus on the uniform distribution to provide a more accessible exposition of the core mechanisms.} This result not only provides us with a constructive method for approximating the median using depth 3, it also gives a flavor of the arguments we use in our main result which build upon it, as this architecture is pivotal in building up more sophisticated constructions to extract the median.

\begin{theorem}\label{theorem:depththreeconstruction}
     For any dimension $d$ and any target accuracy $\epsilon>0$, there exists a ReLU neural network $\mathcal N$ of depth $3$ and width $\mathcal O \left(d^{2}\right)$, and magnitude of weights bounded by $\frac{12d^4}{\epsilon}$, such that
    \[
        \Exp_{\bx \sim \mathcal{U} \left( [0,1]^d \right)} \left[ \left( \mathcal N(\bx) - \med(\bx)\right)^2 \right] \leq \epsilon. 
    \]
    
\end{theorem}

The above construction, whose proof is available in Appendix~\ref{theorem:depththreeconstruction_proof}, can be seen as a generalization of \citet[Theorem~3.3]{safran2024max}. It not only allows the extraction of the median element, but any rank $k\in[d]$ element, which includes the maximum as a special case, while also using the same network architecture. The main intuition behind the construction is that we can use the first hidden layer to compute pairwise indicators, comparing all pairs of inputs (hence the quadratic width requirement). Thereafter, the second hidden layer aggregates these results, identifies the entry whose indicators sum to the desired rank, and then outputs it. We remark that more generally, this implies a depth-3 width-$\Ocal(d^2)$ construction with $d$ outputs that sorts all inputs, which is a fundamental building block in the constructions that we present next.

\subsection{Depth $5$ and width roughly $\mathcal O(d^{5/3})$ median computation}

The following result demonstrates that even a small increase in depth from 3 to 5 can result in a meaningful reduction in the required width for approximating the median. Specifically, the reduction in width is from a quadratic to $\Ocal(d^{5/3+\gamma})$, where $\gamma>0$ is a confidence parameter that can be made arbitrarily small, at the cost of an (exponentially negligible) additional approximation error. More formally, our result is the following.

\begin{theorem} \label{theorem:depth_5_median_approx}
    For any dimension $d$, any target accuracy $\epsilon>0$, there exists a ReLU neural network $\mathcal N$ of depth $5$ and width $\mathcal O(d^{5/3+\gamma})$, and with magnitude of weights bounded by $\frac{12d^6}{\epsilon}$, such that
    \[
       \Exp_{\bx \sim \mathcal{U} \left([0,1]^d\right)} \left[ \left( \mathcal N(\bx) - \med(\bx)\right)^2 \right] \leq \epsilon + \exp \left( - \Omega \left(d^{2\gamma} \right)\right) .
    \]
    
\end{theorem}

    The proof of the theorem, which appears in Appendix~\ref{theorem:depth_5_median_approx_proof}, relies on a probabilistic argument. Our architecture utilizes the first two hidden layers to partition the input into batches and compute their central elements using the construction devised in the previous subsection. Because a sufficiently large sample of central elements is exponentially likely to capture the global median, and because the resulting candidate set is significantly smaller than $d$, the next step employs two additional hidden layers to compare all candidates against all inputs. This identifies the true median while maintaining a width strictly smaller than quadratic. The parameter $\gamma$ controls the confidence of this construction's success, which leads to the inevitable exponentially small error floor. 
    
    We remark that our construction carries a width requirement with an exponent of at least $5/3$, representing a fundamental limitation of the current approach that cannot be further improved. In contrast, the analogous result for the maximum in \citet[Theorem 3.4]{safran2024max} requires an exponent of only $4/3$. This suggests that while the median exhibits depth-width tradeoffs similar to those of the maximum, its required width appears to scale more steeply. This naturally raises the question of whether the median can be approximated by a depth-$\mathcal{O}(\log\log d)$, linear-width network, as is possible for the maximum function. In the next subsection, we show, perhaps surprisingly, that not only is linear width achievable, but it can even be attained using \emph{constant} depth.

\subsection{Depth $\mathcal O(1)$ and width $\mathcal O(d)$ median computation}

The following is our main result in this paper.

\begin{theorem}\label{thm:linear_width_med_computation_distribution_zeroone}
    For any dimension $d$ and any target accuracy $\epsilon>0$, there exists a ReLU neural network $\mathcal N$ of depth $46$ and width $\mathcal O(d)$, with magnitude of weights bounded by $\Ocal\p{\frac{d^2}{\epsilon}}$, such that
    \[
        \Exp_{\bx \sim \mathcal{U} \left([0,1]^d\right)} \left[ \left( \mathcal N(\bx) - \med(\bx)\right)^2 \right] \leq \epsilon + \exp \left( - d^{\Omega(1)}\right) .
    \]
\end{theorem}
    
    The proof of the above theorem, which appears in Appendix~\ref{thm:linear_width_med_computation_distribution_zeroone_proof}, is technically involved and requires several intermediate constructions. Given the inherent complexity of the proof, we have prioritized clarity of exposition over obtaining a tighter depth analysis, and therefore we did not attempt to optimize the depth, which can likely be further improved to a constant smaller than 46. 
    
    At a high level, the proof relies on an iterative procedure that partitions the current elements into batches and propagates only their central elements, effectively zeroing out inputs unlikely to be the median. This process continues until no more than $d^{\alpha}$ (for a sufficiently small $\alpha > 0$) of the original inputs remain non-zero. While sorting these remaining elements using the construction from Theorem~\ref{theorem:depththreeconstruction} would ideally yield the median, a significant implementation hurdle remains: ReLU networks cannot selectively sort only non-zero values. To circumvent this, we employ a novel hashing scheme that maps the non-zero elements into a reduced-dimension vector where they are guaranteed to be distinct with high probability. By subsequently sorting this hashed vector, we can extract the true median.

\subsubsection{Proof sketch of Theorem~\ref{thm:linear_width_med_computation_distribution_zeroone}---specialized for the maximum}

In this subsection, we provide a more detailed yet accessible proof sketch of our main result. To illustrate our primary construction and implementation techniques, we present the argument for the simpler special case of extracting the maximum. Focusing on the maximum allows us to bypass some of the technical intricacies unique to the median, while still demonstrating the most significant and novel elements of the proof. This approach highlights how our method overcomes existing depth barriers that previously suggested a $\log \log d$ requirement for such architectures.

For simplicity of presentation, it will be more convenient to describe an architecture as an algorithm, deferring implementation considerations with ReLU networks. We point out, however, that all neural network constructions used and implemented in our proofs are rigorously analyzed in Appendix~\ref{sec:helper_nets}. Additionally, we make the following remark regarding our terminology and usage of randomness within our neural network constructions.

\begin{remark}[Random selection within ReLU neural networks]\label{remark:randomness} The algorithms in this section are described as though they have access to randomness---despite neural networks being deterministic objects. Our ultimate neural network constructions will ``derandomize'' the ideas of this section, sometimes by leveraging randomness coming from the input distribution of the neural network, and sometimes via explicit derandomization. Deriving our final linear-width high-probability bounds for the median involves subtle probabilistic analysis of the distribution of intermediate network data. However, for clarity of exposition, this section will describe the main (randomized) algorithms ideas without reference to these analytical hurdles.

\end{remark}

\paragraph{Sparsification step.}

The first step in our construction is the sparsification process which is described in Algorithm~\ref{alg:informalAlgoMax}, where a significant portion of the lower-rank inputs are zeroed out.

\begin{algorithm}[t]
\caption{Sparsification step for finding the maximum}\label{alg:informalAlgoMax}
\KwIn{ Vector: $\bx_1 \in (0,1)^d $ with unique 
entries}
\KwOut{Modified Vector $\bx_3$}
\SetKwFunction{maximum}{sparsifyingMax} 
\DontPrintSemicolon
\maximum{$\bx_1$}{:} \;
\begin{enumerate}
    \item For $i\in\{1,2\}$ define the parameters $\alpha_1=d^{0.5},\alpha_2=d^{0.4}$.

    \item For $i \in \{1,2\}$ do:
    \begin{enumerate}

    \item \label{informalalg:stepa} From $\bx_{i}$ pick a random sample of 
${\alpha_i}$ non-zero entries, denoted $S_i$. 

\item \label{informalalg:stepb}  Define $m_i = \max \left( S_i \right)$.

\item \label{informalalg:stepc} Create a new copy of $\bx_{i}$ represented as $\bx_{i+1}$, where in $\bx_{i+1}$ all entries with values $< m_{i}$  are changed to $0$.
        
\end{enumerate}
\end{enumerate}
\end{algorithm}

Unlike our construction for the median which requires four iterations of the sparsification loop, the maximum only requires two. In the first iteration, we randomly select\footnote{We emphasize that while ReLU networks are inherently deterministic, the required randomness is induced by our distributional assumptions. Consequently, the term `random selection' refers to a deterministic architecture choosing an arbitrary subset of the input. See Remark~\ref{remark:randomness} for further discussion.} a subset of the inputs of size $\alpha_1=d^{0.5}$ and extract its maximum using two hidden layers with width linear in $d$. We then utilize an additional hidden layer with linear width to zero out all elements smaller than this approximate maximum. In expectation, this would zero out all but $d^{0.5}$ elements. Next, we seek to extract $\alpha_2$ non-zero elements from the remaining entries. However, unlike the first iteration, we now have a sparsified vector, and sampling a non-zero subset from it is not straightforward. To overcome this obstacle, we require designing a new neural network, which we call the \emph{non-zero element shortlisting} neural network (see Definition~\ref{defn:shortlist_non-zero_net}). It would be ideal to have $\alpha_2=d^{0.5}$, since then we would hope to extract all the remaining non-zero elements and immediately find the overall maximum; however, the \emph{non-zero element shortlisting} architecture would require superlinear width. Instead, we can extract $\alpha_2=d^{0.4}$ elements in linear width, with high probability over a random permutation of the input, by chopping the input into small blocks, relying on concentration bounds on the number of non-zero elements in each block, to repeatedly run \emph{non-zero element shortlisting} neural network on these blocks.

The second sparsification iteration leaves us with a sparsified vector of $d$ entries, roughly $d^{0.1}$ of which are non-zero. In what follows, this extreme sparsity will enable us to explicitly extract these non-zero entries using a different approach. 

\paragraph{Hashing step.}
In this final stage, we begin with a sparse vector $\bx'$ containing approximately $d^{0.1}$ non-zero entries. If we could retain these non-zero entries but in a vector of size $\leq d^{0.5}$, then we could use a quadratic-width maximum architecture to finish the algorithm; so our goal here is to ``hash down'' the $d$ input locations into $\leq d^{0.5}$ locations, so as not to produce any collisions between non-zero entries of $\bx'$. One naive attempt is to partition the vector $\bx'$ into blocks of size $d^{0.5}$ and sum the values within each batch to produce a lower-dimensional vector of size $d^{0.5}$, since in a randomly ordered vector, the $d^{0.1}$ non-zero locations will typically be spread out. If each block contains at most one non-zero coordinate, this reduced vector preserves the maximum. However, the probability that some block contains multiple non-zero values is inverse polynomial, which is not small enough for our aims. 

To guarantee success, we implement a universal hashing scheme that maps the indices of $\bx'$ to discrete bins. Because a single hash function may still fail, our architecture implements every element of a universal family of hash functions in parallel. We prove that for any sparse $\bx'$, there exists at least one function in this family that maps each non-zero entry to a unique bin. We then employ a counting and masking procedure to identify this successful instance: the network counts the non-zero entries in $\bx'$, and compares this to the non-zero counts of each hashed block. A match indicates a collision-free hash, and the corresponding low-dimensional output $\bx''$ is propagated. Finally, we extract the maximum from $\bx''$ using a brute-force comparison that now requires only linear width due to the significantly reduced dimension. This hashing process is outlined in Algorithm~\ref{alg:sparseToMaximumFunction}.

\begin{algorithm}[t]
\DontPrintSemicolon
\caption{Hashing step for finding the maximum}
\label{alg:sparseToMaximumFunction}
\SetKwFunction{maximum}{sparseToMaximum} 
\KwIn{Sparse vector $\bx'$}
\KwOut{$\max(\bx')$}
\maximum{$\bx'$,\,p}{:} \;
\begin{enumerate}
\item Use all hash functions $h:\mathbb{R}^d\to\mathbb{R}^p$ in $\mathcal H$ to map $\bx'$. \;
\item Identify a collision-free $h\in\mathcal H$ and set $\bx''=h(\bx')$. \;
\item Return $\max(\bx'')$ by direct comparisons. \;
\end{enumerate}
\end{algorithm}

\section{Lower bounds for computing the median} \label{sec:lower_bounds} 

Having established upper bounds in the previous section, we now turn to the study of lower bounds. While this section, as with the rest of the paper, is primarily concerned with approximation results for the median function, we begin by considering the simpler case of exact computation. Specifically, we present a reduction from a recent exact computation result in \citet{safran2026depth}. This approach allows us to illustrate the core intuition behind our reduction technique while bypassing the technical complexities inherent in the approximation setting, which we will delve into later in this section. Formally, we build upon the following result, stated here for completeness.
\begin{theorem}[{\citet[Theorem~1.1]{safran2026depth}}]\label{theorem:max_hierarchy}
    Suppose that $3\le k\le\log_2\log_2d$. Let $\Ncal$ be a depth-$h$ ReLU network such that $\Ncal(\bx)=\max(\bx)$ for all $\bx\in[0,1]^d$. Then, $\Ncal$ has width at least
    \[
        \frac{1}{10}d^{1+\frac{1}{2^{k-2}-1}}.
    \]
\end{theorem}
With the above, we are able to establish the following depth hierarchy lower bound for the $r$-rank function:
\begin{theorem}\label{theorem:exact_median}
    Suppose that $r\in[d]$ and $3\le k\le\log_2\log_2d$. Let $\Ncal$ be a depth-$k$ ReLU network such that $\Ncal(\bx)=\Rcal_r(\bx)$ for all $\bx\in[0,1]^d$. Then, $\Ncal$ has width at least
    \[
        \frac{1}{40}d^{1+\frac{1}{2^{k-2}-1}}.
    \]
\end{theorem}
The above is proven via a direct reduction from the maximum function to the $r$-rank function, the details of which are provided in Appendix~\ref{app:exact_median_proof}. To provide intuition for this reduction, we focus on the median for simplicity. Suppose we are given a neural network architecture capable of computing the median for any input dimension $d$. For a given $d$-dimensional input, we can employ a $(2d-1)$-dimensional version of this architecture and pad the input with $d-1$ coordinates fixed to a value of $1$. Since these auxiliary coordinates are at least as large as the original $d$ inputs, the median of the resulting $(2d-1)$-dimensional vector coincides with the maximum of the original $d$ inputs. Furthermore, since fixing input coordinates to constants yields a valid ReLU network of the same architecture (by appropriately modifying the bias terms in the first hidden layer), it follows that there exists a network of this type that computes the maximum precisely. Given that the exponent in Theorem~\ref{theorem:max_hierarchy} is at most $2$ and $r \le d-1$, this reduction results in a constant multiplicative blow-up of at most $4$ in the network width.

The previous section suggested that computing the median is inherently more difficult than computing the maximum, which is reflected in the additional insights required to compute the median in our upper bounds. Theorem~\ref{theorem:exact_median} formalizes this by demonstrating that the maximum can be reduced to the median (or any rank $r$), thereby rigorously establishing that the median is at least as difficult to compute as the maximum. Because the median represents the most challenging rank to compute, the remainder of our analysis focuses exclusively on the median function. We emphasize, however, that our results generalize to any rank $r$ in a straightforward manner.

Despite the simplicity of the exact reduction, our primary objective is to characterize the complexity of approximate computation. To this end, we introduce a general reduction scheme that establishes an analogous result for $L_2$ approximation with respect to the uniform distribution on the unit hypercube, $[0,1]^d$. This setting is particularly motivated by existing lower bounds for the maximum function in the $L_2$ regime; by reducing from these results, we can derive corresponding approximation lower bounds for the median. Our reduction is as follows:

\begin{theorem} \label{theorem:depthany_lower_bounds}
    Let $\sigma$ be any measurable activation function, and suppose that for all $d\ge2$, there exists a depth-$k$, width-$w(d)$ $\sigma$-neural network $\Ncal:\reals^d\to\reals$ with weights bounded by $M(d)$, that satisfies
    \[
        \Exp_{\bx \sim \Ucal \left([0,1]^d \right)} \left [ \left( \Ncal (\bx) -  \med (\bx) \right)^2 \right] \le \varepsilon.
    \]
    Then, there exists a depth-$k$, width-$w(2d)$ $\sigma$-neural network $\Ncal'$ with weights bounded by $2M(2d)$ such that
    \[
        \Exp_{\bx \sim \Ucal \left([0,1]^d \right)} \left [ \left( \Ncal' (\bx) -  \max (\bx) \right)^2 \right] \le 8\sqrt{\pi d}\varepsilon.
    \]
\end{theorem}

The proof of the above theorem, deferred to Appendix~\ref{theorem:depthany_lower_bounds_proof}, builds upon the technical framework established in the proof of Theorem~\ref{theorem:max_hierarchy}. The primary challenge in extending the reduction to the $L_2$ setting is that simply padding the $d$-dimensional input with $d-1$ ones restricts the resulting $(2d-1)$-dimensional inputs to a set of measure zero in $(2d-1)$-dimensional space, rendering the standard reduction unusable. To circumvent this, we instead sample the $d-1$ auxiliary coordinates from a uniform distribution supported in $[1,2]$ and permute the coordinates. This construction ensures that the median of the augmented input coincides with the maximum of the original input, while guaranteeing that the support of the resulting distribution has a non-negligible measure in $[0,2]^{2d-1}$. This measure is bounded away from zero because the probability of drawing approximately half of the auxiliary values from the lower half of their support is proportional to the mode of a binomial distribution, which accounts for the $\sqrt{d}$ factor appearing in our accuracy bound.

To utilize the reduction described above, we leverage existing lower bound results for the maximum function. First, however, we must adopt the same assumptions regarding the activation function as those employed in the prior literature:

\begin{assumption}[Polynomially-bounded activation]\label{asm:poly_bounded}
    The activation function $\sigma$ is Lebesgue measurable and satisfies
        \[
            \abs{\sigma(x)}\le C_{\sigma}\p{1+\abs{x}^{\alpha_{\sigma}}},
        \]
        for all $x\in\reals$ and for some constants $C_{\sigma},\alpha_{\sigma}>0$.
\end{assumption}
The following result, established in \citet{safran2024max} and restated here in a slightly modified manner to better suit our context, shows that achieving an arbitrarily accurate approximation of the maximum function with respect to the uniform distribution on $[0,1]^d$, using a depth-2 ReLU network, requires the width to scale with the target accuracy $\epsilon$.

\begin{theorem}[{\citet[Theorem~4.2]{safran2024max}}]\label{thm:max_depth2}
    For all natural $n\ge1$, suppose that $\sigma$ satisfies Assumption~\ref{asm:poly_bounded}. Then, there exist constants $c_1,c_2>0$ which depend solely on $\sigma$ such that for all dimensions $d\ge c_1$, a $\sigma$-neural network $\Ncal$ of depth 2 and width at most $n$ and with weights bounded by $\Ocal(\exp(\Ocal(d)))$ must satisfy
    \[
        \Exp_{\bx\sim\Ucal\p{\pcc{0,1}^d}}\pcc{\p{\Ncal(\bx) - \max(\bx)}^2} > \Omega\p{n^{-c_2}}.
    \]
\end{theorem}

By utilizing the above result and our Theorem~\ref{theorem:depthany_lower_bounds}, absorbing constants inside the asymptotic notation, the following is an immediate corollary:

\begin{corollary}\label{cor:depth2_lower_bounds}
    For all natural $n\ge1$, suppose that $\sigma$ satisfies Assumption~\ref{asm:poly_bounded}. Then there exist constants $c_3,c_4>0$ which depend solely on $\sigma$ such that for all dimensions $d\ge c_3$, a $\sigma$-neural network $\Ncal$ of depth 2 and width at most $n$ and with weights bounded by $\Ocal(\exp(\Ocal(d)))$ must satisfy
    \[
        \Exp_{\bx\sim\Ucal\p{\pcc{0,1}^d}}\pcc{\p{\Ncal(\bx) - \med(\bx)}^2} > \Omega\p{n^{-c_4}}.
    \]
\end{corollary}

Next, we extend our analysis to higher depths by providing a reduction from a known depth-3 lower bound for the maximum function. Specifically, we leverage the following result from \citet{safran2024max}:

\begin{theorem}[{\citet[Theorem 4.3]{safran2024max}}]
    Suppose that $\mathcal{N}$ is a depth-$3$ ReLU network of width at most $\frac{d^2}{5}$ and with weights bounded by $\exp \left( \mathcal O (d)\right)$. Then, there exist absolute constants $c_1,c_2 > 0$ such that for all $d \geq c_1$,

    \begin{equation*}
        \begin{split}
        \Exp_{\mathbf{x} \sim \mathcal{U} \left([0,1]^{d} \right)} \left [ \left( \mathcal N (\mathbf{x}) - \med (\mathbf{x}) \right)^2 \right] > \Omega \left( d^{-c_2} \right) .
        \end{split}
    \end{equation*}
\end{theorem}

With the above and Theorem~\ref{theorem:depthany_lower_bounds}, we are able to show the following result via a reduction.

\begin{corollary} \label{cor:depth3_lower_bounds}
    Suppose that $\mathcal{N}$ is a depth-$3$ ReLU network of width at most $\frac{d^2}{20}$ and with weights bounded by $\exp \left( \mathcal O (d)\right)$. Then, there exist absolute constants $c_3,c_4 > 0$ such that for all $d \geq c_3$,

    \begin{equation*}
        \begin{split}
        \Exp_{\bx \sim \mathcal{U} \left([0,1]^{d} \right)} \left [ \left( \mathcal N (\bx) - \med (\bx) \right)^2 \right] > \Omega \left( d^{-c_4} \right).
        \end{split}
    \end{equation*}
\end{corollary}

Lastly, the following result is an extension of the lower bound derived in \citet[Theorem~4.4]{safran2024max} for neural networks of arbitrary depth.

\begin{theorem} \label{theorem:all_depth_lower_width_bounds}
    Let $d\ge2$, and suppose that $\mathcal N$ is a neural network employing any activation function and having first hidden layer width of at most $d-1$. Then,
    \[
        \Exp_{\bx \sim \mathcal{U} \left( [0,1]^d\right)} \left[ \left( \mathcal N(\bx) - \med(\bx)\right)^2 \right] \geq \Omega \left( d^{-5.5} \right).
    \]
\end{theorem}

The above theorem shows that width of at least $d$ is required for approximating the median. Since a direct application of the reduction technique from Theorem~\ref{theorem:depthany_lower_bounds} incurs a constant factor loss of $0.5$, we instead adapt the approach of \citet{safran2024max} to establish a width $d$ lower bound for the median. See Appendix~\ref{theorem:all_depth_lower_width_bounds_proof} for the full proof.

\section{Summary and future work}\label{sec:summary}

In this work, we have established several lower and upper bounds for the approximation of the median function using neural networks. Our analysis highlights that the median, which is inherently more complex than the maximum, requires novel techniques for efficient approximation, culminating in a constant-depth, linear-width construction for the uniform distribution over the unit hypercube. Together with our exact computation lower bound in Theorem~\ref{theorem:exact_median}, these results establish the first separation between the exact and approximate computation regimes. Specifically, we demonstrate that relaxing the accuracy requirement to even an exponentially small error threshold allows for a dramatic reduction in the depth required to achieve a high-accuracy approximation using a linearly-sized network.

Furthermore, while our constant-depth, linear-width construction cannot achieve arbitrarily high accuracy via weight-scaling alone as the primary upper bound in \citet{safran2024max} does, it offers a superior architecture for error thresholds exceeding the exponentially small regime. This highlights that for the levels of approximation most commonly encountered in practical settings, the architectural requirements for the median (and, by extension, for the maximum via Theorem~\ref{theorem:depthany_lower_bounds}) may be significantly more modest than previously established.

There are several natural directions for future research. First, while our lower bounds establish that the median is at least as difficult to approximate as the maximum, our upper bounds strongly suggest that it is strictly more difficult. The latter required significantly more sophisticated insights and techniques to implement. It would be of great interest to explore whether this intuition can be formally verified by deriving lower bounds that do not rely on a reduction from the maximum, potentially yielding strictly stronger complexity results that are specific for the median function.

Second, given that scaling our weights beyond an exponential magnitude yields no additional benefit, it remains to be seen whether alternative constant-depth, linear-width constructions exist that can achieve arbitrarily high accuracy solely through weight-scaling, analogous to the primary result in \citet{safran2024max}. Conversely, establishing a provable barrier to such accuracy in the linear-width regime would further delineate the fundamental trade-offs between parameter magnitude and network architecture.

Finally, while Corollary~\ref{cor:depth3_lower_bounds} establishes that a linear-width approximation of the median to exponentially small accuracy is impossible at depth 3, Theorem~\ref{thm:linear_width_med_computation_distribution_zeroone} demonstrates that such an approximation is achievable at depth 46. Since we did not attempt to optimize the constant 46, it remains an open and intriguing question to determine the precise minimal depth required to achieve this level of accuracy within the linear-width regime.

\subsection*{Acknowledgments}
Itay Safran is supported by Israel Science Foundation Grant No.\ 1753/25. Abhigyan Dutta and Paul Valiant are partially supported by NSF award CCF-2127806 and by Office of Naval Research award N000142412695.

\bibliography{citations_colt}

\appendix

\etocdepthtag.toc{appendix}

\newpage

\begingroup
  \etocsettocstyle{\section*{Map of the Appendix}}{}
  \etocobeydepthtags
  \tableofcontents
\endgroup

\newpage

\section{Appendix-specific notations} \label{appdx:notn}

Recall our previously defined notations: We use bold-faced letters to denote vectors: \newline $\bx = \left(x_1,x_2, \ldots, x_d \right) \in \mathbb{R}^d$. For a positive integer $n$, we use the shorthand $[n] \coloneqq \{1,2,\ldots,n\}$. We define by $\mathbb{Z}^{>0}$ the set of strictly positive integers. The indicator function is denoted by $\mathbbm{1}{\{\cdot\}}$. For a vector $\bx$, we define $\max(\bx)$ and $\med(\bx)$ to be the \textit{maximum} and \textit{median} of the entries in the vector $\bx$, correspondingly. Given a set $S$, we denote by $\mathcal{U}(S)$ the uniform distribution over this set $S$. Given a vector $\bx\in\reals^d$, we denote its rank-$k$ element using $\Rcal_k(\bx)$---namely, the $k^{\textrm{th}}$ entry when $\bx$ is sorted in ascending order. All logarithms have base $e$ unless otherwise stated.

We call a vector $\bx \in \mathbb{R}^d$, $(d,\varepsilon)$ sparse if $\bx$ has at most $d^{\varepsilon}$ non-zero entries, i.e., $\sum_{i=1}^d \mathbbm{1}\left \{ x_i \neq 0 \right\} \leq d^{\varepsilon}$. We occasionally treat a vector $\bx$ as an array, and hence define the corresponding notations $\bx[i:j]\coloneqq \left(x_i,x_{i+1}, \ldots x_j \right)$ and $\bx[i] = x_i$.
For $p \in [1, \infty)$, the $\ell_p$ norm of $\bx$ is defined as $\norm{x}_p:=\p{\sum_{i=1}^d|x_i|^p}^{1/p}$, where the specific case $p=\infty$ is defined as $\norm{\bx}_{\infty} :=\max_{i\in[d]}|x_i|$. Finally, we denote by $\bx^{\neq 0}$ the set of non-zero entries among $\bx$. We extensively use the following notion of separatedness and boundedness in Section \ref{sec:helper_nets};
\begin{definition}[$\delta$ separation and boundedness of vectors] \label{notn:sdelta}
    Given a vector $\bx$, suppose that its non-zero entries satisfy the following: 
\begin{itemize}
    \item
    $x_i \in [\delta,1-\delta]$ and,
    \item 
    $\forall i \neq j ,\;\left| x_i - x_j\right| \geq \delta$.
\end{itemize}
We denote the set of all $\bx$ satisfying the above conditions by $\mathcal S^d_{\delta}$.
\end{definition}

\section{Depth $3$, width $\Ocal(d^{2})$ median computation}\label{theorem:depththreeconstruction_proof}

We first show in Proposition~\ref{thm:depth3sorting} how to correctly sort inputs of size $d$ using a quadratic width ``all pairs'' approach, provided the input is $\delta$-separated and bounded. We then show that for inputs from the uniform hypercube, the input will be $\delta$-separated with high probability for inverse-polynomial $\delta$, leading to Theorem~\ref{theorem:depththreeconstruction}, bounding the expected squared error of this neural network on random input.

The below proposition shows how a neural network to \emph{sort} its input, and we point out that we can trivially use this sorting network to return the median, by extracting the $d/2^{\textrm{th}}$ output.

\begin{proposition} \label{thm:depth3sorting}
    For any dimension $d>0$ and $\delta>0$, there exists a ReLU neural network $\mathcal N:\mathbb{R}^d \to \mathbb{R}^d$ of width $4d^2$, using 2 hidden layers, and magnitudes of weights $\leq \frac{1}{\delta}$, that sorts any entirely non-zero input $\bx \in \mathcal S^{d}_{\delta}$.
    For general input $\bx\in\mathbb{R}^d$, the return values are bounded by $d \norm{\bx}_{\infty}$.
\end{proposition}

\begin{proof}
     We compute the answer using the \emph{rank selection} neural network (Definition~\ref{defn:rank_selection_net}, $\Ncal^{RS}_{\delta}$) by plugging in $d'=d$ and $\br= (1,2, \ldots, d)$ and $p=d$. By Lemma~\ref{lemma:rank_selection_net_output}, since $\bx \in \mathcal S^d_{\delta}$ and is entirely non-zero, we have that the neural network outputs the elements of ranks $\br= (1,2, \ldots, d)$ in ascending order. From Lemma~\ref{lemma:rank_selection_net_output}, the network uses 2 hidden layers, has width $4d^2$, and the magnitudes of weights are bounded by $\frac{1}{\delta}$. 
    
    For all other cases, the outputs are upper bounded by $d \norm{\bx}_{\infty}$ from Lemma~\ref{lemma:rank_selection_net_output}.
\end{proof}


\subsection{Proof of Theorem~\ref{theorem:depththreeconstruction}}

\begin{proof} We use Proposition~\ref{thm:depth3sorting} to show that, when the input $\bx$ is $\delta$-separated and appropriately bounded, the \emph{rank selection} neural network (Definition~\ref{defn:rank_selection_net}, $\Ncal^{RS}_{\delta}$) will accurately compute the median. We combine this with an analysis of the expected squared error in the rare cases that the randomly chosen $\bx$ violates the assumptions of Proposition~\ref{thm:depth3sorting}.

From Lemma~\ref{lemma:nice_dist_imply_delta_separatedness}, plugging in $d'=d,\delta= \frac{\epsilon}{12d^4}$ we have $\pr \left[ \exists i, x_i=0 \; \text{or} \;\bx \not \in \mathcal S^d_{\delta}\right] \leq \frac{\epsilon}{4d^2}$. Denoting the event $\left[ \exists i, x_i=0 \; \text{or} \;\bx \not \in \mathcal S^d_{\delta}\right]$ as $F$ we can use Proposition~\ref{thm:depth3sorting} to conclude,
    \begin{equation*}
    \begin{split}
         \E_{\bx \sim  \Ucal([0,1]^d)} \left[ \left( \mathcal N(\bx) - \med(\bx)\right)^2 \right] =  \pr  \left[  F \right]\cdot    \E_{\bx \sim \Ucal([0,1]^d)} \left[ \left( \mathcal N(\bx) - \med(\bx)\right)^2 \big | F \right] \\
         + \pr  \left[  \overline{F} \right] \cdot \E_{\bx \sim \Ucal([0,1]^d)} \left[ \left( \mathcal N(\bx) - \med(\bx)\right)^2 \big | \overline{F} \right]  \leq \left(d \norm{\bx}_{\infty}+1 \right)^2 \cdot \frac{\epsilon}{4d^2} \leq  \epsilon
     \end{split}
     \end{equation*}
     Also, from Proposition~\ref{thm:depth3sorting} we have that the width of the neural network used is $\Ocal(d^2)$, with $2$ hidden layers with magnitude of weights bounded by $1/\delta=12d^{4}/\epsilon$ concluding the proof.
\end{proof}

\section{Depth $5$, width roughly $\Ocal(d^{5/3})$ median computation}\label{theorem:depth_5_median_approx_proof}

In this section, we present the depth $5$ construction for extracting the median from an input of size $d$, and show its output is correct with high probability. For a concrete statement on the width and depth requirement along with the probability of correctness, refer to Theorem~\ref{theorem:depth_5_median_approx}.

\subsection{Construction and auxiliary lemmas}

\SetKwFunction{median}{depth5MedianComputation} 
\begin{algorithm}[t]
\caption{Probabilistic algorithm for median approximation  (Superlinear Width).}\label{alg:depth5medianComp}
\KwIn{ Vector: $\bx \in (0,1)^d $ with unique and uniformly randomly permuted entries.}
\KwOut{$\med (\bx)$ with probability $\geq 1- \exp \p{- \Omega(d^{2 \gamma})}$.}
\median{$\bx, \gamma$}{:} \;
\BlankLine\;  
\begin{enumerate}
   
    \item \label{nnprob5:step1} Partition the $d$ elements into blocks of size $\ceil{d^{2/3}}$ and let $q$ be the number of such blocks, where the last block might have smaller size. Denote the $i^{\textrm{th}}$ block as $\bx_i$, for $i\in[q]$. For $i\in[q-1]$, create a new vector $\by_i$ containing the elements of $\bx_i$ with ranks belonging to the set $ \left\{ \floor{\frac{d^{2/3}}{2} - d^{\frac{1}{3} + \gamma}},\ldots,\ceil{\frac{d^{2/3}}{2}+d^{\frac{1}{3}  + \gamma} } \right\}$; and let $\by_q=\bx_q$, effectively selecting the entirety of the last block, for simplicity.

    \item \label{nnprob5:step2}  Consider the entries of the concatenation $(\by_1,\ldots,\by_q)$ and compare each with all the entries of $\bx$, counting the number of entries it is larger than; then output the median---namely, the element that wins $d/2-1$ comparisons. If no such entry exists, consider the algorithm to have failed.

\end{enumerate}
\end{algorithm}

\begin{definition}\label{defn:successful_depth5}
We call an execution of Algorithm~\ref{alg:depth5medianComp} ``successful'' if one of the shortlisted blocks $\by_i$, for $i \in [q]$, contains the median of $\bx$.
\end{definition}

\begin{proposition}  \label{prop:depth5medianComp} Algorithm~\ref{alg:depth5medianComp}, when given a uniform random permutation of a set $S \subset(0,1)$ of size $d$ represented in the algorithm by vector $\bx \in [0,1]^d$, is ``successful'' in the sense of Definition~\ref{defn:successful_depth5} with probability at least $1- \exp \left(-\Omega\p{d^{2 \gamma}}\right)$, and subsequently outputs the median $\med(\bx)$ when it is ``successful".
\end{proposition}

\begin{proof} Recall from Algorithm~\ref{alg:depth5medianComp} we partition $\bx$ into blocks $\bx_i$ of size $\leq \ceil{d^{2/3}}$ for $i\in [q]$, where $q \leq \ceil{d^{1/3}}$. Since the entries of $\bx$ are uniformly randomly permuted, the entries of any particular block $\bx_i$ are a uniformly random subset of size $\leq\ceil{d^{2/3}}$ of the $d$ entries. Defining $\bL =  \floor{d^{2/3}/2 - d^{1/3 + \gamma}} , \bR =\ceil{ d^{2/3}/2 + d^{1/3 + \gamma}}$, for each block $\bx_i$ for $i\in [p-1]$ we define $e_{i}^{-} = \mathcal R_{\bL}(\bx_i), e_{i}^{+} = \mathcal R_{\bR}(\bx_i)$. For a particular block $\bx_i$ we show that with high probability the number of elements less than or equal to $\med(\bx)$ in $\bx_i$ is in the range $[\bL,\bR]$; this implies that $e_{i}^{-} \leq \med (\bx) \leq e_{i}^{+}$.

We bound the number of elements $\leq \med(\bx)$ in $\bx_i$ via part 1 of Lemma~\ref{lemma:neg_assoc}: $\bx_i$ is a random subset of the elements of $\bx$ of, called $S$ in the context of Lemma~\ref{lemma:neg_assoc} and which has size $n=\ceil{d^{2/3}}$; and let $T$ denote those elements of $\bx$ that are $\leq \med(\bx)$, which has size $k=d/2$; let $\epsilon=d^{\frac{1}{3}+\gamma}-1$. Part 1 of Lemma~\ref{lemma:neg_assoc} says that the probability that $|S\cap T|$ has distance $\geq \epsilon$ from its expectation $\frac{(d/2)\ceil{d^{2/3}}}{d}$ is at most $2 e^{-2 \epsilon^2/\ceil{d^{2/3}}}$. Thus, except with probability $1-\exp(-\Omega(d^{2\gamma}))$ we have that the number of elements less than or equal to $\med(\bx)$ in $\bx_i$ is in the range $\frac{(d/2)\ceil{d^{2/3}}}{d}\pm \epsilon$, which is a subset of the range $[\bL,\bR]$, as desired.

Taking a union bound over all $q-1$ blocks for which we throw out elements, we conclude that, except with $(q-1) \exp(-\Omega(d^{2\gamma}))=\exp(-\Omega(d^{2\gamma}))$ probability, we will not throw out the median from any block. Thus since the median lies in \emph{some} block $\bx_i$, it must also lie in some block $\by_i$ for $i\in [q-1]$, except with probability $\exp(-\Omega(d^{2\gamma}))$.

Thus since Step 2 explicitly tests whether each element of the blocks $\by_i$ for $i\in[q]$ is the overall median, the algorithm will find and return the median as desired.

\end{proof}

\SetKwFunction{median}{depth5MedianComputation} 
\begin{algorithm}[t]
\caption{Given $\bx$ compute the median of the entries (Neural network construction of Algorithm~\ref{alg:depth5medianComp}).}\label{alg:depth5medianComp_const}
\KwIn{Entirely non-zero vector $\bx \in \mathcal S^d_{\delta}$, with uniformly randomly permuted entries.}
\KwOut{$\med (\bx)$ or some value $\leq d^2$.}
\median{$\bx, \gamma$}{:} \;
\BlankLine\;  
\begin{enumerate}
   
    \item \label{nnconstruction5:step1} Partition the $d$ elements into blocks of size $\ceil{d^{2/3}}$ and let $q$ be the number of such blocks, where the last block might have smaller size. Denote the $i^{\textrm{th}}$ block as $\bx_i$, for $i\in[q]$. For $i\in[q-1]$, create a new vector $\by_i$ containing the elements of $\bx_i$ with ranks belonging to the set $ \left\{ \floor{\frac{d^{2/3}}{2} - d^{\frac{1}{3} + \gamma}},\ldots,\ceil{\frac{d^{2/3}}{2}+d^{\frac{1}{3}  + \gamma} } \right\}$; and let $\by_q=\bx_q$, effectively selecting the entirety of the last block, for simplicity. 
    \begin{itemize}
        \item We do this in parallel for each of the $q$ blocks using \emph{rank selection} neural network (Definition~\ref{defn:rank_selection_net}, $ \Ncal_{\delta}^{RS}$). 
    \end{itemize}

    \item \label{nnconstruction5:step2} Consider the entries of the concatenation $(\by_1,\ldots,\by_q)$ and compare each with all the entries of $\bx$, counting the number of entries it is larger than; then output the median---namely, the element that wins $d/2-1$ comparisons. If no such entry exists, consider the algorithm to have failed.

    \begin{itemize}
        \item  We do this in parallel for each $\by_i, i\in[q]$ using \textit{modified} \emph{comparison} neural network (Definition~\ref{defn:comparison_net}, $ \Ncal_{\delta}^{C}$) and the \emph{indicator function product} neural network (Definition~\ref{defn:indicator_net}, $ \Ncal^{IFP}$). 

    \end{itemize}

\end{enumerate}
\end{algorithm}

\begin{proposition} \label{prop:depth5medianCompNN}
    For any dimension $d>0$ and any $\delta>0$, if for some entirely non-zero vector $\bx \in \mathcal S^d_{\delta}$ Algorithm~\ref{alg:depth5medianComp} is ``successful'' in the sense of Definition~\ref{defn:successful_depth5}, then the ReLU neural network outlined by Algorithm~\ref{alg:depth5medianComp_const} successfully implements Algorithm~\ref{alg:depth5medianComp} and thereby returns $\med(\bx)$, else, for all input $\bx \in [0,1]^d$ it outputs a value within the interval $[-d^2,d^2]$. Moreover, this neural network has width $\mathcal O \left(d^{5/3+\gamma} \right)$ with 4 hidden layers, and with the magnitude of weights upper bounded by $ 1/\delta$.
\end{proposition}

\begin{proof} We prove the correctness of the neural network presented in Algorithm~\ref{alg:depth5medianComp_const} in implementing Algorithm~\ref{alg:depth5medianComp} assuming the ``success" of Algorithm~\ref{alg:depth5medianComp} (Definition~\ref{defn:successful_depth5}), on an entirely non-zero input $\bx \in \mathcal S^d_{\delta}$. Specifically, will show how to implement steps 1 and 2 of Algorithm~\ref{alg:depth5medianComp}  using a ReLU neural network of width $\Ocal \p{d^{5/3+\gamma}}$ and depth $5$ in the corresponding steps \hyperref[nnconstruction5:step1]{(1,2)} of Algorithm~\ref{alg:depth5medianComp_const}.
    
    \begin{itemize}
     \item \textbf{Step 1:} Informally, in this step we partition $\bx$ into $q \leq \frac{d}{\ceil{d^{2/3}}} \in \Ocal(d^{1/3})$ blocks and select entries with  ranks (with respect to the block) in the set $\left\{ \floor{\frac{d^{2/3}}{2} - d^{\frac{1}{3} + \gamma}},\ldots,\ceil{\frac{d^{2/3}}{2}+d^{\frac{1}{3}  + \gamma} } \right\}$ from each block. Step \hyperref[nnconstruction5:step1]{(1)} is implements this step using the \emph{rank selection} neural network. Subsequently we let $\left|\left\{ \floor{\frac{d^{2/3}}{2} - d^{\frac{1}{3} + \gamma}},\ldots,\ceil{\frac{d^{2/3}}{2}+d^{\frac{1}{3}  + \gamma} } \right\}\right| = p' \in \Ocal(d^{1/3+\gamma})$. Since $\bx \in \mathcal S^d_{\delta}$ and is entirely non-zero it follows that $\forall i \in [q-1], \; \bx_i \in \mathcal S^{\ceil{d^{2/3}}}_{\delta}$ and is entirely non-zero, and using Lemma \ref{lemma:rank_selection_net_output} (with $d'=\ceil{d^{2/3}}, p=p'$) we conclude that \emph{rank selection} neural network (Definition~\ref{defn:rank_selection_net}, $ \Ncal_{\delta}^{RS}$) correctly computes this step for each block $i$. The $q$'th block is propagated unmodified.
     
     We implement this step for each block $i \in [q-1]$ in parallel and again using Lemma~\ref{lemma:rank_selection_net_output} with $d'=\ceil{d^{2/3}}, p=p'$ we conclude that the \emph{rank selection} neural network (Definition~\ref{defn:rank_selection_net}, $ \Ncal_{\delta}^{RS}$) uses 2 hidden layers and width of $\mathcal O \left(q \cdot d^{4/3}\right) \in \mathcal O \left( d^{5/3}\right)$. Also, from Lemma~\ref{lemma:rank_selection_net_output} we have that the magnitude of weights used by this layer is  $ 1/\delta $. 

     \item \textbf{Step 2} Intuitively, in Step \hyperref[nnconstruction5:step2]{(2)} we aim to compare all the shortlisted entries, denoted by $\by_i, \forall i \in [q]$ with the entries of $\bx$ and return the element which is larger than $d/2-1$ entries of $\bx$. For an entry $y$ in $\by=(\by_1, \cdots, \by_q)$, since $y$ is also a non-zero entry of $\bx \in \mathcal S^d_{\delta}$ we can compute using the \emph{comparison} neural network (Definition~\ref{defn:comparison_net}, $\Ncal^C_{\delta}$) to count the number of entries in $\bx$ it is larger than. Formally, we can compute $\sum_{j} \mathbbm{1} \{ y > x_j\}$ correctly, by using the neural network defined as $\sum_{j} \Ncal^C_{\delta}(y,x_j) $, requiring a width of $2d$ and $1$ hidden layer. We do this in parallel for the all the entries of $\by$, and the number of such entries is $\Ocal(p'q) \in\Ocal \p{ d^{2/3+\gamma}}$. Hence, by Fact~\ref{fact:comparison_net} this comparison step requires a width of $\Ocal \p{d \cdot d^{2/3+\gamma}} \in \Ocal \p{d^{5/3+\gamma}}$ and $1$ hidden layer. Finally we use the \emph{indicator function product} neural network (Definition~\ref{defn:indicator_net}) to shortlist the element which wins $d/2-1$ comparisons by computing $\sum_{k} \Ncal^{IFP}\p{y_k,\sum_{j} \Ncal^C_{\delta}(y_k,x_j) -(d/2-1) }$. Namely, for each $k$ the expression $\sum_{j} \Ncal^C_{\delta}(y_k,x_j)$ counts the number of elements of $\bx$ that are smaller than $y_k$; we then compare this number to $d/2-1$ and the $\Ncal^{IFP}$ network outputs the single $y_i$ for which the count matches, by Lemma~\ref{lem:4-term}. Counting the width, depth, and weights used, we see that this step uses width $\Ocal(d^{5/3+\gamma})$, 2 hidden layers, and the magnitudes of weights are bounded by $1/\delta$.
   
    \end{itemize}

    From our assumption that Algorithm~\ref{alg:depth5medianComp} is ``successful" we have that there exists a block $\by_i$ that contains $\med(\bx)$. Combining this with Step 2, where the overall ranks (with respect to $\bx$) of every entry in  $(\by_1, \cdots, \by_q)$ are correctly computed, and the element which wins $d/2-1$ comparisons is returned: we have that the neural network construction in Algorithm~\ref{alg:depth5medianComp_const} correctly computes the median of $\bx$. Thus, overall the construction requires 4 hidden layers, width of $\Ocal(d^{5/3+\gamma})$ and magnitudes of weights bounded by $1/\delta$.

    Otherwise, for arbitrary input $\bx\in [0,1]^d$, we explicitly bound the magnitude of the outputs. From Lemma~\ref{lemma:rank_selection_net_output}, $\norm{\bx'}_{\infty} \leq d\norm{\bx}_{\infty} $ i.e., the output of the first two layers of the neural network construction is bounded by $d\norm{\bx}_{\infty}$. Using Lemma~\ref{lem:4-term} on this output, we have that the final output is bounded by $ d \norm{\bx'}_{\infty}  \leq d^2 \norm{\bx}_{\infty} \leq d^2$, since $\bx \in [0,1]^d$.
\end{proof}

\subsection{Proof of Theorem~\ref{theorem:depth_5_median_approx}} 
\begin{proof} 
We use Proposition~\ref{prop:depth5medianCompNN} to show that, when the input $\bx$ is $\delta$-separated and appropriately bounded, the neural network outlined in Algorithm~\ref{alg:depth5medianComp} will accurately compute the median with probability at least $1-\exp(-\Omega(d^{2\gamma}))$ (where recall  increasing $\gamma$ will increase the width of our neural network), and always returns values in $[-d^2,d^2]$. We combine this with the bounds on the probability that a randomly chosen input will fail the input requirements of Proposition~\ref{prop:depth5medianCompNN}: from Lemma~\ref{lemma:nice_dist_imply_delta_separatedness}, plugging in $d'=d,\delta\coloneqq\frac{\epsilon}{12d^6}$ we have that $\pr_{\bx \sim \mathcal \Ucal([0,1]^d)} \left[ \bx \not \in \mathcal{S}^d_{\epsilon/12d^6}\right] \leq \epsilon/4d^4$. Thus, from the union bound on these two failure modes,
\[\pr_{\bx \sim \mathcal \Ucal([0,1]^d)} \left[ \Ncal(\bx) \neq \med(\bx)\right] \leq \frac{\epsilon}{4d^4} + \exp \left( -  d^{\Omega(1)} \right).\]

In the cases that $\Ncal(\bx) \neq \med(\bx)$, since the true median has a range of $[0,1]$ and our algorithm's returned answer is in the range $[-d^2,d^2]$ we have $$ \left( \Ncal(\bx) - \med(\bx)\right)^2 \leq 4d^4$$. Thus the mean squared error can be bounded as, 
    \[
        \Exp_{\bx \sim \mathcal{U} \left([0,1]^d\right)} \left[ \left( \mathcal N(\bx) - \med(\bx)\right)^2 \right] \leq 4d^4 \cdot \p{\frac{\epsilon}{4d^4} + \exp \left( - \Omega \left( d^{2\gamma}\right) \right)} \leq \epsilon +  \exp \left( - \Omega \left( d^{2\gamma}\right) \right)
    \]

concluding the proof of the first part of the theorem.

    The neural network construction outlined in Proposition~\ref{prop:depth5medianCompNN} has $4$ hidden layers, width of $\Ocal(d^{5/3+\gamma})$ and magnitudes of weights bounded by $1/\delta = 12d^{6}/\epsilon$ concluding the proof.
\end{proof}

\section{Depth $46$, width $\Ocal(d)$ median computation} \label{thm:linear_width_med_computation_distribution_zeroone_proof}

In this appendix, we present the proof of Theorem~\ref{thm:linear_width_med_computation_distribution_zeroone}. This appendix section is organized as follows:
  
In the initial sparsification step, we present our construction as an algorithm (Algorithm~\ref{alg:informalAlgo}) and analyze the probabilistic properties of its data stream in Proposition~\ref{prop:alg_prob_analysis}. Informally, this proposition demonstrates that we can sparsify the input vector with high probability while preserving the essential probabilistic properties of the input distribution required for the proof to succeed. Subsequently, we show that for $\delta$-separated and bounded inputs, this algorithm can be implemented via the neural network outlined in Algorithm~\ref{alg:sparsifyingFunctionNeuralNet} (Proposition~\ref{prop:linear_sparsifying_NN}), achieving constant depth and linear width. This implementation employs multiple sub-architectures, each serving a distinct role in the median extraction pipeline. These components are rigorously defined and analyzed in Appendix~\ref{sec:helper_nets}.

In the second and final step, we utilize a deterministic hashing construction to reduce the dimensionality of the sparsified vector to $\mathcal{O}(\sqrt{d})$. This reduction allows us to compute the \emph{median} via brute-force comparisons, a process that requires quadratic width relative to the reduced dimension, yet remains linear with respect to the original dimension $d$. We outline this step in Algorithm~\ref{alg:medianFindingFunctionLinear} (Proposition~\ref{prop:complete_algo_analysis}) and conclude this appendix section with the proof of Theorem~\ref{thm:linear_width_med_computation_distribution_zeroone}.

\subsection{Sparsification step} 

The goal of this step is to obtain a sparse vector $\bx'$ from an non-zero-entry vector $\bx$ that still maintains the median with high probability, with respect to the randomness induced by the input distribution (see step 1 in Algorithm~\ref{alg:informalAlgo}). Subsequently, we show that when the entries of $\bx$ are bounded and $\delta$-separated, the desired steps can be implemented using some neural network architecture (see implementation details in Proposition~\ref{prop:linear_sparsifying_NN}).

We start with a vector $\bx$, and sample a random subset of it. Using the median of this random subset as a reference, we zero out all the entries of $\bx$ whose values do not fall within a certain radius around this reference median. This results in a significant reduction with high probability in the number of non-zero entries, while keeping maintaining the median. Crucially, we show that we can repeat this process by repeatedly taking a subset of the remaining non-zero entries in a deterministic manner that succeeds with high probability with respect to the initial randomness of input $\bx$. For some specific choice of parameters, we show that with high probability, we can reach our desired sparsity within a constant number of iterations of the above scheme, where the median is preserved at each iteration.


\SetKwFunction{median}{probabilisticAlgorithm} 
\begin{algorithm}[t]
\caption{Probabilistic algorithm for median approximation  (Linear Width)}\label{alg:informalAlgo}
\DontPrintSemicolon
\KwIn{ Vector: $\bx_1 \in (0,1)^d $ with unique and uniformly randomly permuted entries}
\KwOut{Modified Vector $\bx_5$}
\median{$\bx_1$}{:} \;
\BlankLine\; 
\;  
\begin{enumerate}
    \item For $i\in\{1,\ldots,4\}$ define the triples of exponents $(y'_i,z'_i,w'_i)= (1, 0.5, 0.26), (0.76, 0.5, 0.26), (0.52, 0.5, 0.26), (0.28, 0.26, 0.14)$, and also let $y'_5=0.16$; then define the parameters $z_i=\lceil d^{z'_i}\rceil, w_i= \lfloor d^{w'_i}\rfloor, y_i=d^{y'_i}$.

    \item For $i \in \{1,\ldots,4\}$ do: 
    \begin{enumerate}

\item \label{informalalg:stepa} From $\bx_{i}$ we pick $z_i$ non-zero entries as follows: if $i=1$ then take the first $z_i$ entries. Otherwise for $i>1$, divide $\bx_{i}$ into blocks of size $\lceil  3^{i} d^{1.01}/y_{i}\rceil$, and from the first $\lceil z_i/d^{0.01}\rceil$ such blocks, choose the first $\lceil d^{0.01}\rceil$ non-zero elements if possible (and FAIL otherwise), but discard elements after $z_i$ total elements have been chosen. Denote this sample set of size $z_i$ by $S_i$.

\item \label{informalalg:stepb}  Define $r_i$ to be the rank in $\bx_i^{\neq 0}$ of the overall median, $\med(\bx_1)$ (see Lemma~\ref{lemma:rank_computing_net_output} for how we compute $r_i$ with a neural network, without knowing the median). Let $\bL_{i} =  \frac{r_{i}}{|\bx_i^{\neq 0}|}z_i-w_i$, and $\bR_{i}= \frac{r_{i}}{|\bx_i^{\neq 0}|}z_i+w_i$. Let $e^-_i = \mathcal R_{\lfloor\bL_{i}\rfloor} (S_i)$ if $\bL_i\geq 1$ and $e^-_i=0$ otherwise; and let $e^+_i = \mathcal R_{\lceil\bR_{i}\rceil} (S_i)$ if $\bL_i\leq z_i$ and $e^+_i=1$ otherwise.

\item \label{informalalg:stepc} Create a new copy of $\bx_{i}$ represented as $\bx_{i+1}$, where in $\bx_{i+1}$ all entries with values $> e_{i}^{+}$ or values $< e_{i}^{-}$ are changed to $0$.
        
\end{enumerate}
\end{enumerate}
\end{algorithm}

This procedure is described in Algorithm~\ref{alg:informalAlgo}. To analyze this algorithm, we will focus on the inner ``for'' loop (by using mathematical induction), and show that, with probability at least $1- \exp \left(-d^{\Omega(1)}\right)$ over the uniformly distributed entries of its input, the algorithm is ``successful'' in the following sense.

\begin{definition}\label{defn:successful_linear}
We call an execution of Algorithm~\ref{alg:informalAlgo} ``successful'' if at the end of each iteration $i$, we have the following three properties:
    \begin{enumerate}
    \item $\bx_i^{\neq 0}$ consists of \textbf{contiguous} elements from the sorted version of $\bx_{1}$.
          \item \label{hyp:i} The overall median, $\med(\bx_1)$, is one of the non-zero entries of $\bx_{i+1}$.
          \item \label{hyp:ii} The number of non-zero entries in $\bx_{i+1}$ has the bounds $ y_{i+1} 3^{-i}\leq |\bx_{i+1}^{\neq 0}| \leq y_{i+1} 6^{i}$.

     \end{enumerate}
\end{definition}

\begin{proposition} \label{prop:alg_prob_analysis} Algorithm~\ref{alg:informalAlgo}, for any set $S \subset(0,1)$ of size $d$, represented in the algorithm by a vector $\bx_1 \in (0,1)^d$, is ``successful'' in the sense of Definition~\ref{defn:successful_linear} with probability at least $1- \exp \left(-d^{\Omega(1)}\right)$ over random permutations of $\bx_1$.
\end{proposition}

 If the final loop iteration is ``successful'', this means that the algorithm outputs a vector $\bx_5$  with at most $ y_5 \cdot 6^{4}$ non-zero entries, and where the overall median $\med(\bx_1)$ is one of them; but we instead show the stronger property that \emph{all} loop iterations $i\in\{1,\ldots,4\}$ are successful with high probability. We use this stronger property below to show how to implement Algorithm~\ref{alg:informalAlgo} with a neural network, in Proposition~\ref{prop:linear_sparsifying_NN}.
\begin{proof}
    The proof is by induction on the loop variable $i$, from $1$ to $4$. However, since we have a probabilistic input, we will show that the induction step holds with high probability, and then use a union bound at the end, adding up the failure probabilities of the 4 steps.

    The algorithm performs a few major steps which each give correct results with probability $\geq 1- \exp \left(-d^{\Omega(1)}\right)$. 

    Informally, iteration $i$ considers the vector of non-zero entries $\bx_i^{\neq 0}$ and looks for an element of some desired rank $r_i$ among these; it does this by taking a random sample $S_i$ of size $z_i\ll |\bx_i^{\neq 0}|$, and then looking for the element of proportionate rank $r_i\frac{z_i}{|\bx_i^{\neq 0}|}$ in $S_i$, throwing out all elements significantly smaller or larger then this. We keep track of how many elements bigger than the median we throw out, so that we exactly know the desired rank we seek among the remaining elements. And then we repeat this process at the next loop iteration, on a smaller input. (Technically, we zero-out entries instead of throwing them out, since zeroing out elements is a natural neural network operation. Think of 0 entries as being ``invisible'' to the algorithm.)
    
    We start by stating our induction hypothesis, which is identical to the notion of ``success'' from Definition~\ref{defn:successful_linear}, but with the addition of one more property, property 0 below that follows easily from the structure of the algorithm, and which we prove first. We will prove the following induction hypothesis for $i\in \{1,\ldots,5\}$.

    \paragraph{Induction hypothesis}

    \begin{enumerate}\setcounter{enumi}{0}
        \item \label{hyp:1} $\bx_i^{\neq 0}$ is the intersection of $\bx_{1}$ with some real interval, which we denote $I_i$.
          \item \label{hyp:2} The overall median, $\med \left( \bx_1 \right)$ is one of the non-zero entries of $\bx_{i}$, and we let $r_{i}$ denote its rank among the non-zero entries of $\bx_{i}$.
          \item \label{hyp:3} The number of non-zero entries in $\bx_{i}$ has the bounds $3^{-(i-1)} y_i \leq |\bx_i^{\neq 0}| \leq 6^{i-1} y_i $.

     \end{enumerate}

     \paragraph{Base case}
     The base case, $i=1$, of the induction hypothesis trivially holds: 1) $\bx_1$ is trivially the intersection of $\bx_1$ with $[0,1]$; 2) the median of $\bx_1$ is trivially in $\bx_1$; and 3) since we defined $y_1=d$, we trivially have that $ 3^{-0}y_1 \leq |\bx_1^{\neq 0}| \leq 6^{0} y_1 $.

     \paragraph{Induction step} 

    \begin{itemize}
        \item We analyze loop $i$ of the algorithm, assuming the induction hypothesis, and show that, at the end of loop $i$, the $i+1$ version of the induction hypothesis will hold with high probability.
        
        \item We first prove property 0 of the induction hypothesis. By the induction hypothesis, the non-zero elements of $\bx_i$ are the intersection of $\bx_1$ with some real interval $I_i$. Recall that the non-zero elements of $\bx_{i+1}$ are defined, in step  \hyperref[informalalg:stepc]{(2c)}, to be exactly the non-zero elements of $\bx_{i}$ that lie in the real interval $[e_i^-,e_i^+]$. Thus the non-zero elements of $\bx_{i+1}$ are exactly the elements of $\bx_1$ that lie in the intersection of these real intervals, $[e_i^-,e_i^+]\cap I_i$, which is itself a real interval, which we denote $I_{i+1}$, proving this property of the induction.
        
        \item As a direct consequence of induction hypothesis  \hyperref[hyp:1]{1} of the induction hypothesis, we point out that $\bx_i$ must consist of a subset of $\bx_1$ of \emph{contiguous ranks} in $\bx_1$, with these ranks comprises some interval of integers $\{r^-_i,\ldots,r^+_i\}$. We use this property crucially in the proof below.
        
        \item The induction hypothesis states that $\bx_i$ contains the median, and also that the number of non-zero elements in $\bx_i$ is in a certain interval. We reexpress both conditions in terms of  $r^-_i,r^+_i$: we have that $\frac{d}{2}\in \left [r^-_i,r^+_i\right]$ and that $r^+_i-r^-_i+1 \in \left[3^{-(i-1)}y_i ,6^{i-1} y_i \right]$.

        \item For each of the $\leq d^2$ potential values of the ranks $r^-_i,r^+_i\in\{1,\ldots,d\}$, we will separately show that the induction step succeeds with high probability (with respect to the uniformly random permutation of $\bx_1$). From now on, \textbf{fix} a particular choice of ranks $r^-_i,r^+_i$ that satisfies the two conditions $\frac{d}{2}\in \left [r^-_i,r^+_i\right]$ and that $r^+_i-r^-_i+1 \in \left[3^{-(i-1)} y_i ,6^{i-1} y_i \right]$. 
         We use the generic probability fact that for two events $A,B$ we have $\pr[A,B]\leq \pr[A|B]$. Specifically, the probability that A) the induction step fails, and B) the $\bx_i$ used in iteration $i$ consists of those elements of $\bx_1$ with ranks in the interval $[r^-_i,r^+_i]$, is at most the probability of the induction step failing \emph{given} that we start the algorithm in iteration $i$, setting $\bx_i$ to be those elements of $\bx_1$ with ranks in $[r^-_i,r^+_i]$.
        
        \item We first analyze step \hyperref[informalalg:stepa]{(2a)} to show that the algorithm does not FAIL in this step (except with probability $\exp(-d^{\Omega(1)})$, with respect to a random permutation of the input $\bx_1$). For $i=1$ the claim is trivially true as the input $\bx_1$ is entirely non-zero by assumption and the algorithm explicitly takes $S_1$ to be the first $z_i$ elements.

        Otherwise, for $i>1$, we have from induction hypothesis \hyperref[hyp:3]{3} that the number of non-zero entries in $\bx_i$ is  $\geq  3^{-(i-1)} y_i$. From our analysis, recall that $\bx_i$ consists of those elements of $\bx_1$ with ranks in $[r_i^-,r_i^+]$, and probabilities are always taken with respect to random permutations of $\bx_1$. Thus the location of non-zero entries of $\bx_i^{\neq 0}$ will be uniformly random. Thus, in each block of size $\lceil 3^{i}  d^{1.01}/y_i\rceil$ the expected number of non-zero entries is thus at least $3\cdot d^{0.01}$, of which we aim to choose the first $\lceil d^{0.01}\rceil$, if they exist. We analyze this existence probability via part 2 of Lemma~\ref{lemma:neg_assoc}: we succeed if the number of non-zero entries is within a factor of 2 of its expectation (for large enough $d$); and part 2 of Lemma~\ref{lemma:neg_assoc} says the probability of this failing is exponentially small in the expectation itself, namely $\exp(-d^{\Omega(1)})$ as desired.

        \item Next we show that $S_i$ is a uniformly random subset (of size $z_i$) of the set of entries of $\bx_1$ with ranks in $[r^-_i,r^+_i]$. Recall that, by the set up our analysis, we fix ranks $r^-_i,r^+_i$, permute $\bx_1$ uniformly at random, and let $\bx_i$ set to 0 those elements of $\bx_1$ whose ranks are not in the interval $[r^-_i,r^+_i]$. In step \hyperref[informalalg:stepa]{(2a)} we choose $S_i$ to be a portion of $\bx_i$, attempting to choose non-zero entries from certain blocks. Importantly, considering different permutations of $\bx_1$, the choice of locations of $\bx_i$ that are chosen for $S_i$ depends \emph{only} on whether those locations in $\bx_i$ are non-zero, which thus depends only on whether those locations in $\bx_1$ have ranks in $[r^-_i,r^+_i]$. Namely, step \hyperref[informalalg:stepa]{(2a)} chooses $S_i$ in a way that is  \emph{unaffected} by permuting the elements of rank $[r^-_i,r^+_i]$ in $\bx_1$. Thus all subsets of $\bx_i^{\neq 0}$ of a given size $z_i$ are \emph{equally} likely to be chosen as $S_i$. Thus, conditioned on the algorithm not FAILing (as analyzed in the previous paragraph), we conclude that $S_i$ must be a uniformly random subset of size $z_i$ of the elements of rank $[r^-_i,r^+_i]$ in $\bx_1$.

        \item We now analyze steps \hyperref[informalalg:stepb]{(2b)}, \hyperref[informalalg:stepc]{(2c)} of the algorithm, using the conclusions from above that: the non-zero entries of $\bx_i$ consist only of elements of ranks $[r^-_i,r^+_i]$, and $S_i$ is a uniformly random subset of these elements of size $z_i$.
        
        \item In step \hyperref[informalalg:stepc]{(2c)} we zero out those elements of $\bx_i$ that are smaller than $e_i^-$ or larger than $e_i^+$; we want to show that we do \emph{not} zero out the median, $\med(\bx_1)$.
        
        Recall that the rank of the median in $\bx_i^{\neq 0}$  was defined to be $r_i$. We thus apply part 1 of Lemma~\ref{lemma:neg_assoc} with $S=S_i$ of size $n=z_i$, and letting $T$ equal the set of elements $\leq med$, which has size $k=|T|=r_i$, and letting $\epsilon=w_i$. We conclude that with probability $\geq 1-2e^{-2w_i^2/z_i}$, the rank of the median in $S_i$ is strictly between $\bL_i=\frac{r_i z_i}{|\bx_i^{\neq 0}|}- w_i$ and $\bR_i=\frac{r_i z_i}{|\bx_i^{\neq 0}|}+ w_i$.

        Thus the median is at least the element of rank $\lfloor\bL_i\rfloor$ in $S_i$ (if there exists an element of that rank, and otherwise the median is at least $0$); this is exactly the condition that $\med\geq e^-_i$ defined in step \hyperref[informalalg:stepc]{(2c)}. In the other direction, the median is thus at most the element of rank $\lceil\bR_i\rceil$ in $S_i$ (if there exists an element of that rank, and otherwise $\med\leq 1$), which means that $\med\geq e^-_i$. Thus, overall, we have shown that with probability $\geq 1-2e^{-2w_i^2/z_i}$, this step will not throw out the median. We have chosen $w_i,z_i$ so that $w_i^2/z_i=d^{\Omega(1)}$, leading to the desired exponentially small failure probability for part  \hyperref[hyp:2]{2} of the induction hypothesis.

        \item We now prove part 2 of the induction step. 
        We aim to apply Lemma~\ref{lem:inverse-process} to bound $|\bx_{i+1}^{\neq 0}|$. In terms of Lemma~\ref{lem:inverse-process}, the universe $U=\bx_i^{\neq 0}$, the random subsample of this is $S=S_i$, and the algorithm selects a real interval $I=I_{i+1}$. We first claim that $|S\cap I|\in [w_i,3w_i]$ (for large enough $d$). By definition, $\bx_{i+1}^{\neq 0}=S\cap I$ contains all elements of $S$ whose ranks are between $\floor{ \frac{r_i z_i}{|\bx_i^{\neq 0}|}-w_i}$ and $\ceil{ \frac{r_i z_i}{|\bx_i^{\neq 0}|}+w_i}$; and the number of such ranks is clearly at most $2w_i+3$, and at least $w_i$---since $\frac{r_i z_i}{|\bx_i^{\neq 0}|}\in (0,z_i]$, so the center of the rank interval, when rounded up, is a valid rank of $S$. For large enough $d$, the interval $[w_i,2w_i+3]$ is trivially contained in $[w_i,3w_i]$.

We thus invoke Lemma~\ref{lem:inverse-process} to conclude that with except with probability $\exp(-d^{\Omega(1)})$ we have $|\bx_{i+1}^{\neq 0}|\in \left(\frac{1}{2} w_i \frac{|\bx_{i}^{\neq 0}|}{z_i}, 2\cdot 3\cdot w_i \frac{|\bx_{i}^{\neq 0}|}{z_i}\right)$. Since by construction of $y_i,y_{i+1}$ we have $\frac{w_i y_i}{z_i}\in [\frac{2}{3} y_{i+1},y_{i+1}]$ for large enough $d$, and using part 2 of the induction hypothesis $3^{-(i-1)} y_i \leq |\bx_i^{\neq 0}| \leq 6^{i-1} y_i $, we conclude that $|\bx_{i+1}^{\neq 0}|\in [3^{-i} y_{i+1}, 6^i y_{i+1}]$, proving the induction step with the desired high probability.

        \item In conclusion, for each of the $\leq d^2$ choices of ranks $r_i^-,r_i^+$, we have shown that the induction step fails with probability $\exp \left( - d^{\Omega(1)} \right)$. We additionally take the union over all 4 iterations of the induction, absorbing $4d^2$ into the asymptotic notation, to yield our desired failure probability of $ \exp \left( - d^{\Omega(1)} \right)$.
    \end{itemize}

    Thus the guarantees in the proposition, as given by hypotheses \hyperref[hyp:i]{1}, \hyperref[hyp:ii]{2}, \hyperref[hyp:ii]{3}, hold with high probability.
    
\end{proof}

\begin{algorithm}[t]
\SetKwFunction{median}{sparsifyingFunction} 
\caption{Neural Network Implementation of Algorithm~\ref{alg:informalAlgo}
}\label{alg:sparsifyingFunctionNeuralNet}
\KwIn{ Entirely non-zero vector $\bx_1 \in \mathcal S^d_{\delta}$, with uniformly randomly permuted entries.}
\KwOut{Vector $\bx' \in \mathbb{R}^d$. }
\DontPrintSemicolon
\median{$\bx_1$}{:} \;
\BlankLine\; 
\;  
\begin{enumerate}
    \item For $i\in\{1,\ldots,4\}$ define the triples of exponents $(y'_i,z'_i,w'_i)= (1, 0.5, 0.26), (0.76, 0.5, 0.26), (0.52, 0.5, 0.26), (0.28, 0.26, 0.14)$, and also let $y'_5=0.16$; then define the parameters $z_i=\lceil d^{z'_i}\rceil, w_i= \lfloor d^{w'_i}\rfloor, y_i=d^{y'_i}$.
    
    \item For $i \in \{1,2,3,4 \}$ do:
    \begin{enumerate}
        \item \label{nnconstruction:stepa} From $\bx_{i}$ we pick $z_i$ non-zero entries as follows: if $i=1$ then take the first $z_i$ entries. Otherwise for $i>1$, divide $\bx_{i}$ into blocks of size $\lceil  3^{i} d^{1.01}/y_{i}\rceil$, and from the first $\lceil z_i/d^{0.01}\rceil$ such blocks, choose the first $\lceil d^{0.01}\rceil$ non-zero elements if possible (and FAIL otherwise), but discard elements after $z_i$ total elements have been chosen. Denote this sample set of size $z_i$ by $S_i$.
        \begin{itemize}
            \item We implement this with our \emph{non-zero element shortlisting} (Definition~\ref{defn:shortlist_non-zero_net},  $ \Ncal_{\delta}^{NZES}$) where we do the shortlisting operation from each block in parallel. 
        \end{itemize}

    \item \label{nnconstruction:stepb} Define $r_i$ to be the rank in $\bx_i^{\neq 0}$ of the overall median, $\med(\bx_1)$.
    Compute $\bL_{i} =  \frac{r_{i}}{|\bx_i^{\neq 0}|}z_i-w_i$, and $\bR_{i}= \frac{r_{i}}{|\bx_i^{\neq 0}|}z_i+w_i$. Finally compute $\floor{\bL_i}$ and $\ceil{\bR_i}$. 
    \begin{itemize}
        \item  First we compute $r_i$ using the \emph{rank computing} neural network (Definition~\ref{defn:rank_computing_net}, $ \Ncal_{\delta}^{RC}$) using the first element of $S_i$ as a non-zero entry of $\bx_i$.
        \item After that we scale the rank using the \emph{rank scaling} neural network (Definition~\ref{defn:rank_scaling_net},  $\Ncal^{RSC}_{\delta,z_i}$) to compute $\frac{r_{i}}{|\bx_i^{\neq 0}|}z_i$. 
        \item Finally we compute $\floor{\bL_i}$ and $\ceil{\bR_i}$ using the \emph{floor} neural network (Definition~\ref{defn:ceil_net}, $\Ncal^{CEI}_d$) and then compute $\bL_i' = \max(\floor{\bL_i}+1, 1)$ and $\bR_i'=\min(\ceil{\bR_i}+1, z_i+2)$ using the \emph{maximum} neural network(Definition~\ref{defn:max_net}, $\Ncal^{MAX}$).
        \item For $i=0$, the above steps are redundant and we use we use pre-computed values as $r_i=d/2, |\bx_{i}^{\neq 0}|=d$. 
        \item \label{nnconstruction:stepc} Then we extract $e_i^{+} = \mathcal R_{\bR'_i}(S_i')$ and $e_i^{-}=\mathcal R_{\bL'_i}(S_i')$ where $S_i' = S_i \cup \{0,1\}$.
      We do the above operation using the \emph{rank selection} neural network (Definition~\ref{defn:rank_selection_net}, $ \Ncal_{\delta}^{RS}$).

    \end{itemize}

     \item \label{nnconstruction:stepd} Create a new copy of $\bx_{i}$ represented as $\bx_{i+1}$, where in $\bx_{i+1}$ all entries with values $< e_{i}^{-}$ or values $> e_{i}^{+}$ are changed to $0$. 
     \begin{itemize}
        \item We do this using our \emph{filtering} neural network (Definition~\ref{defn:filtering_net}, $ \Ncal_{\delta}^{F}$).
     \end{itemize}
     
 \end{enumerate}
 \item Return $\bx' \leftarrow \bx_5 $ which contains at most $6^4 d^{0.16}$ non-zero entries with high probability one of which is $\med(\bx_1)$.
\end{enumerate}

\end{algorithm}

\begin{proposition}\label{prop:linear_sparsifying_NN}
    For any $d>0$ and $\delta>0$ if for some entirely non-zero vector $\bx_1 \in \mathcal S^{d}_{\delta}$ Algorithm~\ref{alg:informalAlgo} is ``successful'' in the sense of Definition~\ref{defn:successful_linear}, then the ReLU neural network outlined by Algorithm~\ref{alg:sparsifyingFunctionNeuralNet} faithfully implements Algorithm~\ref{alg:informalAlgo} on input $\bx_1$ and thereby returns a  sparse vector with at most $6^4 d^{0.16}$ non-zero entries with the true median, $\med(\bx_1)$, as one of its non-zero entries. Moreover, this neural network has width $\mathcal O(d)$ and depth $34$ with magnitude of weights bounded by $\Ocal(\max(d^{1.5}, 1/\delta))$. 
\end{proposition}
\begin{proof} We will show that for each loop iteration $i$, and for each step  \textbf{2(a)} through \textbf{2(c)} of Algorithm~\ref{alg:informalAlgo} in loop $i$, we can faithfully implement this step using a ReLU neural network of width $\mathcal O(d)$ and depth $\mathcal O(1)$ in the corresponding steps \hyperref[nnconstruction:stepa]{(2a,b,c)} of Algorithm~\ref{alg:sparsifyingFunctionNeuralNet}. Our analysis relies on the conditions that $\bx_i\in \mathcal S^d_\delta$ (treated as an induction hypothesis), and the ``success'' conditions of Definition~\ref{defn:successful_linear}, which are assumed in this proposition. As a base case, the input $\bx_1\in \mathcal S^d_\delta$ by assumption.

\begin{itemize}
        \item \textbf{Parameter Properties:} \label{line:parameter_properties} The choice of the above parameters will satisfy the following inequalities for $i\in \{ 1,2,3,4\}$:

    \begin{enumerate}
        \item $1.01 + z_i'- y'_i< 1 \quad $ 
        \item $2 z'_{i} \leq 1$
    \end{enumerate}

    \item \textbf{Step 2(a):} Informally, in this step we partition $\bx$ in blocks of size $\Theta\p{ d^{1.01-y_i'}}$ and from $\Theta \p{ d^{z_i'-0.01}}$ such blocks try to shortlist $\Theta \p{d^{0.01}}$ non-zero entries. Step \hyperref[nnconstruction:stepb]{(2a)} is used to implement this step, except when $i=1$ when this does not require any neural network since all entries are non-zero and we simply extract the first $z_1$ entries. Otherwise, by the induction hypothesis for $i>1$, we have $\bx_i \in \mathcal S^d_{\delta}$. Using Lemma~\ref{lemma:shortlisting_non-zero_net_output} and the fact $\bx_i \in \mathcal S^d_{\delta}$ we see that Step \hyperref[nnconstruction:stepa]{(2a)} which uses the \emph{non-zero element shortlisting} (Definition~\ref{defn:shortlist_non-zero_net},  $ \Ncal_{\delta}^{NZES}$) correctly implements Step 2(a) (of Algorithm~\ref{alg:informalAlgo}). The shortlisting operations are done on blocks of size $\ceil{3^{i} d^{1.01}/y_i}$ in parallel $\ceil{z_i/d^{0.01}}$ times, where in each instance, $\leq \ceil{d^{0.01}}$ entries from the block are shortlisted. (As described in Algorithm~\ref{alg:informalAlgo}, we choose the number of entries $\leq \ceil{d^{0.01}}$ to shortlist from each block so that $z_i$ total non-zero entries are chosen.)
    
 The width required for shortlisting non-zero entries from each block can be found via Lemma~\ref{lemma:shortlisting_non-zero_net_output} by plugging in $d'=\ceil{3^{i} d^{1.01}/y_i},p
 \leq \ceil{d^{0.01}}$, yielding width $\Ocal (d^{1.01-y_i' + 0.01 })$. Since we repeat this in parallel for each of  the $\ceil{z_i/d^{0.01}}$ blocks, the total width required is $\mathcal O (d^{1.01-y_i'  + 0.01 }\cdot d^{ z_i'-0.01}) \in \mathcal O (d^{1.01+z_i'-y_i'}) \in \mathcal O(d)$ width (see \hyperref[line:parameter_properties]{Parameter Properties}). Also from Lemma~\ref{lemma:shortlisting_non-zero_net_output} this requires {3 hidden layers} and magnitude of weights bounded by $2/\delta$. We denote this vector of non-zero entries as $S_i$.
    
    \item \textbf{Step 2(b):} Informally, in this step we compute the rank of $\med(\bx_1)$ among the non-zero entries in the sparse vector $\bx_i$ and then compute the endpoints of new interval around $\med(\bx_1)$, i.e., $e_i^{-}, e_i^{+}$ using $S_i$. Step \hyperref[nnconstruction:stepb]{(2b)} is used to implement this step. For $i=1$ we have $r_1=d/2, \left|\bx_i^{\neq 0}\right|=d$, $z_1=\ceil{d^{0.5}}$ and $w_1= \ceil{d^{0.26}}$ (see \hyperref[line:parameter_properties]{Parameter Properties}) and hence we can pre-compute $r'_i = r_i z_i/d$ and subsequently $\floor{\bL_1} = \floor{r_1' -w_1 }, \ceil{\bR_1} = \ceil{r_1' +w_1 }$ without needing any neural network and we can skip to computing $e_i^+,e_i^-$. For iteration $i > 1$, we are processing $\bx_i^{\neq 0}$, which is a block of contiguous entries from $\bx_1$, i.e., the intersection of $\bx_1$ with some real interval, and $\bx_i$ contains the overall median. 
     We can thus use Lemma~\ref{lemma:rank_computing_net_output} by plugging in $d' =d''= d,  \bx= \bx_1, \by = \bx_i, e=e_{i-1}^{+}, r=d/2$ to conclude that \hyperref[nnconstruction:stepb]{(2b)} correctly computes the rank $r_i \in \mathbb{Z}^{>0}$ of the median in the new universe $\bx_i^{\neq 0}$ using the \emph{rank computing} (Definition~\ref{defn:rank_computing_net}, $\Ncal^{RC}_{\delta})$ neural network i.e., $\mathcal R_{r_i} \left(\bx_i^{\neq 0} \right) = \med(\bx_1)$. Further, from the same Lemma~\ref{lemma:rank_computing_net_output} we have that this step will require 1 hidden layer, a width of $\mathcal O(d)$ and magnitude of weights bounded by $1/\delta$. 
    
    Next we note $r_i \in [d]$, (from the correct computation of $r_i$ in the previous step) $ \left|\bx_i^{\neq 0} \right| \in [d]$ (for sufficiently large $d$ by the assumption of ``success" of Algorithm~\ref{alg:informalAlgo}), $z_i \in [d]$ is a pre-determined constant and $\bx_i \in \mathcal S^d_{\delta}$. We apply Lemma~\ref{lemma:rank_scaling_net_output} with $d'=d,\bx=\bx_i, r=r_i, b=z_i$ to conclude that the \emph{rank scaling} neural network (Definition~\ref{defn:rank_scaling_net}, $\Ncal^{RSC}_{\delta,z_i}$) in \hyperref[nnconstruction:stepb]{(2b)} correctly scales the new rank and produces $r'_i = r_iz_i / \left| \bx_i^{\neq 0} \right|$. In the previous step we used \emph{rank computing} (Definition~\ref{defn:rank_computing_net}, $\Ncal^{RC}_{\delta})$ with $\by = \bx_i$ and hence the quantity $\left|\bx_i^{\neq 0} \right|$ has already been computed by one of its layers. Thus, using Lemma~\ref{lemma:rank_scaling_net_output} with pre-computed $\left|\bx_i^{\neq 0} \right|$ we have that this step requires 1 hidden layer, width of $\mathcal O(d)$ and magnitude of weights bounded by $\mathcal O (\max(dz_i, 1/\delta)) \in \mathcal O (\max(d^{1.5}, 1/\delta)) $ (see \hyperref[line:parameter_properties]{Parameter Properties}). 

    Next we compute $\floor{\bL_i} = \floor{r_i'-w_i} , \ceil{\bR_i} = \ceil{r_i'+w_i}$ using the \emph{ceiling} neural network ( \ref{defn:ceil_net}, $ \Ncal^{CEI}_{d}$) on $r_i'$ in parallel (where we use the ceiling network to also compute the floor, using the identity $\floor{x} = - \ceil{-x}$) and then adding $w_i$ (since $w_i$ is an integer we do not need to include it in the input to $ \Ncal^{CEI}_{d}$). In the previous step we correctly computed $r_i' = r_iz_i / \left| \bx_i^{\neq 0}\right|$ where $ r_i,z_i,\left|\bx_i^{\neq 0}\right|,w_i \in [d]$ and $r_iz_i/\left| \bx_i^{\neq 0}\right| \in [-d,d]$ (follows from \hyperref[line:parameter_properties]{Parameter Properties} and ``success" Definition~\ref{defn:successful_linear}) and hence by Lemma~\ref{lemma:ceil_net_output} the quantities $\floor{\bL_i} = -\Ncal^{CEI}(-r_i') -w_i , \ceil{\bR_i} = \Ncal^{CEI}(r_i') +w_i $ are computed correctly. Further by Lemma~\ref{lemma:ceil_net_output} this requires 1 hidden layer, $\mathcal O(d)$ width and magnitude of weights bounded by $\Ocal(d)$. 
    
    We then compute $\max (\floor{\bL_i}+1,1)$ and $\min(\ceil{\bR_i}+1, z_i+2 )=-\max(-\ceil{\bR_i}-1, -z_i-2 ) $ using the \emph{maximum} neural network (Definition~\ref{defn:max_net}, $\Ncal^{MAX}$). This requires 1 hidden layer, $\mathcal O(1)$ width and $\mathcal O(1)$ weights.

     Finally recall that at the end of Step 2(b)  we let $e_i^{-}$ be the element with rank $\bL_i$ from $S_i$ if $\bL_i \geq 1$ else we let with $e_i^{-}=0$. And similarly we let $e_i^{+}$ be the element with rank $\bR_i$ from $S_i$ if $\bR_i \leq z_i$ else we let with $e_i^{+}=1$. It is easy to see that since $S_i \subseteq \bx_1 \in \mathcal S^d_{\delta}$ (from the ``success" of Algorithm~\ref{alg:informalAlgo}) the above step is equivalent to selecting $e_i^{-}, e_i^{+}$ as the elements with ranks $\max (\floor{\bL_i}+1,1),\min(\ceil{\bR_i}+1, z_i+2 )$ from $S_i \cup \{0,1\}\in \mathbb{R}^{z_i+2}$.
     We can thus use the \emph{rank selection} neural network (Definition~\ref{defn:rank_selection_net}, $\Ncal^{RS}_{\delta}$), letting $d'=z_i+2, \br = \{\max (\floor{\bL_i}+1,1),\min(\ceil{\bR_i}+1, z_i+2 )\}$ and $\bx = S_i \cup \{ 0,1\}$. From Lemma~\ref{lemma:rank_selection_net_output} we  conclude that this correctly computes $e_{i}^{-}, e_{i}^+$. Further, from Lemma~\ref{lemma:rank_selection_net_output} this step will require 2 hidden layers, a width of $\mathcal O(d^{2z_i}) \in \mathcal O(d)$ (see \hyperref[line:parameter_properties]{Parameter Properties}) and magnitude of weights bounded by $1/\delta$.

    \item \textbf{Step 2(c):} Informally, in this step we zero out entries not lying in the interval $[e_i^{-}, e_i^{+}]$. Noting from the earlier step that $e_{i}^{-}, e_{i}^{+}$ are elements of $\bx_i^{\neq 0} \cup \{0,1\}$ and $\bx_i \in \mathcal S^d_{\delta}$ we can use Lemma~\ref{lemma:filtering_net_output} with $d'=d, \ell = e_{i}^{-},u = e_{i}^{+}$ and $\bx = \bx_i$ to conclude that the \emph{filtering} neural network (Definition~\ref{defn:filtering_net}, $ \Ncal_{\delta}^{F}$) used in \hyperref[nnconstruction:stepc]{(2c)} correctly filters the entries of $\bx_i$, i.e., only the entries of $\bx_i$ lying in the range $\left[  e_{i}^{-}, e_{i}^{+}\right]$ are kept unchanged while the rest of the entries are zeroed. This new vector is denoted as $\bx_{i+1}$. Thus, Step 2(c) from Algorithm~\ref{alg:informalAlgo} is correctly implemented by \hyperref[nnconstruction:stepc]{(2c)}. Also from Lemma~\ref{lemma:filtering_net_output} we have that this step requires 1 hidden layer, a width of $\mathcal O(d)$ and magnitude of weights bounded by $1/\delta$.
    
    Finally, we point out, as promised at the beginning of the proof, that the vector $\bx_{i+1}$ output by this step will lie in $\mathcal S^d_{\delta}$: we zero out certain entries from $\bx_i$, and from the definition of $\mathcal S^d_\delta$, zeroing out entries maintains membership in $\mathcal S^d_\delta$; and thus since $\bx_i$ was in $\mathcal S^d_\delta$ (either from the analysis of the previous iteration for $i>1$, or because $i=1$ and $\bx_1\in \mathcal S^d_\delta$ by assumption), we conclude our induction step that $\bx_{i+1}\in \mathcal S^d_\delta$.

\end{itemize}

From the above neural network implementation of Algorithm~\ref{alg:informalAlgo} we see that each step requires a width of $\mathcal O(d)$ and hence the entire neural network implementation requires $\mathcal O(d)$ width. For the number of hidden layers required, we first note for the special case of $i=1$ we skipped directly to the rank selection step in Step \hyperref[nnconstruction:stepb]{(2b)}. Thus, for $i=1$ only $3$ hidden layers are required. For the subsequent iterations $(i>1)$ we see from the above implementation that $10$ hidden layers are required and since we iterate $3$ times the total number of hidden layers required is $33$. Also, each step of the implementation has magnitude of weights bounded by $\mathcal{O}(\max(d^{1.5}, 1/\delta))$. Finally we conclude that a ``successful" implementation of Algorithm~\ref{alg:informalAlgo} implies by Definition~\ref{defn:successful_linear} that $\bx_5$ will have at most $6^4 d^{0.16}$ non-zero entries with $\med(\bx_1)$ as one of its non-zero entries.
    
\end{proof}

\clearpage

\subsection{Hashing step} 

Recall that our ReLU neural network outlined by Algorithm~\ref{alg:sparsifyingFunctionNeuralNet} will return a very sparse vector $\bx'$, with $d$ entries but at most $6^4 d^{0.16}$ non-zero entries, one of which is the true median (with high probability, as guaranteed by Proposition~\ref{prop:linear_sparsifying_NN}). We next explain how to ``hash'' these non-zero entries, deterministically and without collisions, into a shorter vector with at most $\sqrt{d}$ entries; given this, it is straightforward to conclude our algorithm with a brute-force quadratic-width sorting neural network (Definition~\ref{defn:rank_selection_net}, $\Ncal^{RS}_{\delta}$) to return the median.

Hashing is typically viewed as a two-stage process where first a hash function $h$ is randomly selected from a hash function family $\cal{H}$, and then $h$ is applied to the input vector $\bx'$. However, in our neural network setting we do not have access to randomness. Instead, we explicitly define (and precompute) a small hash function family $\cal{H}$ and deterministically evaluate $h(\bx')$ for \emph{all} $h\in\cal{H}$, showing mathematically that at least one such $h$ must have 0 collisions, and algorithmically identifying and using this $h$ in our neural network.

\begin{algorithm}[t]
\SetKwFunction{hash}{hashingFunction} 
\caption{Given any $(d,\varepsilon)$ sparse vector $\bx'\in \mathcal S^d_{\delta}$ return $\bx'' \in \mathbb{R}^{d^{\varepsilon}}$ containing all non-zero entries of $\bx'$}\label{alg:hashingFunction}
\KwIn{Vector: $\bx' \in \mathcal S_{\delta}^d$ , $\bx'$ is $(d,\varepsilon)$ sparse}
\KwOut{Vector $\bx'' \in \mathcal S_{\delta}^{d^{\varepsilon}} $ }
\DontPrintSemicolon
\hash{$\bx', \varepsilon, \gamma$}{:} \;
\BlankLine\; 
\;  
\begin{enumerate}
    \item Letting $p$ be the smallest prime larger than $d^{2\varepsilon+\gamma}$, use each of the  $p$ hash functions from the hash function family $\mathcal H_{d,2\varepsilon+\gamma}$ (Definition~\ref{const:vectorized_hashing}) to map $\bx'$ to vectors of size $p$, creating a new vector $\by \in  \mathbb{R}^{p^2}$.
    \begin{itemize}
        \item We do this using $p$ copies of the \emph{hashing} neural network (Definition~\ref{defn:hashing_net}, $ \Ncal_{\delta}^{H}$). 
    \end{itemize}

    \item Identify and extract the output of the first of the $p$ hash functions that caused zero collisions, calling this output $\by'\in\mathbb{R}^p$.
    \begin{itemize}
        \item We do this by using the \emph{block extraction} neural network (Definition~\ref{def:extract-big-block}, $\Ncal^{BE}_{\delta,p}$)
        
    \end{itemize}

    \item Select the $d^{\varepsilon}$ largest entries from $\by'$ and denote this vector as $\bx''$.
    
    \begin{enumerate}
        \item We do this by using the \emph{rank selection} neural network (Definition~\ref{defn:rank_selection_net} $ \Ncal_{\delta}^{RS}$).
    \end{enumerate}
    
    \item Return $\bx''$
    
    \end{enumerate}
  
\end{algorithm}

\begin{proposition} \label{prop:hashing_algo_correctness}
    The algorithm \textup{\texttt{hashingFunction($\bx', \varepsilon, \gamma)$}} (Algorithm~\ref{alg:hashingFunction}), when given any $(d,\varepsilon)$ sparse vector $\bx' \in \mathcal S^d_{\delta}$ as input, where $d \geq \frac{1}{(2 \varepsilon+\gamma)^{1/\gamma}}$, returns $\bx'' \in \mathbb{R}^{d^{\varepsilon}}$ containing all the non-zero entries of $\bx'$ and is padded with $0$'s if there are less than $d^{\varepsilon}$ non-zero entries in $\bx'$. Moreover, the function can be implemented by a neural network with width $\mathcal O \left(d^{\max( 4\varepsilon+2\gamma,1)} \right)$, and depth $7$ with magnitude of weights bounded by $1/\delta$.
\end{proposition}

\begin{proof} Intuitively, the function uses hash functions $h$ in the hash family $\mathcal H_{d, 2\varepsilon+\gamma'}$ (Definition~\ref{const:vectorized_hashing}) to hash the entries of a $(d, \varepsilon)$ sparse $\bx'$ to $\mathcal O \left(d^{2\varepsilon + \gamma}\right)$ locations, looking for a hash function that produces zero collisions on the non-zero entries of the given input $\bx'$. By Lemma~\ref{lemma:existence_of_success} such a hash function exists in $\mathcal H_{d, 2\varepsilon+\gamma'}$ provided $d \geq \frac{1}{(2 \varepsilon+\gamma)^{1/\gamma}}$. We then use the results of this hash function to extract the non-zero entries of $\bx'$. The neural networks used to do these operations are elaborated  below. Let $p$ be the smallest prime number greater than $d^{2 \varepsilon+\gamma}$; because there is a prime in any positive integer interval $[\ell,2\ell]$ (Bertrand's Postulate), we have that $p\in \Ocal(d^{2\epsilon+\gamma})$.

\begin{itemize}
    
    \item \textbf{Step 1:} Since the dimension $d$ of the problem is fixed, and $ \varepsilon, \gamma$ are pre-determined parameters, we can thus implement each of the $p$ hash functions $h: \mathbb{R}^d \to \mathbb{R}^p$ in $\mathcal{H}_{d, 2 \varepsilon+\gamma}$ (Definition~\ref{const:vectorized_hashing})  using a pre-determined \emph{hashing} neural network (Definition~\ref{defn:hashing_net}, $\Ncal^H_h$) as proven in Lemma~\ref{lemma:hashing_net_output}. Thus we implement all $p$ hash functions in $H_{d, 2 \varepsilon+\gamma}$ using $1$ hidden layer and width of $p^2 \in \mathcal O \left( d^{4 \varepsilon + 2\gamma}\right)$ with magnitudes of weights in $\{0,1\}$. 
    
    Further, by our assumption, the input $\bx'$ is $(d, \varepsilon)$ sparse and thus Lemma~\ref{lemma:existence_of_success} guarantees that at least one of the hash functions in  $\mathcal H_{d, 2\varepsilon+\gamma}$ hashes all the non-zero elements of $\bx'$ to \textit{distinct} locations in $[p]$.

    To prepare for the next step, in parallel with this we should count the number of non-zero entries $s$ in the input: let $s=\Ncal_\delta^{NZC}(\bx')$, which by Lemma~\ref{lemma:non_zero_counting_net_output} uses 1 hidden layer (in parallel with the previous construction, for no extra depth), width $2d'$, and weights of magnitude $\frac{1}{\delta}$.

    \item \textbf{Step 2:} We search for and extract the result of the hash function that produced zero collisions using the \textit{block extraction} neural network $\Ncal_{\delta,p}^{BE}(\by,s)$, which by Lemma~\ref{lem:extract-big-block} will operate correctly and use 3 hidden layers, width $p^2$, and have magnitudes of weights bounded by $\frac{1}{\delta}$. This will return a vector $\by'\in \mathbb{R}^p$ containing the $s\leq d^{\varepsilon}$ non-zero entries of our original $\bx'$.

     \item \textbf{Step 3:} We then select the non-zero entries of $\by'$ with the \emph{rank selection} neural network (Definition~\ref{defn:rank_selection_net}, $\Ncal_{\delta}^{RS}$). Namely, since we want to return the $s$ non-zero entries from the nonnegative vector $\by'\in\mathbb{R}^p$, we simply ask for the elements of ranks $p-s+1,\ldots,p$, letting $\bx''=\Ncal_{\delta}^{RS}(\by',(p-s+1,\ldots,p))$, which by Lemma~\ref{lemma:rank_selection_net_output} correctly returns the answer using $2$ hidden layers, a width of $\mathcal O \left(d^{4\varepsilon+2\gamma}\right)$ and magnitude of weights bounded by $1/\delta$. 
    
\end{itemize}

Thus, the neural network describe here correctly implements  Algorithm~\ref{alg:hashingFunction} using a total of 6 hidden layers, a width of $\mathcal{O}( \max\p{d, d^{4\varepsilon+2\gamma}})$, and has magnitudes of weights bounded by $\Ocal(\frac{1}{\delta})$.

\end{proof}

\begin{algorithm}[t]
\SetKwFunction{median}{computeMedian}
\caption{Given $\bx$ compute its median}\label{alg:medianFindingFunctionLinear}
\KwIn{ Entirely non-zero vector $\bx \in \mathcal S^d_{\delta}$, with uniformly randomly permuted entries.}
\KwOut{ $\med(\bx)$ with probability $1-\exp \p{-d^{\Omega(1)}}$ or a value $\in [0,1]$ }
\DontPrintSemicolon
\median{$\bx $}{:} \;
\BlankLine\;  
    \begin{enumerate}
        \item  $\bx' \leftarrow \texttt{sparsifyingFunction}(\bx) $, $\qquad$ (Algorithm~\ref{alg:sparsifyingFunctionNeuralNet} )
        
        \begin{itemize}
            \item Requires $33$ hidden layers, width of $\mathcal O \left(d\right)$ and magnitude of weights bounded by $\Ocal \p{\max(1/\delta), d^{1.5}}$ to implement using a neural network. 
            \item Returns $(d, 0.16+4\log_d 6 )$ sparse vector $\bx' \in \mathcal S^d_{\delta}$  with probability at least $1-\exp \left(- 
             \left( d^{\Omega(1)} \right)\right)$. Also the entries of $\bx'$ come from a contiguous block of entries in $\bx$'s sorted order.
        \end{itemize}
        \item $\bx'' \leftarrow  \texttt{hashingFunction}(\bx')$ $\qquad$ (Algorithm~\ref{alg:hashingFunction})
        \begin{itemize}
            \item Requires $6$ hidden layers, width of $\mathcal O \left(d\right)$ and magnitude of weights bounded by $1/\delta$ to implement using a neural network.
            \item Returns $\bx'' \in \mathbb{R}^{6^4d^{0.16}}$ which contains all the non-zero entries of $\bx'$  if $\bx'$ is a $(d, 0.16+4\log_d 6 )$ sparse vector and $d\geq \frac{1}{(2\varepsilon+\gamma)^{1/\gamma}}$ (we choose $\varepsilon=0.16+4\log_d 6$ and choose $\gamma=0.18$). 
        \end{itemize}
        \item Compute relative rank $r''$ of $\mathcal R_{d/2} (\bx)$ in $\bx''$. 
        \begin{itemize}
            \item We do this using our \emph{rank selection} (Definition \ref{defn:rank_selection_net}, $\Ncal^{RS}_{\delta}$) and \emph{rank computing} neural network (Definition~\ref{defn:rank_computing_net}, $\Ncal^{RC}_{\delta}$) requiring $3$ hidden layers, width of $\Ocal(d)$ and magnitude of weights bounded by $1/\delta$.
        \end{itemize}
        \item Output $\min \left(1, \max \left(0,\mathcal R_{r''}(\bx'') \right) \right)$
        \begin{itemize}
            \item We do this by using the \emph{rank selection} neural network (Definition~\ref{defn:rank_selection_net}  $\Ncal_{\delta}^{RS}$), letting $m\leftarrow \mathcal R_{r''}(\bx'')$, followed by computing and returning $\left[m\right]_+ - \left[m-1\right]_+ $ requiring $3$ hidden layers.
        \end{itemize}
    \end{enumerate}
\end{algorithm}

\begin{proposition} \label{prop:complete_algo_analysis}
For any dimension $d>0$ and any $\delta>0$ the function \textup{\texttt{computeMedian($\bx$)}} (Algorithm~\ref{alg:medianFindingFunctionLinear}) where  $\bx \in \mathcal S^d_{\delta}$ and is entirely non-zero, with uniformly randomly permuted entries, returns $\med(\bx)$ with probability at least $1-\exp \left( - d^{ \Omega(1)}\right)$. In all other cases, the value returned by \textup{\texttt{computeMedian($\bx$)}} lies in $[0,1]$. Moreover, the function can be implemented by a ReLU neural network with width $\mathcal O(d)$ and depth $45$ hidden layers with magnitude of weights bounded by $\mathcal O(\max\left(d^{1.5},1/\delta \right))$. 

\end{proposition}

\begin{proof} At a high level, the neural network \textup{$\texttt{computeMedian()}$}, as outlined by Algorithm~\ref{alg:medianFindingFunctionLinear}, first sparsifies the input $\bx$ using the neural network \texttt{sparsifyingFunction}$()$ (outlined by Algorithm~\ref{alg:sparsifyingFunctionNeuralNet}) producing $\bx'$. Following this, it reduces the dimensionality of $\bx'$ using \texttt{hashingFunction}$()$ (as outlined by Algorithm~\ref{alg:hashingFunction}) while preserving the non-zero elements, outputting $\bx''$. Finally, it determines the rank $r''$ of the overall median relative to $\bx''$ and then computes the median from these entries by a brute force all-pairs algorithm, provided $\med(\bx)$ is indeed in this subset. We trim the output to the interval $[0,1]$ so that even the rare cases where the neural network fails do not contribute much to the mean squared error.

Formally, since entries of $\bx$ are uniformly randomly permuted and $\bx$ is entirely non-zero we have by Proposition~\ref{prop:alg_prob_analysis} that Algorithm~\ref{alg:informalAlgo} is ``successful" as stated in Definition~\ref{defn:successful_linear} with probability at least $1- \exp \left( - d^{\Omega(1)}\right)$. Further, using the fact $\bx \in \mathcal S^d_{\delta}$ and is entirely non-zero, we have from Proposition~\ref{prop:linear_sparsifying_NN} that Algorithm~\ref{alg:informalAlgo}  is correctly implemented by the neural network described in Algorithm~\ref{alg:sparsifyingFunctionNeuralNet}, which has $33$ hidden layers, width $\mathcal O(d)$,  and magnitudes of weights bounded by $\mathcal{O} \p{\max(1/\delta),d^{1.5}}$.

    In the case of ``success" (Definition~\ref{defn:successful_linear}), the number of non-zero entries in $\bx'$ returned by Algorithm~\ref{alg:sparsifyingFunctionNeuralNet} is at most $6^4 d^{0.16}$ with $\med(\bx)$ being one of them, and these non-zero entries form a contiguous block from the sorted version of $\bx$ (Proposition~\ref{prop:alg_prob_analysis}). We use Algorithm~\ref{alg:hashingFunction}, with parameters $\varepsilon = 0.16 + 4\log_d 6 , \gamma=0.18$, to hash the non-zero locations of $\bx'$ and extract its non-zero entries $\bx^{\neq 0}$. We denote the extracted vector as $\bx'' \in \mathbb{R}^{6^4 d^{0.16}}$, which contains all the non-zero entries of $\bx'$ and is padded with $0$'s when there are less than $6^4d^{0.16}$ non-zero entries. From Proposition~\ref{prop:hashing_algo_correctness} we have that Algorithm~\ref{alg:hashingFunction} can be implemented by a neural network of width $\mathcal O(d^{\max(4 \varepsilon+2\gamma, 1)}) \in \mathcal O(d)$---by our choice of $\epsilon,\gamma$---and $6$ hidden layers, with magnitudes of weights bounded by $1/\delta$.

Use the \emph{rank selection} neural network $\Ncal_\delta^{RS}$ to select the maximum element $e$ of $\bx''$, using 2 hidden layers, width $\Ocal(d^{0.32})$, and magnitudes of weights bounded by $\frac{1}{\delta}$, by Lemma~\ref{lemma:rank_selection_net_output}.

     Recall from the proof of Proposition~\ref{prop:linear_sparsifying_NN} that the non-zero entries of $\bx''$ comprise a contiguous block of elements in $\bx$ when $\bx$ is sorted; also letting $d'=6^4d^{ 0.16}$, we have that $\bx'' \in \mathcal S^{d'}_{\delta} $ since $\bx''^{\neq 0} \subseteq \bx$. Thus the \emph{rank computing} neural network (Definition~\ref{defn:rank_computing_net}, $\Ncal^{RC}_{\delta}$) accurately computes the rank $r'$ of $\med(\bx)$ in $\bx''$ (using $e$ as an auxiliary input) using 1 hidden layer, a width of $\mathcal O(d)$ and magnitudes of weights bounded by $1/\delta$.
     
     We then compute $\mathcal R_{r''} (\bx'')$ using the \emph{rank selection} neural network (Definition~\ref{defn:rank_selection_net}, $\Ncal^{RS}_{\delta}$). Using Lemma~\ref{lemma:rank_selection_net_output} and plugging in $d' = 6^4d^{0.16}$ we have  that this operation uses {2 hidden layers}, a width of $\mathcal O(d)$ and magnitudes of weights bounded by $1/\delta$.
     
     It is easy to check that $[x]_+ - [x-1]_+ = \min \left( 1 , \max \left(0,x \right) \right) = x$ when $x \in [0,1]$. Since for our input, the median will always be in $[0,1]$, this final trimming step will never modify the returned median $m$ in the cases that $m$ is accurate; but in all other cases the output is in the interval $[0,1]$ concluding the first claim of the proposition.
     
    This final step clearly requires 1 hidden layer, $\mathcal O(1)$ width and $\mathcal O(1)$ magnitude of weights.
    
    Finally, we put everything together to see that implementing \textup{$\texttt{computeMedian()}$} requires a $45$ hidden layers, width of $\Ocal(d)$ and magnitude of weights bounded by $\Ocal(d^{1.5}, 1/\delta)$ concluding the proof.

\end{proof}


\subsection{Proof of Theorem~\ref{thm:linear_width_med_computation_distribution_zeroone}} 

\begin{proof} 
We use Proposition~\ref{prop:complete_algo_analysis} to show that, when the input $\bx$ is $\delta$-separated and appropriately bounded, the neural network outlined in Algorithm~\ref{alg:medianFindingFunctionLinear} will accurately compute the median with probability at least $1-\exp(-d^{\Omega(1)})$, and always returns values in $[0,1]$. We combine this with the bounds on the probability that a randomly chosen input will fail the input requirements of Proposition~\ref{prop:complete_algo_analysis}: from Lemma~\ref{lemma:nice_dist_imply_delta_separatedness}, plugging in $d'=d,\delta\coloneqq\frac{\epsilon}{3d^2}$ we have that $\pr_{\bx \sim \mathcal \Ucal([0,1]^d)} \left[ \bx \not \in \mathcal{S}^d_{\epsilon/3d^2}\right] \leq \epsilon$. Thus, from the union bound on these two failure modes,
\[\pr_{\bx \sim \mathcal \Ucal([0,1]^d)} \left[ \Ncal(\bx) \neq \med(\bx)\right] \leq \epsilon + \exp \left( -  d^{\Omega(1)} \right).\]

In the cases that $\Ncal(\bx) \neq \med(\bx)$, since both the true median and our algorithm's returned answer are in the range $[0,1]$, our error is at most 1, and thus our squared error is also at most 1. Thus the mean squared error is at most the probability of failure, which we bounded above as $\epsilon + \exp \left( -  d^{\Omega(1)}\right)$, concluding the proof of the first part of the theorem.

    The neural network construction outlined in Proposition~\ref{prop:complete_algo_analysis} has $45$ hidden layers, width of $\Ocal(d)$ and magnitudes of weights bounded by $\Ocal(d^{1.5}, 1/\delta) \in \Ocal(d^2/\epsilon)$ concluding the proof.
\end{proof}

\section{Lower bounds proofs}

\subsection{Proof of Theorem~\ref{theorem:exact_median}}\label{app:exact_median_proof}
    Let $r\le d-1$, and let $\Ncal$ be as in the theorem statement, where the input dimension is $d+r-1$. Namely, it is a depth-$k$, width-$n$ ReLU network satisfying
    \[
        \Ncal(\bx)=\Rcal_r(\bx)
    \]
    for all $\bx\in[0,1]^{d+r-1}$. Suppose by contradiction that
    \[
        n < \frac{1}{40}(d+r-1)^{1+\frac{1}{2^{k-2}-1}}.
    \]
    Given an input $\bx'\in[0,1]^d$, we concatenate it with $r-1$ ones to obtain the vector $\bx''$, and feed this input to $\Ncal$. Since the original $d$ inputs are not greater than the $r-1$ added inputs, the rank-$r$ element is necessarily the maximum of the original $d$ inputs, and since all $d+r-1$ coordinates are in the interval $[0,1]$, we have by our assumption on $\Ncal$ that
    \[
        \Ncal(\bx'')=\max(\bx')
    \]
    for all $\bx'\in[0,1]^{d}$. Note that $r=d-1$ entails that $\Ncal$ computes the median, so for $r\in[d-1]$ we can compute any rank between the median and the maximum. To compute lower ranks, one can simply pad with zeros instead of ones, in which case taking $r\in[d-1]$ computes all the ranks between the minimum and the median, hence the assumption that $r\le d-1$ does not impose limits on the rank that we wish to compute.
    
    We now construct a neural network $\Ncal'$ that receives $\bx'$ as input, rather than the $(d+r-1)$-dimensional $\bx''$. This is easily achieved by substituting all the $r-1$ ones and modifying the bias term of the neurons in the first hidden layer, which does not change the architecture, and therefore of $\Ncal'$ has the same width and depth as $\Ncal$. 

    We compute 
    \[
        n < \frac{1}{40}(d+r-1)^{1+\frac{1}{2^{k-2}-1}} \le \frac{1}{40}2^{1+\frac{1}{2^{k-2}-1}}d^{1+\frac{1}{2^{k-2}-1}} \le \frac{1}{10}d^{1+\frac{1}{2^{k-2}-1}},
    \]
    where the second inquality holds since $r\le d-1$, and the last inequality holds since $k\ge3$ which entails that the exponent is at most $2$. But this contradicts Theorem~\ref{theorem:max_hierarchy}, so it must hold that
    \[
        n \ge \frac{1}{40}(d+r-1)^{1+\frac{1}{2^{k-2}-1}},
    \]
    where a change of variables $d+r-1\mapsto d$ concludes the proof of the theorem.

\subsection{Proof of Theorem~\ref{theorem:depthany_lower_bounds}} \label{theorem:depthany_lower_bounds_proof}

\begin{proof} Let $\Ncal:\reals^{2d-1}\to\reals$ be as in the theorem statement, where the input dimension is $2d-1$. Namely, it holds that
    \[
        \Exp_{\bx \sim \Ucal \p{[0,1]^{2d-1}}} \left [ \left( \Ncal (\bx) -  \med (\bx) \right)^2 \right] \le \varepsilon,
    \]
    for a depth-$k$, width-$w(2d-1)\le w(2d)$ $\sigma$-neural network $\Ncal$ with weights bounded by $M(2d-1)\le M(2d)$. Suppose that $\bx_0\sim\Ucal\p{[0,1]^d}$, and consider the process of concatenating the input $\bx_0$ with $\bx_1\sim\Ucal\p{\pcc{1,2-\frac{1}{2d-1}}^{d-1}}$ added coordinates, applying a uniformly sampled permutation on $[2d-1]$ on the resulting vector, and finally scaling it to the unit interval by multiplying it by $\frac{2d-1}{4d-3}$, to receive the outcome denoted by $\bx'\in\reals^{2d-1}$. Let $A$ denote the resulting set of possible outcomes of this process, and note that applying this process to the uniform distribution over $[0,1]^d$ results in a uniform distribution over $A$. 

    We proceed by first lower bounding the probability that $\bx\sim\Ucal\p{[0,1]^{2d-1}}$ will satisfy $\bx\in A$. This is exactly the probability of successfully drawing a coordinate from $\pcc{0,\frac{d}{2d-1}}$ in precisely $d$ draws out of $2d-1$, which is captured by the following binomial probability
    \[
        \pr\pcc{\bx\in A} = \binom{2d-1}{d}\p{\frac{d}{2d-1}}^d\p{\frac{d-1}{2d-1}}^{d-1} = \frac{(2d-1)!}{d!(d-1)!}\p{\frac{d}{2d-1}}^d\p{\frac{d-1}{2d-1}}^{d-1}.
    \]
    Using standard Stirling bounds (Lemma~\ref{lem:stirling}), we bound the above probability by lower bounding $(2d-1)!$ and upper bounding $d!$ and $(d-1)!$, yielding
    \begin{align}
        \pr\pcc{\bx\in A} &\ge \frac{\sqrt{2\pi(2d-1)}\exp(-(2d-1))}{\sqrt{2\pi d}\exp(-d)\exp\p{\frac{1}{12d}}\sqrt{2\pi(d-1)}\exp(-(d-1))\exp\p{\frac{1}{12(d-1)}}}\nonumber\\
        &\ge \frac{\sqrt{2d-2}}{\sqrt{2\pi d(d-1)}\exp\p{\frac{1}{12d}+\frac{1}{12(d-1)}}} \ge \frac{1}{\sqrt{\pi d}\exp\p{\frac{1}{8}}} \ge \frac{1}{2\sqrt{\pi d}},\label{eq:stirling}
    \end{align}
    where the $d^d$, $(d-1)^{d-1}$ and $(2d-1)^{2d-1}$ terms cancel in the first inequality, the penultimate inequality holds for all $d\ge2$, and the last inequality follows since $\exp(-1/8)\le0.5$.

    With the above, we now turn to bound the approximation error as follows
    \begin{align}
        \Exp_{\bx'\sim\Ucal(A)}\pcc{\p{\Ncal(\bx')-\max(\bx_0)}^2} &= \Exp_{\bx'\sim\Ucal(A)}\pcc{\p{\Ncal(\bx')-\med(\bx')}^2}\nonumber\\
        &= \intop_{\bx'\in A}\p{\Ncal(\bx')-\med(\bx')}^2\p{Pr[\bx'\in A]}^{-1}d\bx' \nonumber\\
        \overset{\text{Equation}~\eqref{eq:stirling}}{}&\le 2\sqrt{\pi d}\intop_{\bx''\in[0,1]^{2d-1}}\p{\Ncal(\bx'')-\med(\bx'')}^2\one{\bx''\in A}d\bx \nonumber\\
        &\le 2\sqrt{\pi d}\intop_{\bx''\in[0,1]^{2d-1}}\p{\Ncal(\bx'')-\med(\bx'')}^2d\bx\nonumber\\
        &= 2\sqrt{\pi d} \Exp_{\bx''\sim\Ucal([0,1]^{2d-1})}\pcc{\p{\Ncal(\bx'')-\med(\bx'')}^2} \le 2\sqrt{\pi d}\varepsilon.\label{eq:max_to_median}
    \end{align}
    Now, assume by contradiction that
    \[
        \Exp_{\bx_0\sim\Ucal(A)}\pcc{\p{\Ncal(\tau(\bx_0,\bx_1))-\max(\bx_0)}^2} > 2\sqrt{\pi d}\varepsilon
    \]
    for all $\bx_1$ and permutations $\tau$. Then by the law of total expectation, this implies that
    \begin{align*}
        \Exp_{\bx'\sim\Ucal(A)}\pcc{\p{\Ncal(\bx')-\max(\bx_0)}^2} &= \Exp_{\bx_1}\pcc{\Exp_{\bx_0}\pcc{\p{\Ncal(\tau(\bx_0,\bx_1))-\max(\bx_0)}^2}|\bX_1=\bx_1,\tau}\\
        &> \Exp_{\bx_1}\pcc{2\sqrt{\pi d}\varepsilon} = 2\sqrt{\pi d}\varepsilon,
    \end{align*}
    contradicting Equation~\eqref{eq:max_to_median}. We thus have that there exists some $\bx_1$ and permutation $\tau$ on $[2d-1]$ such that
    \[
        \Exp_{\bx_0\sim\Ucal\p{\pcc{0,\frac{2d-1}{4d-3}}^{d}}}\pcc{\p{\Ncal(\tau(\bx_0,\bx_1))-\max(\bx_0)}^2} \le 2\sqrt{\pi d}\varepsilon.
    \]
    Since we can substitute $\bx_1$ in the first hidden layer of $\Ncal$ and simulate the permutation $\tau$ by composing it with the weights of the first hidden layer, it follows that there exists a $\sigma$-neural network $\Ncal''$ such that
    \[
        \Exp_{\bx_0\sim\Ucal\p{\pcc{0,\frac{2d-1}{4d-3}}^{d}}} \left [ \left( \Ncal'' (\bx) -  \max (\bx) \right)^2 \right] \le 2\sqrt{\pi d}\varepsilon.
    \]
    Lastly, we rescale the hidden layer weights of $\Ncal''$ by $0.5$ and its output neuron by $2$ to obtain a neural network $\Ncal'(\bx)=2\Ncal''(0.5\bx)$ whose domain is multiplied by $2$ and satisfies
    \[
        \Exp_{\bx_0\sim\Ucal\p{\pcc{0,1}^{d}}} \left [ \left( \Ncal' (\bx) -  \max (\bx) \right)^2 \right] \le 8\sqrt{\pi d}\varepsilon,
    \]
    where the accuracy is rescaled according to \citet[Theorem~9]{safran2019depth}. We note that $\Ncal'$ maintains the same depth and width of $\Ncal$, and has its weights multiplied by at most $2$.

\end{proof}


\subsection{Proof of Theorem~\ref{theorem:all_depth_lower_width_bounds}} \label{theorem:all_depth_lower_width_bounds_proof}

\begin{proof} Assume $d$ is even w.l.o.g.\ and consider the matrix of first hidden layer weights $W\in\reals^{k\times d}$. Since $k\le d-1$ by our assumption, we have $\dim(\ker(W))\ge1$ (for the case $d$ is odd using a slightly different choice of parameters in the below arguments will ensure we obtain the same result asymptotically). Let $\bv=(v_1,\ldots,v_d)\in\ker(W)$ such that $\norm{\bv}_2=1$ and assume w.l.o.g.\ $\max(\bv)=v_1$. Consider the triangular matrix $P$ below

    \[
            P\coloneqq \frac{1}{2}\p{
                \begin{matrix}
                    \frac1dv_1 & 0 & 0 & \cdots & 0 & 0\\
                    \frac1dv_2 & 1-\frac2d & 0 & \cdots & 0 & 0\\
                    \frac1dv_3 & 0 & 1-\frac2d  & \cdots & 0 & 0\\
                    \vdots & \vdots & \vdots & \ddots & \vdots & \vdots \\
                    \frac1dv_{d-1} & 0 & 0 & \cdots & 1-\frac2d & 0\\
                    \frac1dv_{d} & 0 & 0 & \cdots & 0 & 1 - \frac2d\\
                \end{matrix}
            },
    \]

and a vector $\bb_i$ whose first coordinate is $\frac{1}{2} \left(1-\frac{1}{d}\right)$, while $\frac{d}{2}-1$ entries have a value of $\frac{1}{2d}$ and $\frac{d}{2}$ entries have a value of $\frac{1}{2}$ where $i$ is used to index the set of all possible $\binom{d-1}{d/2}$ choices for entries with value $\frac{1}{2}$. For ease of exposition, w.l.o.g.\ we consider the case where the first $\frac{d}{2}-1$ entries after the first coordinate is $\frac{1}{d}$ and call this vector $\bb_1$. 

\[
            \bb_1 \coloneqq \frac{1}{2} \left. \p{
                \begin{matrix}
                    1-\frac1d\\
                    \frac1d\\
                    \frac1d\\
                    \vdots\\
                    1 \\
                    1
                \end{matrix}
            }
            \;\right\} \begin{array}{l}
                1\ \text{time} \\[1em]
                \frac{d}{2}-1\ \text{times} \\[2em]
                \frac{d}{2} \ \text{times}
                \end{array}.
\]

Now consider the set $\Pcal_1 = \{P \bx + \bb_1 : \bx \in [0,1]^d\}$. It is easy to verify that $\Pcal_1 \subset [0,1]^d$ for $d \geq 2$. We also have,
\begin{equation} \label{eq:med_over_para}
    \med(\bp) = p_1, \quad \forall \bp \in \Pcal_1
\end{equation}

The above follows by noting that every $\bp \in \Pcal_1 $ can be written as $P\bx+\bb_1$ for some $\bx \in [0,1]^d$ by definition. Hence, for some $\bx \in [0,1]^d$ and $p_i, 2 \leq i \leq d/2$ we have $p_i = \frac{1}{2d} v_i x_1 + \frac{1}{2}\left( 1- \frac2d\right)x_i \leq \frac{1}{2d} v_1 x_1 + \frac{1}{2} \left( 1- \frac1d \right)=p_1$ 
since $v_i \leq v_1$ in our setup. Noting that $x_i,v_i \in [0,1]$  for the case $p_i, i>d/2$ we have,
 $p_i = \frac{1}{2d} v_i x_1 + \frac{1}{2}\left( 1- \frac2d\right)x_i + \frac{1}{2} \geq   \frac{1}{2} \geq \frac{1}{2d} v_1 x_1 + \frac{1}{2} \left( 1- \frac1d \right)=p_1$.

 It easy to check for the different choices of $\bb_i, \forall i \in \left[\binom{d-1}{d/2} \right]$'s the corresponding set $\Pcal_i$ has the $p_1$ as the median, with the only change being the indices of the entries greater than $p_1$. Also by a standard Stirling bound (Lemma~\ref{lem:stirling}) we have,

 \begin{equation} \label{eq:lb_stirling}
    \begin{split}
    \binom{d-1}{d/2} \geq  \frac{c}{\sqrt{d}} \frac{(d-1)^{d-1}}{\left(\frac{d}{2} \right)^{\frac{d}{2}} \left(\frac{d}{2}-1 \right)^{\frac{d}{2}-1}} \geq \frac{c}{\sqrt{d}} \frac{(d-1)^{d-1}}{\left(\frac{d}{2} \right)^{\frac{d}{2}} \left(\frac{d}{2} \right)^{\frac{d}{2}-1}} \geq  \frac{c}{\sqrt{d}} 2^{d-1}  \left( 1 - \frac{1}{d} \right)^{d-1} \\
    \geq \frac{c}{10 \sqrt{d}} 2^{d-1}, d \geq 2 , 
    \end{split}
\end{equation}
  where $c$ is a universal constant. Thus, we have at least $\frac{c}{10 \sqrt{d}} 2^{d-1}$ choices of $\bb_i$. 

We also have if $i\neq j$ then $\Pcal_i \cap \Pcal_j = \phi$. This follows by noting that if $i\neq j$ then there exists a coordinate $k$ such that $(\bb_i)_k \neq (\bb_j)_k, k>1 $ and from the definition of $\Pcal_i$ it is easy to verify that for $ \bp \in \Pcal_i$ either $p_k \leq 0.5$ or $p_k > 0.5$ depending on whether $(\bb_i)_k = \frac{1}{2d}$ or $\frac{1}{2}$ respectively. Thus, it follows that if $i\neq j$ then $\Pcal_i \cap \Pcal_j = \phi$ as there exists a coordinate $p_k$ whose range of values in $\Pcal_i$ and in $\Pcal_j$ are completely disjoint. Also since $\Pcal_i \subseteq [0,1]^d$ we have $\cup_i \Pcal_i \subseteq [0,1]^d$. 

Putting all of the above together we get,

\begin{equation} \label{eq:lb_set_decomp}
    \begin{split}
    \E_{\bp\sim\Ucal\p{[0,1]^d}}\pcc{\p{\Ncal(\bp)-\med(\bp)}^2} \stackrel{\cup_i \Pcal_i \subseteq [0,1]^d}{\geq} 
     \int_{\cup_i \Pcal_i} \p{\Ncal(\bp)-\med(\bp)}^2 d \bp \\
   = \sum_i \int_{ \Pcal_i} \p{\Ncal(\bp)-\med(\bp)}^2 d \bp, 
    \end{split}
\end{equation}

where the last equality is due to the fact that $\mathcal P_i$'s are disjoint.

 For ease of exposition, w.l.o.g.\ we focus on the set $\mathcal P_1$. We evaluate the integral inside the sum and using the change of variables $\bp=P\bx+\bb_1$, $d\bp=\abs{\det\p{P}}d\bx$, we have
        \begin{equation} \label{eq:change_of_variables}
        \begin{split}
            \int_{ \Pcal_1} \p{\Ncal(\bp)-\med(\bp)}^2 d \bp =  \int_{ [0,1]^d} \p{\Ncal(P\bx+\bb_1)-\med(P\bx+ \bb_1)}^2 |\det(P)| d \bx
        \end{split}
        \end{equation}

        Letting $\be_i$ denote the standard unit vector with coordinate $e_i=1$, we get from $P\bx=\frac1{2d}\bv x_1+\sum_{i=2}^d\frac{1}{2}\p{1-\frac2d}x_i\be_i$ and $\bv\in\ker(W)$ that we can write $\Ncal(P\bx+\bb_1)=c(x_2,\ldots,x_d)$ for some function $c:\reals^{d-1}\to\reals$. Since $P$ is triangular, we have $\abs{\det(P)}=\frac{1}{2^d}\frac1d\p{1-\frac2d}^{d-1}v_1 \ge \frac{1}{2^d} \frac{1}{10d}v_1$ for $d\geq 2$. Moreover, since $\norm{\bv}_{\infty}=v_1$ and $\norm{\bv}_2=1$, we have that $v_1\ge d^{-0.5}$ and we can further lower bound the above to obtain $\abs{\det(P)}\ge \frac{1}{2^d} \frac{1}{10d^{1.5}}$. Plugging in  $\abs{\det(P)}$ and Equation~\eqref{eq:med_over_para} in Equation~\eqref{eq:change_of_variables}, we obtain
        
        \begin{equation} \label{eq:multivariate_integral}
        \begin{split}  
             \int_{ [0,1]^d} \p{\Ncal(P\bx+\bb_1)-\med(P\bx+ \bb_1)}^2  |\det(P)| d \bx  \\ \geq \frac{1}{2^d} \frac{1}{10d^{1.5}} \int_{ [0,1]^d} \p{\Ncal(P\bx+\bb_1)-\med(P\bx+ \bb_1)}^2d \bx \\
            \ge \frac{1}{2^d} \frac{1}{10d^{1.5}}\int_{[0,1]^d}\p{f(x_2,\ldots,x_d) - \p{\frac{1}{2d} v_1 x_1 + \frac{1}{2} \left( 1- \frac1d \right)}}^2d\bx \\
            = \frac{1}{2^d} \frac{1}{10d^{1.5}}\int_{x_d}\ldots\int_{x_2}\int_{x_1}\p{f(x_2,\ldots,x_d) - \p{\frac{1}{2d} v_1 x_1 + \frac{1}{2} \left( 1- \frac1d \right)}}^2dx_1dx_2\ldots dx_d.
        \end{split}
        \end{equation}

It is easy to verify that the optimal constant approximation for the linear function $\frac{1}{2d} v_1 x_1 + \frac{1}{2} \left( 1- \frac1d \right)$ is $\frac{1}{4d} v_1  + \frac{1}{2} \left( 1- \frac1d \right)$, in which case the optimal $L_2$ approximation error is,
        \[
            \int_0^1\p{\frac{1}{4d} v_1  + \frac{1}{2} \left( 1- \frac1d \right) - \p{\frac{1}{2d} v_1x_1  + \frac{1}{2} \left( 1- \frac1d \right)}}^2dx = \frac{v_1^2}{4d^2}\int_0^1\p{\frac12-x}^2dx = \frac{v_1^2}{48d^2}.
        \]
Plugging the above back in Equation~\eqref{eq:multivariate_integral} and using the fact that $v_1\ge d^{-0.5}$ again, we obtain,

        \begin{equation} \label{eq:small_set_approx_error}
        \begin{split}  
            \int_{ \Pcal_1} \p{\Ncal(\bp)-\med(\bp)}^2 d \bp \ge \frac{1}{2^d} \frac{1}{10d^{1.5}}\int_{x_d}\ldots\int_{x_2}\frac{v_1^2}{48d^2} dx_2\ldots dx_d \ge \frac{1}{2^d} \frac{1}{480d^{4.5}}.
        \end{split}
        \end{equation}

It is easy to see that for any choice of the set $\Pcal_i$ defined by $(P, \bb_i),$ the lower bound on the approximation error is the same as given by Equation~\eqref{eq:small_set_approx_error}. Thus, plugging this error in Equation~\eqref{eq:lb_set_decomp}, we get

\begin{equation*}
    \begin{split}
    \E_{\bx\sim\Ucal\p{[0,1]^d}}\pcc{\p{\Ncal(\bx)-\med(\bx)}^2} 
   \geq \sum_i \int_{ \Pcal_i} \p{\Ncal(\bp)-\med(\bp)}^2 d \bp \\
   \ge \frac{c}{10\sqrt{d}} 2^{d-1}   \cdot  \frac{1}{2^d} \frac{1}{480d^{4.5}} = \frac{c}{9600d^{5.5}}, \\
    \end{split}
\end{equation*}

where the second inequality follows from the lower bound on the total number of choices of $i$ (Equation~\eqref{eq:lb_stirling} and Equation~\eqref{eq:small_set_approx_error}), concluding the proof.

\end{proof}

\section{Technical auxiliary lemmas}

\subsection{Probabilistic lemmas}

\begin{lemma} [Stirling's Approximation \citep{robbins1955remark}] \label{lem:stirling} The following lower and upper bounds for $n!$ apply for all $n\geq 1$,
\[
    \sqrt{2 \pi n} \left(\frac{n}{e} \right)^n \leq n! \leq  \sqrt{2 \pi n} \left(\frac{n}{e} \right)^n e^{\frac{1}{12n}}
\]
\end{lemma}

Recall that $\mathcal{S}^{d'}_{\delta}$ is the set of vectors in $\mathbb{R}^{d'}$ whose entries are $\delta$-separated (or 0). We show that randomly chosen vectors will be $\delta$-separated with high probability, for inverse-polynomial $\delta$.

\begin{lemma} \label{lemma:nice_dist_imply_delta_separatedness}
    For any dimension $d'$, we have for any $\delta >0$,
    \[
        \pr_{\bx \sim \mathcal{U} \left([0,1]^{d'} \right)} \left[ \bx \in \mathcal S^{d'}_{\delta} \textrm{ and } \forall i, \; x_i \neq 0\right] \geq 1-3{d'}^2\delta.
    \]
\end{lemma}

\begin{proof}
    Sampling $\bx \sim \mathcal{U}([0,1]^{d'})$ is equivalent to independently sampling $x_i \stackrel{}{\sim} \mathcal{U}([0,1]), \; \forall i$. Since we want $\bx$ to be entirely non-zero and we want $\bx \in \mathcal S^{d'}_{\delta}$, we instead show that $\forall i, \; x_i \in [\delta,1-\delta] $, and that for all indices $i<j$, we have $|x_i-x_j|\geq\delta$.

     We have $\pr_{x_i \stackrel{}{\sim} \mathcal{U}([0,1])} \left [x_i \in [0, \delta) \cup (1-\delta, 1] \right] = 2 \delta$. And thus, taking a union bound over all $i$,
    \begin{equation} \label{eq:uniform_1}
        \pr \left[ \bx \not \in [\delta, 1-\delta]^{d'} \right] \leq 2{d'} \delta.
    \end{equation}

    Also, for a given pair of entries $i\neq j$ we have
    \[
        \pr \left[ \left| x_i - x_j\right| < \delta \right] = \int_{[0,1]}  \pr \left[ \left| x_i - x_j\right| < \delta \big | x_i\right] d x_i \leq 2 \delta.
    \]

    Hence, taking a union bound over all such pairs of entries, we get
     \begin{equation} \label{eq:uniform_2}
        \pr \left[ \exists \; (i,j) \; s.t. \;   \left| x_i - x_j\right| < \delta \right] \leq {d'}^2 \delta
    \end{equation}

    Thus, combining the results of Equation~\ref{eq:uniform_1} and Equation~\ref{eq:uniform_2}, we conclude
    \[
        \pr_{\bx \sim \mathcal{U} \left([0,1]^{d'} \right)} \left[ \bx \not\in \mathcal S^{d'}_{\delta} \textrm{ or } \exists i, \; x_i \neq 0\right] \leq 2d'\delta+d'^2\delta\leq 3{d'}^2\delta,
    \]
as desired.
    
\end{proof}

The below lemma states standard Chernoff bounds for the process of choosing a subset without replacement, often referred to as the \textit{hypergeometric distribution}. These bounds are used in the probabilistic analysis of the subsampling process in Algorithm~\ref{alg:informalAlgo}.

\begin{lemma} \label{lemma:neg_assoc}
Given a universe $U$ of $d'$ elements, with $T$ a subset of size $k$. Let $S$ be a random subset of $U$ of size $n$ (chosen without replacement). Then we have the following bounds on the probability that the random variable $r=|S\cap T|$ is far from its expectation $\frac{kn}{d'}$:

\begin{enumerate}
\item $\textrm{For any } \varepsilon\geq 0, \quad
    \pr \left[\left|r-\frac{kn}{d'}\right|\geq \varepsilon\right] \leq 2 \exp\left({-2\frac{\varepsilon^2}{n}}\right) \quad\quad\quad\quad\textrm{(Additive Chernoff)}
$
\item
$
    \pr \left[r\notin \left(\frac{1}{2} \frac{kn}{d'},2\frac{kn}{d'}\right)\right] \leq 2\exp \left(-\frac{1}{8}\frac{kn}{d'}\right) \quad \quad\textrm{(Multiplicative Chernoff for } \delta=\frac{1}{2},1)
$
\item Generalizing the part 2 upper bound, for any $\delta\geq 0$, \quad
$    \pr \left[ r \geq \frac{kn}{d'} \left( 1  + \delta\right)\right] \leq \exp \left(- \frac{\delta^2 kn/d'}{2+\delta}\right)
$
\end{enumerate}
\end{lemma}

\begin{proof}
The probability that a given entry (of the random permutation) is in $S$ equals $\frac{n}{d'}$. If all the entries were independent, then we are bounding the probability that the sum of $k$ samples from a $\frac{n}{d'}$ biased coin is more than $\varepsilon$ from $k\frac{n}{d'}$. Additive and multiplicative Chernoff bounds yield the respective stated bounds \emph{if} entries are independent.

Finally, we point out that 
a Chernoff bound for the \emph{independent} case implies \emph{identical} Chernoff bound for the ``negatively associated'' case, which includes sampling without replacement. (See \citep{Wajc2017NegativeAssociation}).

\end{proof}

While the previous lemma gives concentration bounds for subsampling without replacement, we next use this to give concentration bounds a sort of ``inverse'' of this process. 

Consider the following game. Alice's favorite positive integer is $r$. Alice finds $d'$ balls arbitrarily arranged at locations $\bx$ on the real line; and then a random subset $S$ of $n$ of these balls is colored red. Alice now  (arbitrarily) chooses a real interval $I$ such the number of red balls in $I$ is her favorite number, $|S\cap I|=r$. What can we say about the overall number of balls in $I$, regardless of how Alice chooses $I$? Intuitively, since $\frac{n}{d'}$ fraction of the balls are red, this number $|\bx\cap I|$ should be roughly $r\frac{d'}{n}$, and we show this is true in Lemma~\ref{lem:inverse-process}.

\begin{lemma}\label{lem:inverse-process}
Let $\bx$ be a set of $d'$ distinct real numbers, and let $[L,R]$ be an interval of integers, and let $n\leq d'$ be a nonnegative integer. Consider a probabilistic process $P$ that uniformly randomly chooses a subset $S\subseteq \bx$ of size $n$, and then arbitrarily selects a real interval $I=[I_\ell,I_r]$---possibly in terms of $S$---such that $\left|S\cap I\right|\in[L,R]$. Then we have 
\begin{equation}\label{eq:probabilistic-condition}\pr_{(S,I)\sim P}\left[\left|\bx\cap I\right|\in \left(\frac{1}{2}L\frac{d'}{n},2R\frac{d'}{n}\right)\right]\geq 1-3d'\exp\left(-\frac{L}{6}\right)\end{equation}

\end{lemma}
\begin{proof}
Index $\bx$ in sorted order, from $x_1$ up to $x_{d'}$.
We point out that for any interval $I$, the set $\bx\cap I$ consists of some (possibly empty) contiguous subset of $\bx$, namely $\{x_i,\ldots,x_j\}$.

Suppose for the sake of contradiction that Equation~\ref{eq:probabilistic-condition} is violated.

One way Equation~\ref{eq:probabilistic-condition} could be violated is if $\left|x\cap I\right|$ is $\leq \frac{1}{2}L\frac{d'}{n}$. In this case, the set $\bx\cap I$ must be a subset of some contiguous set $\{x_i,\ldots,x_j\}$ of size $k:=\lfloor\frac{1}{2}L\frac{d'}{n}\rfloor$. For each fixed pair $i,j$ such that $j-i+1=k$ we separately apply part 3 of Lemma~\ref{lemma:neg_assoc}, letting $T=\{x_i,\ldots,x_j\}$, and letting  $1+\delta\coloneqq\frac{d'}{n}\cdot \frac{L}{\lfloor\frac{1}{2}L\frac{d'}{n}\rfloor}$, where we point out that $\delta\geq 1$. This leads to a probability bound of \[\exp\left(-\frac{\delta^2 \lfloor\frac{1}{2}L\frac{d'}{n}\rfloor n/d'}{2+\delta}\right)=\exp\left(-L\frac{\delta^2}{(2+\delta)(1+\delta)}\right)\leq \exp\p{-\frac{L}{6}}\]
This bound applies for each choice of indices $i,j$ with $j-i+1=\lfloor\frac{1}{2}L\frac{d'}{n}\rfloor$, so we thus take a union bound over the at most $d'$ such choices of indices, to cover all cases.

On the other side, the other way Equation~\ref{eq:probabilistic-condition} could be violated is if $\left|\bx\cap I\right|$ is $\geq 2R\frac{d'}{n}$. In this case, the set $\bx\cap I$ must contain some contiguous set $\{x_i,\ldots,x_j\}$ of size $j-i+1=\lceil2R\frac{d'}{n}\rceil$, which we denote as $k$ when we invoke Lemma~\ref{lemma:neg_assoc}. For each pair $i,j$ satisfying this condition, we consider the lower bound side of part 2 of Lemma~\ref{lemma:neg_assoc}: we bound the probability that a random subset $S\subset x$ of size $n$ has intersection with $\{x_i,\ldots,x_j\}$ of size $\leq \frac{1}{2}\frac{kn}{d'}$ (which, since $\frac{1}{2}\frac{kn}{d'}\geq R$ also bounds the probability that this intersection has size $\leq R$) by $2\exp(-\frac{1}{8}\frac{k n}{d'})$, which is thus $\leq 2\exp(-\frac{R}{4})$. Since $R\geq L$, this probability is trivially at least $2\exp(-\frac{L}{6})$.
This bound applies for each choice of indices $i,j$ with $j-i+1=\lceil2R\frac{d'}{n}\rceil$, so we thus take a union bound over the at most $d'$ such choices of indices, to cover all cases.

\end{proof}

\subsection{Hash function construction}

In this subsection, we explain how to adapt standard hash function techniques to construct a hash function family that will enable a collision-free hash of any $(d',\varepsilon)$ sparse input vector. 
Given a hash function $h:[d']\rightarrow[p]$, we apply it to a sparse vector $\bx\in\mathbb{R}^{d'}$ to map it to a smaller dimensional vector $\by\in \mathbb{R}^p$ by applying $h$ to each input coordinate, and summing the results that map to the same coordinate: the $j^{\textrm{th}}$ coordinate of the output will equal $y_j=\sum_{i:h(i)=j} x_i$. Thus if the locations of non-zero entries of $\bx$ get hashed to distinct locations by $h$, then these entries will be preserved in our smaller-dimensional vector $\by$.

The size of the hash family $\Hcal$ and its output dimension $p$ are a function of $d'$, $\varepsilon$ and a tunable parameter $\gamma$ that decides the regime of $d'$ for which the aforementioned success holds. We start with a few preliminaries before proceeding with the actual construction and its properties.

\begin{definition}\label{defn:universal_hashing}
    We call a set $ \mathcal H$ of hash functions mapping from $U \to M$ a \emph{$\delta$-nearly universal hash function family} if for all $x \neq y \in U$,
    \[
        \pr_{h \sim \mathcal{U} \left( \mathcal H \right)} \left[ h(x) = h(y)\right] \leq \delta.
    \]
\end{definition}

\begin{definition}[Vector representation of indices] \label{defn:vector_rep_of_indices} Given positive integers $d',n,p'$ 
such that $n \geq \log_p d'$, we can represent any number $i\in[d']$ as an $n$-digit number in base $p$, considered as a vector with $n$ entries. We denote the base-$p$ representation of $i$ as $\base_{[p]}(i)$, and, using standard indexing notation, the $j^{\textrm{th}}$ digit of this base-$p$ representation is $\base_{[p]}(i)_j$.
\end{definition}

Next we construct the hash function family which will help us create \emph{small} hash function families with the desired properties as stated at the start of the section.

\begin{definition}[Vectorized Hashing] \label{const:vectorized_hashing} Given any $d' \in \mathbb{Z}^{>0}$ and $\varepsilon' \in (0,1)$, we let $p$ be the smallest prime number greater than ${d'}^{\varepsilon'}$. We
define a hashing scheme mapping indices $[d'] \to \left[ p  \right]$ using the base $p$ vector representation $\base_{p}(i)$ of indices $i \in [d']$  using $n = \left \lceil \frac{1}{\varepsilon'}\right \rceil$ digits (Definition~\ref{defn:vector_rep_of_indices}). We define the hash function $h_a$ for each $a \in [p]$ as,
\[
    h_a(i) = a+\sum_{j=1}^n a^j \cdot\base_{p}(i)_j \mod p.
\]

Define $\mathcal H_{d',\varepsilon'}$ to consist of all such $h_a$, namely, 
$\mathcal H_{d',\varepsilon'} \coloneqq \left\{ h_a : a \in [p]\right\}$.
    
\end{definition}

In the following lemma we prove that the family of hash functions constructed in Definition~\ref{const:vectorized_hashing} is a $\delta$-nearly universal hash function family (Definition~\ref{defn:universal_hashing}) for some $\delta$ depending on the input and the output dimensions of the hash functions. 

\begin{lemma} \label{lemma:prime_construction_is_universal} For any $d' \in \mathbb{Z}$ and $\varepsilon'>0$, the hash function family $\mathcal H_{d',\varepsilon'}: [d'] \to \left[ p \right]$ defined in Definition~\ref{const:vectorized_hashing} is a  $\delta$-\emph{nearly universal hash function family} (Definition~\ref{defn:universal_hashing}) with $\delta=\frac{\left \lceil \frac{1}{\varepsilon'} \right \rceil}{d'^{\epsilon'}}$ where $p$ is defined to be the the smallest prime number larger than ${d'}^{\varepsilon'}$.
\end{lemma}

\begin{proof}
     Picking $h_a \in \mathcal H_{d',\varepsilon'}$ uniformly randomly, we can write the probability of collision between any $i,i' \in [d]$ with $i \neq i'$ as,
    \begin{equation}\label{eq:hash}
    \begin{split}
        \pr_{a \sim \mathcal{U}([p])} \left[ h_a(i) = h_a(i')\right] = \pr_{a \sim \mathcal{U}([p])} \left[ a+\sum_{j=1}^n a^j \base_{p}(i)_j \mod p = a+\sum_{j=1}^n a^j \base_{p}(i')_j \mod p\right] \\
        = \pr_{a \sim \mathcal{U}([p])} \left[ \sum_{j=1}^n a^j \base_{p}(i)_j - \base_{p}(i')_j) \mod p = 0\right].
    \end{split}
    \end{equation}

By our underlying construction  (Definition~\ref{const:vectorized_hashing}) and by our choice of $p$ as the smallest prime number greater than $ {d'}^{\varepsilon'}$ and $n= \left \lceil \frac{1}{\varepsilon'} \right \rceil$, we have $p^n \geq d^{\varepsilon'}$, implying that $n$-digit base-$p$ representation is unique for each $i\in [d']$; in other words $i \neq i' \implies \base_{p,n}(i) \neq \base_{p,n}(i')$. Thus, the expression $\sum_{j=1}^n a^j \p{ \base_{p}(i)_j - \base_{p}(i')_j} \mod p$ is a non-zero polynomial mod $p$ (in the variable $a$) of degree at most $n$, and hence has at most $n=\left \lceil \frac{1}{\varepsilon'} \right \rceil$ roots. Namely, at most $n$ out of the $p$ hash functions in $ \mathcal H_{d',\varepsilon'}$ make $i$ collide with $i'$, giving us a bound on Equation~\ref{eq:hash} of

    \begin{equation*}
    \begin{split}
        \pr_{a \sim \mathcal{U}([p])} \left[ h_a(i) = h_a(i')\right] \leq \frac{n}{p}\leq \frac{\left \lceil \frac{1}{\varepsilon'} \right \rceil}{{d'}^{\varepsilon'}}, 
    \end{split}
    \end{equation*}
where the final inequality is by noting from the definition of $p$ that $p\geq {d'}^{\varepsilon'}$, concluding the proof.

\end{proof}

In our construction we will want to hash a $(d', \varepsilon)$ sparse vector $\bx$ to a lower dimension in a way such that no collisions occur between any of its non-zero entries. Since we cannot induce randomness in a neural network we ``derandomize" this by hashing $\bx$ with \emph{all} the hash functions $h$ from a small  yet well-behaved hash function family $\mathcal{H}$. In the following lemma we show that the family of Definition~\ref{const:vectorized_hashing} will in fact allow us to hash $(d', \varepsilon)$ sparse vectors without collisions.

\begin{lemma} \label{lemma:existence_of_success}
    For any $d'\in \mathbb{Z}^{>0}$ given $\varepsilon, \gamma \in (0,1)$, if $d' \geq  \left( \frac{1}{2} \left \lceil \frac{1}{2 \varepsilon+\gamma} \right \rceil \right)^{1/\gamma}$ then let $p$ to be the smallest prime number larger than ${d'}^{2\varepsilon+\gamma}$, and consider the hash function family $\mathcal{H}_{d',2\varepsilon+\gamma}$ (Definition~\ref{const:vectorized_hashing}), where each $h\in \mathcal{H}_{d',2\varepsilon+\gamma}$ maps the index set $[d']$ to the set $\left[ p \right]$. Then for any $(d',\varepsilon)$ sparse vector $\bx'$, there exists $h_{\star} \in \mathcal H_{d',2\varepsilon+\gamma}$ that maps the non-zero entries in $\bx'$  
    to distinct locations in $\left[ p\right]$,
     i.e., $h_{\star}$ hashes $\bx' \in \mathbb{R}^d$ to $\bx'' \in \mathbb{R}^{p}$ without any collisions between non-zero entries of $\bx'$. Formally,
     \[
        \forall \text{ $(d', \varepsilon)$ sparse $\bx'$}, \quad \exists h_{\star} \in \mathcal H_{d',2\varepsilon+\gamma} \; \; s.t \; \; \forall \; i \neq j,\; s.t. \; x'_i \neq 0 \neq x'_j \text{ we have } h_{\star}(i) \neq h_{\star}(j).
     \]    
\end{lemma}

\begin{proof}
    At a high level, the proof relies on the fact $\mathcal H_{d', 2\varepsilon+\gamma}$ is a $\delta$-nearly universal hash function family with $\delta=\frac{\left \lceil \frac{1}{2\varepsilon+\gamma} \right \rceil}{{d'}^{2\varepsilon+\gamma}} $ (Lemma~\ref{lemma:prime_construction_is_universal}), which can be used to show that, for any $(d,\varepsilon)$ sparse vector in $ \mathcal H_{d', 2\varepsilon+\gamma}$, the probability that no collision-free $h_{\star}$ exists  is $< 1$. We use the shorthand $\mathcal H $ as a substitute of $ \mathcal H_{d', 2\varepsilon+\gamma}$ for convenience. 
    
    Formally, given any $(d',\varepsilon)$ sparse vector $\bx$, if we uniformly randomly sample a hash function $h$ from $\mathcal H$ then the probability that any two non-zero entries get hashed to the same location) is bounded as
    \begin{equation*}
    \begin{split}
        \pr_{h \sim \mathcal{U}(\mathcal H)} \left[ \text{Failure} \right] =\pr_{h \sim \mathcal{U} (\mathcal H)} \left[ \exists x_i\neq 0  \neq x_j; h(i) = h(j)\right] \leq \sum_{\substack{i,j \\ i \neq j \\  x_i\neq 0 \neq x_j}} \pr_{h \sim \mathcal{U} (\mathcal H)} \left[  h(i) = h(j)\right]  \\
        \leq \binom{{d'}^{\varepsilon}}{2} \delta < \frac{{d'}^{2\varepsilon}}{2} \frac{\left \lceil \frac{1}{2\varepsilon+\gamma} \right \rceil}{{d'}^{2\varepsilon+\gamma}}  = \left \lceil \frac{1}{2\varepsilon+\gamma} \right \rceil\frac{1}{2{d'}^{\gamma}} \leq 1,
    \end{split}
    \end{equation*}
    where the second inequality is from the fact that $\mathcal{H}$ is $\delta$-almost universal, taking a union bound over all $\leq \binom{{d'}^{\varepsilon}}{2}$ pairs of $\{i,j\}$ with non-zero entries, which is strictly less than $\frac{{d'}^{2\varepsilon}}{2}$, as used in the third inequality;
    and the last inequality is from the assumption $d' \geq \left( \frac{1}{2} \left \lceil \frac{1}{2 \varepsilon+\gamma} \right \rceil \right)^{1/\gamma}$. Since we have bounded the probability over $h\in\mathcal{H}$ of \emph{any} collision as strictly less than 1, we conclude that there exists $h_{\star} \in \mathcal H$ which has \emph{no collisions} on the non-zero entries of the given $\bx'$, as desired. 
    
\end{proof}

\section{Neural network constructions and properties} \label{sec:helper_nets}

Recall Definition~\ref{notn:sdelta}. In this section we describe all the components we assemble to produce the neural networks of our main results. We explicitly show how to implement various natural primitives with shallow neural networks, along with a few more technical primitives needed for the specifics of our algorithms. The particular ``circuitry'' required by these neural networks is sometimes slightly intricate, but overall none of these constructions should be surprising, and this section can be skipped on a first read.

We state each neural network in a definition, and describe its properties immediately afterwards in a lemma.

Sometimes our neural networks may make use of data from previous layers, which could be thought of as using a ``skip connection'' where the result of a previous layer is sent along a wire that ``skips'' some intermediate layer(s). However, ``skip connections'' may lead to subtleties in the definition of width and therefore we do \emph{not} use them in this paper. Instead, we can emulate a skip connection by the \emph{identity transform}, where an input $x$ is preserved across a ReLU layer via adding 2 neurons to the width, using the identity $x=[x]_+ - [-x]_+$.
 
\subsection{Maximum neural network}

\begin{definition} \label{defn:max_net}
    We define our \emph{maximum} neural network $\mathcal{N}^{MAX}: \mathbb{R}^2 \to \mathbb{R}$ as
    \[
        \mathcal{N}^{MAX} \left( x_1,x_2\right) =  \relu{x_2} - \relu{-x_2} + \relu{x_1-x_2}. 
    \]
\end{definition}

\subsection{Comparison neural network}
We next define a neural network that tests whether $x_1>x_2$; this network is parameterized by a tolerance $\delta$, where the network will return the correct 0 or 1 answer \emph{unless} $x_1,x_2$ are within $\delta$ of each other.

\begin{definition} \label{defn:comparison_net}
    We define our \emph{comparison} neural network $\Ncal_{\delta}^{ C}: \mathbb{R}^2 \to \mathbb{R}$ as,
    \[
        \Ncal_{\delta}^{ C} \left( x_1,x_2\right) =  \relu{\frac{1}{\delta} \left(x_1-x_2\right)} -  \relu { \frac{1}{\delta}\left(x_1 - x_2 - \delta\right)}
    \]
\end{definition}

\begin{fact}
\label{fact:comparison_net} 
Given $\delta>0$, for all inputs $x_1,x_2\in\mathbb{R}$ we have
\[
    \Ncal_{\delta}^{ C} \left(x_1,x_2 \right)
    =\begin{cases}
        1, & \text{If } x_1>x_2 \text{ and } \left | x_1-x_2 \right| \geq \delta,\\
        0, & \text{If } x_1 \leq x_2 , \\
        \frac{1}{\delta} \left( x_1-x_2 \right) \text{ otherwise }.
    \end{cases}
\]
Additionally, this neural network has 1 hidden layer, employs 2 hidden neurons, and the magnitude of weights is upper bounded by $\frac{1}{\delta}$. Lastly, the output is always in the interval $[0,1]$.
\end{fact}

\subsection{Non-zero counter neural network}

One important primitive in our algorithm is counting the number of non-zero entries in a nonnegative vector, which we implement by repeated application of the comparison neural network $\Ncal_\delta^{C}$ of Definition~\ref{defn:comparison_net}. Parameterized by a tolerance $\delta>0$, we define the \emph{non-zero counter} neural network $\Ncal_{\delta}^{NZC}$ to output the number of non-zero entries in a vector $\bx \in \mathbb{R}^{d'}$, provided that $\forall i, x_i \in \{0\} \cup [\delta, \infty)$.

\begin{definition} \label{defn:non_zero_counting_net}
    We define our \emph{non-zero counter} neural network $\Ncal_{\delta}^{NZC}:  \mathbb{R}^{d'} \to \mathbb{R} $ as,
    \begin{equation*}
    \begin{split}
        \textrm{Given } \bx \in \mathbb{R}^{d'}:\quad
        \Ncal_{\delta}^{NZC} (\bx) = \sum_{i=1}^{d'} \Ncal_{\delta}^{ C} \left( x_i,0\right). 
    \end{split}
    \end{equation*}
\end{definition}

\begin{lemma} \label{lemma:non_zero_counting_net_output}
    The \emph{non-zero counter} neural network $\Ncal_{\delta}^{NZC}$ (Definition~\ref{defn:non_zero_counting_net}) when given an input $\bx \in \mathbb{R}^{d'}$ such that $\forall i, x_i \in \{0\} \cup [\delta, \infty)$ outputs the number of non-zero entries:
    \begin{equation*}
    \begin{split}
        \Ncal_{\delta}^{NZC} (\bx) = \sum_{i=1}^{d'} \mathbbm{1}\{x_i > 0 \}.
    \end{split}
    \end{equation*}

    The \emph{non-zero counter} neural network requires $1$ hidden layer, a width of $2d'$, and the magnitudes of the weights are bounded by $\frac{1}{\delta}$.
        
\end{lemma}
\begin{proof}
    The output is evident from the use of \emph{comparison} neural network (Definition~\ref{defn:comparison_net}, $\Ncal^C_{\delta}$) and the assumption $\forall i , x_i \in \{0\} \cup [\delta, \infty) $. It is clear from Definition~\ref{defn:non_zero_counting_net} and Definition~\ref{defn:comparison_net} that $\Ncal_{\delta}^{NZC}$ only requires $1$ hidden layer and each $i\in[d']$ uses $2$ neurons. Also from the use of \emph{comparison} neural network (Definition~\ref{defn:comparison_net}, $\Ncal^C_{\delta}$) we get that the magnitude of weights required is $\frac{1}{\delta}$.
    
\end{proof}

\subsection{Masking neural network} We will often want to create a binary ``mask'', recording which entries in a vector $\bx\in\mathbb{R}^{d'}$ lie in a given interval $[\ell,u]$. We make use of $2d'$ copies of our \emph{comparison} neural network to construct this \emph{masking} neural network, which will be correct as long as each input $x_i$ is \emph{either} inside the interval $[\ell,u]$ or $\delta$-far from the interval $[\ell,u]$.

\begin{definition} \label{defn:filtering_masking_net}
We define our \emph{masking} neural network $\Ncal^{M}_{\delta}:  \mathbb{R}^{d'} \times \mathbb{R} \times  \mathbb{R}\to \mathbb{R}^{d'} $ to have $i$'th output defined as,
\[\textrm{Given }\bx\in \mathbb{R}^{d'}, u > \ell,i\in[d']:\quad \Ncal^{M}_{\delta} \left( \bx,u,\ell\right)[i]\coloneqq\Ncal_{\delta}^{C}(x_i,\ell-\delta)-\Ncal_{\delta}^{C}(x_i,u)\]

\end{definition}

\begin{lemma} \label{lemma:filtering_masking_net_output}
    The \emph{masking} neural network $\Ncal^{M}_{\delta}$ (Definition~\ref{defn:filtering_masking_net}) when given an input $\bx\in\mathbb{R}^{d'}$ such that $\forall i\in[d']$ we have $x_i\notin (\ell-\delta,\ell)\cup (u,u+\delta)$, will output a vector $\bx' \in \mathbb{R}^{d'}$ where,
    \begin{equation*}
    \begin{split}
        &x'_i = \mathbbm{1} \{\ell \leq x_i \leq u\} . \\
    \end{split} 
    \end{equation*}
    The  \emph{masking} neural network requires $1$ hidden layer, a width of $4d'$, and the magnitudes of the weights are bounded by $\frac{1}{\delta}$.
\end{lemma}
\begin{proof}
This network consists of $2d'$ copies of the \emph{comparison} neural network $\Ncal_\delta^C$ of Definition~\ref{defn:comparison_net}, each copy of which has 1 hidden layer, width 2, and magnitude of weights bounded by $\frac{1}{\delta}$ by Fact~\ref{fact:comparison_net}, which yields our depth, width, and weight bounds. Correctness for each $i$ follows from the correctness of $\Ncal_\delta^C$, described in Fact~\ref{fact:comparison_net}.
\end{proof}

\subsection{Filtering neural network}

We next define a neural network that is intuitively similar to the \emph{masking} network, except for $x_i\in[\ell,u]$ we will return $x_i$ itself instead of 1 (and 0 if $x_i$ is $\delta$-far from the interval). We call this a \emph{filtering} neural network---it filters out those $x_i$ not in the interval $[\ell,u]$ and leaves $x_i\in[\ell,u]$ unchanged.

While the \emph{filtering} neural network is analogous to the \emph{masking} neural network, we will implement it different, explicitly defining a piecewise-linear function of $x_i$ with 4 breakpoints. The input requirements for this network will be analogous to those for the \emph{masking} network except with the additional requirement that all inputs $x_i\in[0,1]$.

\begin{definition} \label{defn:filtering_net}
We define our \emph{filtering} neural network $\Ncal^{F}_{\delta}:  \mathbb{R}^{d'} \times \mathbb{R} \times  \mathbb{R}\to \mathbb{R}^{d'} $ to have $i$'th output defined as,
\begin{equation*}
    \begin{split}
        &\textrm{Given }\bx\in \mathbb{R}^{d'}, u > \ell, i\in[d']:\\
        \Ncal^{F}_{\delta} \left( \bx,u,\ell\right)[i] &\coloneqq 
           \left[ \frac{1}{\delta} \left( x_i -\ell \right) +x_i\right]_+ -\left[ \frac{1}{\delta} \left(x_i - \ell  \right) \right]_+   -\left[ \frac{1}{\delta} \left( x_i - u\right) \right]_+ + \left[ \frac{1}{\delta}\left( x_i -  u \right) - x_i\right]_+   .
        \end{split} 
\end{equation*}

\end{definition}

\begin{lemma} \label{lemma:filtering_net_output}
    The \emph{filtering} neural networks $\Ncal^{F}_{\delta}$ (Definition~\ref{defn:filtering_net}) for $\delta\in (0,1]$, when given an input $\bx \in [0,1]^{d'}$ and numbers $0 \leq \ell<u\leq 1 $ such that $\forall i\in[d']$ we have $x_i\notin (\ell-\delta,\ell)\cup (u,u+\delta)$, 
    will output a vector $\bx' \in \mathbb{R}^{d'}$ where,
    \begin{equation*}
    \begin{split}
        &x'_i = x_i \cdot\mathbbm{1}{\{\ell \leq x_i \leq u\}} . \\
    \end{split} 
    \end{equation*}
    The \emph{filtering} neural network requires $1$ hidden layer, a width of $4d'$, and the magnitudes of the weights are bounded by $\frac{1}{\delta}+1$.
    
\end{lemma}

\begin{proof} 
To show correctness, we point out that $\Ncal^{F}_{\delta} \left( \bx,u,\ell\right)[i]$ is defined to be a piecewise-linear function of $x_i$ with 4 breakpoints at $\frac{\ell}{1+\delta},\ell,u,\frac{u}{1-\delta}$. To the left of the first breakpoint, the function equals 0; between the first two breakpoints the function has slope $\frac{1}{\delta}+1$ so attains value $\ell$ at the second breakpoint, $x_i=\ell$; when $x_i\in[\ell,u]$ the function has slope $1$ and thus will equal $x_i$; to the right of the third breakpoint $u$, the function will have slope $\frac{1}{\delta}-1$, and will thus be 0 for $x_i$ to the right of the last breakpoint. Our neural network will thus be correct for $x_i\in [\ell,u]$, and also for $x_i$ beyond the two outermost breakpoints. 

To show correctness for $x\leq \ell-\delta$, we have from above that the function is correct to the left of the first breakpoint, namely, for $x\leq \frac{\ell}{1+\delta}$, and we conclude from the fact that $\ell-\delta\leq \frac{\ell}{1+\delta}$, which we verify by multiplying through by $1+\delta$, moving all terms to the left, and dividing by $\delta$ to get the equivalent easily checked inequality that $\ell-1-\delta\leq 0$, since $\ell<1$ and $\delta>0$.

Finally, to show correctness for $x\geq u+\delta$ though under the guarantee that $x\leq 1$, we have from above that the function is correct to the right of the last breakpoint, namely, for $x\geq \frac{u}{1-\delta}$. Thus the only cases where the function could be wrong are those for which $u+\delta<\frac{u}{1-\delta}$; solving for $u$ shows that this can only be true when $u>1-\delta$; but since we only need to show correctness for $x\geq u+\delta$, this means we only need to show correctness when $x>1$, which is trivially true since these values of $x$ are out of range of our assumptions, thus showing the lemma.

This network uses $1$ hidden layer, a width of $4d'$, and the magnitudes of the weights are bounded by $\frac{1}{\delta}+1$.
\end{proof}

\subsection{Indicator function product neural network}

We introduce a simple neural network here that, given an input $x\in [0,1]$ and $s\in\mathbb{Z}$ will return $x\cdot \mathbbm{1}\{s=0\}$.

\begin{definition} \label{defn:indicator_net}
    We define our \emph{indicator function product} neural network $\Ncal^{IFP}:  \mathbb{R} \times \mathbb{R} \to \mathbb{R} $ as,
    \begin{equation*}
    \begin{split}
        &\textrm{Given } x,s \in \mathbb{R}, \quad
    \Ncal^{IFP} (x,s)= [x+s]_+ -[x+s-1]_+ -[s]_+ +[s-1]_+.
    \end{split}
    \end{equation*}
\end{definition}

\begin{lemma}
\label{lem:4-term}
The \emph{indicator function product} neural network $\Ncal^{IFP}$ (Definition~\ref{defn:indicator_net}), when given input $x\in[0,1]$ and integer $s\in\mathbb{Z}$ outputs \[\Ncal^{IFP}(x,s)=x\cdot \mathbbm{1}\{s=0\}.\]
The \emph{indicator function product} neural network $\Ncal^{IFP}$ has 1 hidden layer, width 4, and magnitude of weights 1.

Further, for any $x,s\in\mathbb{R}$, the output of the neural network has magnitude at most $|x|$.
\end{lemma}
\begin{proof}
For convenience, we define the function $\mathrm{trim}_{[0,1]}(y)$ to ``trim'' the real number $y$ to the range $[0,1]$, defined equivalently as $\mathrm{trim}_{[0,1]}(y)\coloneqq\max(0,\min(1,y))$. 
It is straightforward to check that trimming can be implemented as the difference of two ReLU units: for a real number $y$, we have $[y]_+-[y-1]_+=\mathrm{trim}_{[0,1]}(y)$. We use this relation twice in the equation defining $\Ncal^{IFP}(x,s)$ to see that $\Ncal^{IFP}(x,s)$ equals \begin{equation}\label{eq:4-term}\mathrm{trim}_{[0,1]}(x+s)-\mathrm{trim}_{[0,1]}(s)\end{equation}

Recall that we assume $x\in[0,1]$. When $s=0$, Equation~\ref{eq:4-term} equals $\mathrm{trim}_{[0,1]}(x)=x$; when $s<0$, since $s\in\mathbb{Z}$, we have $s\leq -1$, and thus both terms of  Equation~\ref{eq:4-term} get trimmed to 0; and analogously, when $s>0$, both terms of Equation~\ref{eq:4-term} get trimmed to 1, and thus yield a difference of 0. Thus in all cases, Equation~\ref{eq:4-term} equals $x \cdot\mathbbm{1} \{s=0\}$, yielding correctness.

This neural network clearly has 1 hidden layer, 4 neurons, and weights of magnitude 1.

Finally, to bound the magnitude for arbitrary $x,s\in\mathbb{R}$, we point out that Equation~\ref{eq:4-term} equals 0 when $x=0$, and, when considered as a function of $x$, has Lipschitz constant 1 since the $\mathrm{trim}()$ function has Lipschitz constant 1. Thus $\Ncal^{IFP}(x,s)$ has magnitude at most $|x|$, as claimed.
\end{proof}

\subsection{Rank selection neural network}

We now define the \emph{rank selection} neural network, which implicitly sorts its input using an all-pairs quadratic-width approach, and returns elements of the desired ranks. This neural network is the central component of the depth 3, quadratic width median finding result in Theorem~\ref{theorem:depththreeconstruction}. While for the subsequent sub-quadratic median finding neural networks we cannot use this network directly on the entire input, this network is a crucial component once we have used other techniques to reduce the size of the input.

The idea behind this neural network is to compare all pairs of elements and use the number of comparisons ``won'' by each entry to determine if that entry has the designated rank, and should hence be returned. 

\begin{definition} \label{defn:rank_selection_net}
    We define our \emph{rank selection} neural network $\Ncal^{RS}_{\delta}:  \mathbb{R}^{d'} \times \mathbb{R}^{p} \to \mathbb{R}^{p} $ as,
    \begin{equation*}
    \begin{split}
        \textrm{Given } \bx \in \mathbb{R}^{d'}, \br \in \mathbb{R}^p, \quad \textrm{ and letting, } \quad y_k  = 1 + \sum_{j \in[d']} \Ncal^{ C}_{\delta}(x_k,x_j), \quad k \in [d'], \\
        \Ncal_{\delta}^{RS} \left( \bx,\br\right)[i] = 
        \sum_{k \in[d'] } \Ncal^{IFP}(x_k,r_i-y_k),
    \end{split}
    \end{equation*}
    where $\Ncal^{RS}_{\delta}\left( \bx, \br\right)[i]$ is the $i'$th coordinate of the output.
\end{definition}

\begin{lemma} 
\label{lemma:rank_selection_net_output}
    The \emph{rank selection} neural network $\Ncal^{RS}_{\delta}$ (Definition~\ref{defn:rank_selection_net}), when given an input $\bx \in [0,1]^{d'}$  where for all $j\neq k$ we have either $|x_j-x_k|\geq \delta$ or $x_j=x_k=0$, and given    
    a vector of ranks $\br \in \{1,\ldots,d'\}^p$, outputs a vector $\bx' \in \mathbb{R}^p$ with
    \[
        x'_i =\mathcal R_{r_i}(\bx),\quad \forall i\in[p].
    \]
    Further, given arbitrary input $\bx \in \mathbb{R}^{d'}$ we always have the bound
    \[
        |x'_i| \leq d'\norm{\bx}_{\infty}.
    \]
    Moreover, the \emph{rank selection} neural network $\Ncal^{RS}_{\delta}$ requires $2$ hidden layers, a width of at most $\max( 2{d'}^2+2d'+2p, 4pd')$ and the magnitudes of weights are bounded by $\frac{1}{\delta}$.
\end{lemma}

\begin{proof} 
We first show the correctness property.

The neural network first computes $y_k$, comparing $x_k$ with every $x_j$ and returning 1 plus the number of strictly smaller elements, yielding that $y_k$ will be the rank of $x_k$ in $\bx$. We implement this with the \emph{comparison} neural network $\Ncal^{C}_{\delta}$, which will return correct answers by Fact~\ref{fact:comparison_net} because all pairs $x_j,x_k$ are either identical or $\delta$-separated. (We point out that our input $\bx$ may have repeated zeros, and for all such elements we will compute $y_k=1$ because there are no strictly smaller elements.)

The $i^{\textrm{th}}$ element of the return vector is computed in the next line via a $d'$-way sum with respect to $k$ over $\Ncal^{IFP}(x_k,r_i-y_k)$, which we analyze with Lemma~\ref{lem:4-term}, since $x_k\in [0,1]$ and $y_k\in\mathbb{Z}$. By Lemma~\ref{lem:4-term} the neural network thus returns, as its $i^{\textrm{th}}$ entry, $\sum_{k\in [d']}x_k \cdot\mathbbm{1} \{y_k=r_i\}$. Namely, given as input a desired rank $r_i$, we return the sum of all entries $x_k$ whose rank (previously stored as $y_k$) equals $r_i$. Since all entries except 0 are unique, the neural network will correctly return the element of rank $r_i$.

Next, we bound the return value of the neural network for arbitrary $\bx\in\mathbb{R}^{d'}$: by Lemma~\ref{lem:4-term}, the expression $\Ncal^{IFP}(x_k,r_i-y_k)$ has magnitude at most $|x_k|$. Summing this over $k\in[d']$ immediately gives our desired universal bound $|x'_i| \leq d'\norm{\bx}_{\infty}$.

    For the depth of the network, we use two hidden layers, one corresponding to $\Ncal^C_{\delta}$ and the other for $\Ncal^{IFP}$. For the width in the first hidden layer, we compute $\Ncal^C_{\delta} \left( x_k,x_j\right)$ for all pairs $j,k \in [d']$ and since $\Ncal^C_{\delta} \left( x_i,x_j\right)$ requires $2$ neurons (Fact~\ref{fact:comparison_net}) this contributes width $ 2{d'}^2$ in the first layer. In the second layer we use inputs $\bx,\br$, and thus need to add 2 extra neurons in the first layer to compute the identity function for each value that we want to reuse later, contributing the remaining $2d'+2p$ to the width of the first layer. In the second hidden layer we use we use $p\cdot d'$ copies of $\Ncal^{IFP}$, thus requiring $4pd'$ neurons (by Lemma~\ref{lem:4-term}); taking the max of the neurons in the two layers gives us the stated width bound. The upper bound on the magnitude of the weights follows from the \emph{comparison} neural network (Definition~\ref{defn:comparison_net}, $\Ncal^C_{\delta}$) that requires weights of magnitude $\frac{1}{\delta}$ (Fact~\ref{fact:comparison_net}), since $\Ncal^{IFP}$ only has weights of magnitude 1 by Lemma~\ref{lem:4-term}.
    \end{proof}

\subsection{Non-zero element shortlisting neural network} 

We next define another key primitive that we exclusively use in our linear width construction. The \emph{non-zero element shortlisting} neural network, parameterized by a return size $p$, when given a vector $\bx$ and an interval $[\ell,u]$ will try to return $p$ non-zero entries of $\bx$ that lie in the interval $[\ell,u]$. This neural network uses the \emph{indicator function product} neural network (Definition~\ref{defn:indicator_net}) in a related manner to the previous construction of the \emph{rank selection} neural network, even though the end result is rather different.

\begin{definition} \label{defn:shortlist_non-zero_net}
    Given $\bx \in \mathbb{R}^{d'}$ and $p \in \mathbb{Z}$ we define our \emph{non-zero element shortlisting} neural network $\Ncal_{\delta,p}^{NZES}:  \mathbb{R}^{d'} \times \mathbb{R} \times \mathbb{R} \to \mathbb{R}^{p} $ as,
    \begin{equation*}
    \begin{split}
        \textrm{Given } \bx\in \mathbb{R}^{d'}, \ell < u, \; \textrm{let}\hspace{4cm}\; \ell' = \Ncal^{MAX}(\ell, \delta/2) \; (\text{Definition } \ref{defn:max_net}),\\ \by = \Ncal_{\delta/2}^{M}(\bx,u,\ell')  \; (\text{Definition } \ref{defn:filtering_masking_net}), \; \mathbf{f}=\Ncal^{F}_{\delta} \left( \bx,u,\ell\right)(\text{Definition } \ref{defn:filtering_net}) \\
\Ncal^{NZES}_{\delta,p} \left( \bx,u,\ell\right)[i] = \sum_{j \in [d']}\Ncal^{IFP}\left(f_j,i-\sum_{j'=1}^j y_{j'}\right)
   , i \in [p] \; (\text{Definition}~ \ref{defn:indicator_net}).
    \end{split}
    \end{equation*}
    where $\Ncal_{\delta,p}^{NZES} \left( \bx,u,\ell\right)[i]$ is the $i'$th coordinate of the output.
\end{definition}

\begin{lemma} \label{lemma:shortlisting_non-zero_net_output}
    The \emph{non-zero element shortlisting} neural network $n_{\delta,p}^{NZES}$ (Definition~\ref{defn:shortlist_non-zero_net}) takes as input $\bx \in \mathcal S^{d'}_{\delta}$, and numbers $\ell<u $ such that $\ell,u \in \bx \cup \{ 0,1\}$. Letting $p'$ denote the number of non-zero entries of $\bx$ lying in the interval $[\ell,u]$, the neural network returns $\bx' \in \mathbb{R}^p$ with
    \begin{equation*}
    \begin{split}
        &x_i'=  \begin{cases}
            x_j \quad 
            \text{(such that $x_j$ is the $i'$th non-zero entry that lies in the interval $[\ell,u]$)}, \quad i \in [\min(p,p')] \\
            0 \quad \text{if }\; i>\min(p,p').
        \end{cases}
    \end{split}
    \end{equation*}
    Moreover the \emph{non-zero element shortlisting} neural network requires $3$ hidden layers, a width of $\mathcal O \left(pd' \right)$ and the magnitudes of weights are bounded by $\frac{2}{\delta}$.
\end{lemma}

\begin{proof}
The proof is similar to the proof of Lemma~\ref{lemma:rank_selection_net_output}. We first show the correctness property.

By definition, we have $\ell'=\max(\ell,\delta/2)$. The \emph{masking} neural network $\Ncal^{M}_{\delta/2} \left( \bx,u,\ell'\right)$ will return a $\{0,1\}$ vector $\by$ recording for each $i\in [d']$ whether $x_i$ is a non-zero value in $[\ell,u]$: by Lemma~\ref{lemma:filtering_masking_net_output}, the output will be correct as long as the boundaries $\ell,u$ are either equal to or $\delta/2$-separated from each $x_i$; this will be satisfied since the non-zero values of $\bx$ lie in the range $[\delta,1-\delta]$ by definition of $\mathcal{S}^{d'}_{\delta}$, and $\ell'$ is rounded up to $\delta/2$ by the max function in the case $\ell=0$. The \emph{filtering} neural network $\Ncal^{F}_{\delta} \left( \bx,u,\ell\right)$ will correctly filter its input, returning $f_i=x_i \mathbbm{1}\{x_i\in[\ell,u]\}$ by Lemma~\ref{lemma:filtering_net_output}, since the boundaries $\ell,u$ are either equal to or $\delta$-separated from each $x_i$.

Since $\mathbf{f}$ contains entries in $[0,1]$ and $\by$ contains integer entries, we apply Lemma~\ref{lem:4-term} to conclude that  \[\Ncal^{NZES}_{\delta,p} \left( \bx,u,\ell\right)[i]=\sum_{j\in[d']} f_j \cdot \mathbbm{1}\{i=\sum_{j'=1}^j y_{j'}\}\]

Namely, defining for the purposes of analysis a vector $\bz$ whose $j^{\textrm{th}}$ entry equals $\sum_{j'=1}^j y_{j'}$, we have that $z_j$ counts how many of the first $j$ entries of $\bx$ are non-zero entries in the range $[\ell,u]$; the indicator function will evaluate whether this matches $i$ and return $f_j$ in this case. Since (as shown above) $f_j=x_j \mathbbm{1}\{x_j\in[\ell,u]\}$, we conclude that $\Ncal^{NZES}_{\delta,p} \left( \bx,u,\ell\right)[i]$ will return the $i^{\textrm{th}}$ non-zero entry of $\bx$ lying in the range $[\ell,u]$, if such an entry exists, concluding the correctness proof.

The bounds on the width, depth, and weights result from corresponding bounds on the components $\Ncal^{MAX}(\ell, \delta/2)$, $\Ncal_{\delta/2}^{M}(\bx,u,\ell')$, $\Ncal^{F}_{\delta} \left( \bx,u,\ell\right)$, and $\Ncal^{IFP}(f_j,i-\sum_{j'=1}^j y_{j'})$ from Lemmas~\ref{lemma:filtering_masking_net_output}, \ref{lemma:filtering_net_output}, \ref{lem:4-term} respectively.

\end{proof}

\subsection{Rank computing neural network} 

Recall that the idea of Algorithm~\ref{alg:informalAlgo} is to repeatedly ``filter'' $\bx$ by zeroing out elements not lying in an interval $[\ell,u]$, in such a way that $\med(\bx)$ is preserved, while returning a much sparser vector $\by$. A crucial component of this is determining the rank $r$ of $\med(\bx)$ in (the non-zero portion of) the new vector $\by$, which we do in the \emph{rank computing} neural network. We compute this rank $r$ taking an element $e\in\by^{\neq 0}$ and noticing that, since $\by^{\neq 0}$ is a contiguous portion of $\bx$, the difference in ranks of $e,\med(\bx)$ in $\bx$ equals the difference of their ranks in $\by^{\neq 0}$, and since we know the rank of $\med(\bx)$ in $\bx$ equals $d'/2$, we can recover $r$ by computing the rank of $e$ in both vectors, and solving for $r$.

\begin{definition} \label{defn:rank_computing_net}
    Given $\bx \in \mathbb{R}^{d'}, \by \in \mathbb{R}^{d''}$, and an element $e \in \mathbb{R}$  we define our \emph{rank-computing} neural network $\Ncal_{\delta}^{RC}: \mathbb{R}^{d'} \times \mathbb{R}^{d''} \times \mathbb{R}  \to \mathbb{R} $ as,
    \begin{equation*}
    \begin{split}
        &\textrm{Given } \bx \in \mathbb{R}^{d'}, \by \in \mathbb{R}^{d''}, e \in \mathbb{R}, \\
        \Ncal_{\delta}^{RC} \left( \bx,\by,e\right) &= d'/2 -d''+\Ncal_\delta^{NZC}(\by)+\sum_{j \in [d'']} \Ncal_{\delta}^{C} \left( e,y_j\right) - \sum_{j \in [d']} \Ncal_{\delta}^{C} \left( e,x_j\right).
    \end{split}
    \end{equation*}
\end{definition}

\begin{lemma} \label{lemma:rank_computing_net_output}
    The \emph{rank-computing} neural network $\Ncal_{\delta}^{RC}$ (Definition~\ref{defn:rank_computing_net}) when given an entirely non-zero input $\bx \in \mathcal S^{d'}_{\delta} $ and another input $ \by \in \mathcal S^{d''}_{\delta}$, whose non-zero entries form a contiguous block in a sorted version of $\bx$ containing $\med(\bx)$ along with a third input $e$ that is any non-zero entry of $\by$, the network outputs a rank $r$ such that
    \[
        \mathcal R_{r} \p{\by^{\neq 0}} = \med(\bx).
    \]
    Moreover the \emph{rank-computing} requires $1$ hidden layer, a width of $\mathcal O(\max(d',d''))$ and the magnitudes of weights are bounded by $\frac{1}{\delta}$. 
\end{lemma}
\begin{proof} Recalling that input $e$ is any non-zero entry of $\by$, let $r_{e,\bx}$ denote its rank among the entries of $\bx$ and let $r_{e,\by}$ denote the rank of $e$ among the entries of $\by^{\neq 0}$. Further, let $r$ be the rank of $\med(\bx)$ among the non-zero entries of $\by$, where the rank of $\med(\bx)$ in $\bx$ equals $d'/2$ by definition of the median. Since both $e$ and $\med(\bx)$ lie in the contiguous block $\by^{\neq 0}$, when $\bx$ is sorted, we have that the difference of ranks of $e$ and $\med(\bx)$ in $\bx$ equals the difference of ranks of these elements in $\by^{\neq 0}$, yielding

    \begin{equation}\label{eq:solve-for-r}
        \begin{split}
           r_{e,\bx} - d'/2=r_{e,\by}-r.
        \end{split}
    \end{equation}

    We use this expression to show that our neural network $\Ncal_{\delta}^{RC}$ correctly computes $r$. Since $\bx \in \mathcal S^{d'}_{\delta}$ and non-zero, all the elements of $\bx$ are $\delta$-separated, and thus we compute the rank of $e$ in $\bx$ as $r_{e,\bx}  = 1+\sum_{j \in [d']} \Ncal_{\delta}^{C} \left( e,x_j\right) $, by applying Fact~\ref{fact:comparison_net} to correctly count 1 plus the number of $x_j$ that are smaller than $e$. Analogously, the rank of $e$ in $\by$ (including the zero entries for the moment) equals $1+\sum_{j \in [d'']} \Ncal_{\delta}^{C} \left( e,y_j\right) $; and the number of zero entries is found by taking $d''$ minus the result from the \emph{non-zero counting} neural network (Definition \ref{defn:non_zero_counting_net}) $\Ncal_\delta^{NZC}(\by)$. Subtracting yields that the rank of $e$ in $\by^{\neq 0}$ equals $r_{e,\by}=1+\sum_{j \in [d'']} \Ncal_{\delta}^{C} \left( e,y_j\right) - \left(d''-\Ncal_\delta^{NZC}(\by)\right)$. Solving for $r$ in Equation~\ref{eq:solve-for-r}, we find that $r$ is exactly the expression computed by our overall neural network.

Since our neural network applies $\Ncal^{NZC}_{\delta}$ once, to an input of size $d''$, and, in parallel applies $\Ncal^C_{\delta}$ $d'+d''$ times, we have (from Fact~\ref{fact:comparison_net} and Lemma~\ref{lemma:non_zero_counting_net_output}) that this network has 1 hidden layer, a width of $\mathcal O(\max(d',d''))$ and the magnitudes of weights bounded by $\frac{1}{\delta}$.
\end{proof}

\subsection{Rank scaling neural network} In our construction, one of the intermediate steps is to compute an analog rank $r'$, of a particular rank $r$ among a small random subset of entries of the non-zero entries. To do this we need to multiply the rank with the size of this subset and divide it by the number of non-zero entries which is non-trivial as the number of non-zero entries in our construction is a random quantity and we cannot pre-compute it. The \emph{rank scaling} neural network helps us do this by simply computing all possible values $r'$ can take and zeroing out all but the one that is the true value. Given $b \in [d']$, a rank $r$ of a coordinate among the non-zero entries of $\bx \in \mathcal S_{\delta}^{d'}$ and assuming $|\bx^{\neq 0}| \geq 1$, this outputs $r'$ defined as,
\[
    r' = \frac{rb}{|\bx^{\neq 0}|}.
\]

\begin{definition} \label{defn:rank_scaling_net}
    Given $\bx \in \mathbb{R}^{d'}$ and $r \in \mathbb{R}$ we define our \emph{rank scaling} neural network $\Ncal^{RSC}_{\delta, b}:  \mathbb{R}^{d'} \times \mathbb{R} \to \mathbb{R} $ as,
    \begin{equation*}
    \begin{split}
        \textrm{Given } \bx \in \mathbb{R}^{d'}, r \in \mathbb{R}, \textrm{ let}\quad\quad\quad\quad z=\Ncal_{\delta}^{NZC} (\bx),\\
        r'= \sum_{k=1}^{d'} \frac{d'b}{k} \cdot \Ncal^{IFP}\left(\frac{r}{d'},z-k\right) \quad \text{(Definition~\ref{defn:indicator_net})}.
    \end{split}
    \end{equation*}
\end{definition}

\begin{lemma}  \label{lemma:rank_scaling_net_output}
    The \emph{rank scaling} neural network $\Ncal^{RSC}_{\delta, b}$ (Definition~\ref{defn:rank_scaling_net}) when given $\bx \in \mathcal S^{d'}_{\delta} $ with $|\bx^{\neq 0}| \geq 1$, and an input $r\in [0,d']$, outputs $r'$ where
    \begin{equation*}
    \begin{split}
        r' = \frac{r b}{|\bx^{\neq 0}|}.
    \end{split}
    \end{equation*}
    Moreover \emph{rank scaling} neural network requires $2$ hidden layers, a width of $4d'$ and magnitude of weights bounded by $\max \p{d'b,\frac{1}{\delta}}$. 
\end{lemma}

\begin{proof} Since $\bx \in \mathcal S^{d'}_{\delta}$ we have by Lemma~\ref{lemma:non_zero_counting_net_output} that $z=\Ncal_{\delta}^{NZC} (\bx)$ correctly computes the number of non-zero entries in $\bx$, namely $|\bx^{\neq 0}|$. In the next line, we apply $\Ncal^{IFP}$, where since $\frac{r}{d'}\in[0,1]$ and both $z,k$ are integers, we have by Lemma~\ref{lem:4-term} that the output value satisfies $r'=\sum_{k=1}^{d'} \frac{d'b}{k} \frac{r}{d'}\cdot \mathbbm{1}\{k=z\}$. Thus the sum has a non-zero contribution only for the term where $k=|\bx^{\neq 0}|$, and thus the neural network outputs $r'=\frac{rb}{|\bx^{\neq 0}|}$, as desired.

The number of hidden layers required by $\Ncal_{\delta}^{NZC}$ is $1$, from Lemma~\ref{lemma:non_zero_counting_net_output}, and 1 more for $\Ncal^{IFP}$ from Lemma~\ref{lem:4-term}. The width of the first layer is $2d'$ from Lemma~\ref{lemma:non_zero_counting_net_output}, and the second layer has width $4d'$ from Lemma~\ref{lem:4-term}. The magnitude of weights required by $\Ncal_{\delta}^{NZC}$ is $1/\delta$, and $\Ncal^{IFP}$ has weights bounded by 1 by Lemma~\ref{lem:4-term}, but we scale the output by $\frac{d'b}{k}\leq d'b$, to give an overall bound of $\max(d'b,1/\delta)$.
    \end{proof}

\subsection{Ceiling neural network} In each step of our construction, we compute ranks scaled by certain quantities to obtain an analogous rank $r'$ of the \textit{median} in a smaller subset. Since ranks can only be integers we convert this to an integer through the \textit{ceiling} ($\ceil{.}$) or \textit{floor} ($\floor{.}$) operations using the \emph{ceiling} neural network (note that $\floor{x} = -\ceil{-x}$). We are able to do this by exploiting the finite structure of the set of values $r'$ can take. The \emph{ceiling} neural network when input $\frac{a}{b}$ where $b \in [d']$ and $a \in \mathbb{Z}$ such that $-d'  \leq \frac{a}{b}  \leq d'$ computes $\ceil{ \frac{a}{b} }$.

\begin{definition} \label{defn:ceil_net}
    We define our \emph{ceiling} neural network $\Ncal^{CEI}_{d'}:  \mathbb{R} \to  \mathbb{R} $ 
    as,
    \begin{equation*}
    \begin{split}
        \textrm{Given } x \in \mathbb{R}, \\
        \Ncal^{CEI}_{d'} (x) =  -d'+\sum_{i \in \{-d', \ldots,d' \}}   d' \left(\left[  x- i\right]_+ - \left[  x-  i  - \frac{1}{ d'}\right]_+ \right).
    \end{split}
    \end{equation*}
\end{definition}

\begin{lemma}  \label{lemma:ceil_net_output}
    The \emph{ceiling} neural network $\Ncal^{CEI}_{d'}$ (Definition~\ref{defn:ceil_net}), when given input $x\in[-d',d']$ that is also rational number $x=\frac{a}{b}$ with denominator $b \in \{1,\ldots,d'\}$ will output the ceiling of $x$,
    \[
        \Ncal^{CEI}_{d'}\left(x \right) = \ceil{ x }.
     \]
    Moreover, the \emph{ceiling} neural network requires $1$ hidden layer, a width of at most $4(d'+1)$ and magnitude of weights $ d'$.
\end{lemma}

\begin{proof} 
We analyze the expression inside the sum for rational $x$. In the case $x\leq i$, then we have $d' \left(\left[  x- i\right]_+ - \left[  x-  i  - \frac{1}{ d'}\right]_+ \right)=0$. On the other hand, if $x>i$, for $i\in\mathbb{Z}$ and $x$ a rational number with denominator at most $d'$, then $x>i+\frac{1}{d'}$, yielding $d' \left(\left[  x- i\right]_+ - \left[  x-  i  - \frac{1}{ d'}\right]_+ \right)=1$. Thus in general, this expression in the sum equals the indicator function $\mathbbm{1}\{x>i\}$. Thus our neural network computes $\Ncal^{CEI}_{d'} (x) =  -d'+\sum_{i \in \{-d', \ldots,d' \}} \mathbbm{1}\{x>i\}$, which for $x\in[-d',d']$ must return exactly $\lceil x\rceil$, as desired.

    The number of hidden layers, the width, and the magnitudes of weights are clear from the definition.
    \end{proof}

\subsection{Hashing neural network}

 In our construction, in one of the steps we reduce the dimensionality of our problem by hashing the locations in a sparse vector to a smaller-dimensional vector, according to hash functions $h$ belonging to some family $\mathcal H$. Each such hash function is an explicit linear transformation with $0,1$ coefficients: for a fixed hash function $h:[d'] \to [p']$, and given an input $\bx\in \mathbb{R}^{d'}$, we output $\by \in \mathbb{R}^{p'}$ where $y_i = \sum_{j:\,h(j)=i} x_j$.
 This linear transform uses 0 ReLU layers in a neural network; however, we choose to include an artificial ReLU layer in our construction to separate this linear transform from any subsequent processing that occurs. This helps us to treat this neural network as an independent primitive in our construction, simplifying the explanation of our construction.

\begin{definition} \label{defn:hashing_net}
    Given a hash function $h:[d'] \to [p']$, we define our \emph{hashing} neural network $\Ncal^{H}_{h}:  \mathbb{R}^{d'} \to \mathbb{R}^{p'} $ as,
    \begin{equation*}
    \begin{split}
        \textrm{Given } \bx \in \mathbb{R}^{d'}, \\
        \Ncal^{H}_{h}(\bx)[i] = \left[\sum_{j:\,h(j)=i} x_j \right]_+.
    \end{split}
    \end{equation*}
\end{definition}

\begin{lemma} \label{lemma:hashing_net_output}
    The \textbf\textit{{hashing}} neural network $\Ncal^{H}_{h}  $ (Definition~\ref{defn:hashing_net}) when given an non-negative input $\bx \in \mathbb{R}^{d'}$ outputs $\by \in \mathbb{R}^{p'}$ where each coordinate $y_i$ is the sum of entries of $\bx$ hashed to position $i$.
    Moreover, \textbf\textit{{hashing}} neural network require one hidden layer,  a width of $p'$ and weights in $\{ 0,1\}$. 
        
\end{lemma}

\begin{proof} The first statement is true by definition and noting that the entries of $\bx$ are positive, while the second statement is straightforward from the definition.
    \end{proof}

\subsection{Block extraction neural network}

Recall that Algorithm~\ref{alg:hashingFunction} takes as input a sparse vector with $s$ non-zero entries, and aims to return its non-zero entries by first trying a few hash functions to hash the locations to a smaller domain, and then identifying a hash function that leads to no collisions, so that we can extract the non-zero entries with small width. In this section we describe the neural network that looks through the results of applying $q$ different hash functions, each mapping to a set of size $p$, and identifies the results of the first hash function that has led to zero collisions. Explicitly, we describe a neural network that takes as input $q$ blocks of size $p$, and returns the first of these blocks that has exactly $s$ non-zero entries; we call this the \emph{block extraction} neural network. This network essentially consists of three layers of indicator functions, assembled via appropriate linear transforms to compute the desired output.

\begin{definition}\label{def:extract-big-block}
We define our \emph{block extraction} neural network $\Ncal^{BE}_{\delta,p}:\mathbb{R}^{pq}\times\mathbb{R}\rightarrow\mathbb{R}^p$ by

    \begin{equation*}
    \begin{split}
        \textrm{Given } \bx \in \mathbb{R}^{pq}, s \in \mathbb{R}, \\
        c_i = \Ncal_\delta^{NZC}(x_{p(i-1)+1},\ldots,x_{p\cdot i}),\, \forall i\in[q],\\
        m_i = \Ncal_\delta^C(c_i,s-1)-\Ncal_\delta^C(c_i,s),\, \forall i\in[q],\\
\Ncal^{BE}_{\delta,p}(\bx,s)[j]=\sum_{i=1}^q \Ncal^{IFP}\left(x_{j+(i-1)p},m_i-1-\sum_{i'=1}^{i-1}m_{i'}\right),\forall j\in[p]   
    \end{split}
    \end{equation*}
    where $\Ncal^{BE}_{\delta,p}(\bx,s)[j]$ is the $j$'th coordinate of the output.
\end{definition}

\begin{lemma}\label{lem:extract-big-block}
The \emph{block extraction} neural network $\Ncal^{BE}_{\delta,p}:\mathbb{R}^{pq}\times\mathbb{R}\rightarrow\mathbb{R}^p$, when given an input $\bx\in \mathbb{R}^{pq}$ all of whose entries are in the set $\{0\}\cup [\delta,1]$, and given an input $s\in\mathbb{Z}$, will consider $\bx$ as being divided into $q$ blocks of size $p$, and will return a copy of the first block that contains exactly $s$ non-zero entries (returning zeros if no such block exists).

The \emph{block extraction} neural network $\Ncal^{BE}_{\delta,p}$ requires 3 hidden layers, has width $\Ocal(pq)$, and the magnitudes of the weights are bounded by $\frac{1}{\delta}$.
\end{lemma}

\begin{proof}
We first apply the \emph{non-zero counter} neural network $\Ncal_{\delta}^{NZC}$ (Definition~\ref{defn:non_zero_counting_net}) to each of the $q$ blocks of the input, correctly storing in $c_i$ the number of non-zero entries in block $i$, by Lemma~\ref{lemma:non_zero_counting_net_output}. Next, for each $c_i$ we compute in $m_i$ the indicator value of whether $c_i=s$, which we compute by two applications of the comparison neural network $\Ncal_d^{C}$ (defined in Definition~\ref{defn:comparison_net} and shown correct in Fact~\ref{fact:comparison_net}).

The final output step is the most intricate. Recall that the \emph{indicator function product} neural network $\Ncal^{IFP}$, on input a real number $y\in[0,1]$ and an integer $r$, returns $y\cdot \mathbbm{1}\{r=0\}$ (see Definition~\ref{defn:indicator_net} and Lemma~\ref{lem:4-term}). Thus the $j$'th output of our neural network equals $\sum_{i=1}^q x_{j+(i-1)p}\cdot \mathbbm{1}\left\{m_i=1+\sum_{i'=1}^{i-1}m_{i'}\right\}$. We analyze the indicator function: since $m_i$ each $m_i$ is either 0 or 1, the indicator function condition ``$m_i=1+\sum_{i'=1}^{i-1}m_{i'}$'' will be true only when $m_i$ is 1, and when $i$ is the \emph{first} such index (so that $\sum_{i'=1}^{i-1}m_{i'}$ will be 0). Thus the $j$'th output of our neural network looks at the $j$'th element of each of the $q$ different blocks $i\in [p]$, and outputs it only for the \emph{first} block with exactly $s$ non-zero entries, as desired.

The bounds on the width, depth, and weights result from corresponding bounds on the components $\Ncal^{NZC}_\delta$, $\Ncal_{\delta}^{C}$, and $\Ncal^{IFP}$ from Fact~\ref{fact:comparison_net} and Lemmas~\ref{lemma:non_zero_counting_net_output} and \ref{lem:4-term} respectively.
\end{proof}

\end{document}